\documentclass[lettersize,journal]{IEEEtran}
\usepackage{amssymb}
\usepackage{amsthm}
\usepackage{amsmath,amsfonts}
\usepackage{algorithm,algpseudocode,algorithmicx}
\usepackage{array}
\usepackage{textcomp}
\usepackage{stfloats}
\usepackage{url}
\usepackage{verbatim}
\usepackage{graphicx, subfigure}
\usepackage{multicol}
\usepackage{multirow}
\usepackage{booktabs}
\usepackage{float}
\usepackage{cite}
\usepackage{siunitx}
\usepackage{enumitem}
\usepackage[hidelinks]{hyperref}
\usepackage{colortbl}
\usepackage{parskip}
\usepackage{orcidlink}
\usepackage{interval}
\usepackage{siunitx}
\intervalconfig{soft open fences}
\usepackage{gensymb}

\usepackage{hyperref}  
\usepackage{cleveref}
\newtheorem{thm}{Theorem}
\newtheorem{lemma}[thm]{Lemma}

\newtheorem{proposition}{Proposition}

\DeclareMathOperator{\atantwo}{atan2}

\DeclareMathOperator{\acos}{acos}

\DeclareMathOperator*{\argmin}{arg\,min}

\newcommand{\vmax}{v_{\max}}
\newcommand{\thetamax}{\frac{\vmax}{r}}
\newcommand{\twonorm}[1]{\left\lVert#1\right\rVert_2}

\newcolumntype{E}{>{\columncolor{lightgray}}c}
\newcolumntype{F}{c}

\begin{document}

\title{Accelerated Reeds-Shepp and Under-Specified Reeds-Shepp Algorithms for Mobile Robot Path Planning}

\author{Ibrahim Ibrahim\,\orcidlink{0000--0001--6840--558X} \emph{Student Member, IEEE},
	\and
	Wilm Decré\,\orcidlink{0000--0002--9724--8103} \emph{Member, IEEE},
	\and
	Jan Swevers\,\orcidlink{0000--0003--2034--5519}

	\thanks{
		This work has been carried out within the framework of Flanders Make's SBO project ARENA (Agile \& REliable NAvigation).}
	\thanks{The authors are with the MECO Research Team, Department of Mechanical Engineering, KU Leuven, Belgium and Flanders Make@KU Leuven, 3000, Belgium. {\tt\footnotesize ibrahim.ibrahim@kuleuven.be}, {\tt\footnotesize wilm.decre@kuleuven.be},
		{\tt\footnotesize jan.swevers@kuleuven.be}}
}

%



\maketitle

\begin{abstract}
	In this study, we present a simple and intuitive method for accelerating optimal Reeds-Shepp path computation. Our approach uses geometrical reasoning to analyze the behavior of optimal paths, resulting in a new partitioning of the state space and a further reduction in the minimal set of viable paths. We revisit and reimplement classic methodologies from the literature, which lack contemporary open-source implementations, to serve as benchmarks for evaluating our method. Additionally, we address the under-specified Reeds-Shepp planning problem where the final orientation is unspecified.
	We perform exhaustive experiments to validate our solutions. Compared to the modern C++ implementation of the original Reeds-Shepp solution in the Open Motion Planning Library, our method demonstrates a $15\times$ speedup, while classic methods achieve a $5.79\times$ speedup. Both approaches exhibit machine-precision differences in path lengths compared to the original solution. We release our proposed C++ implementations for both the accelerated and under-specified Reeds-Shepp problems as open-source code.
\end{abstract}
\section{Introduction}
\IEEEPARstart{P}{ath} planning for autonomous ground vehicles and mobile robots is a critical challenge that has seen continuous development over the decades. The Reeds-Shepp path planning problem, formulated in the early 1990s, focuses on finding the shortest paths between two poses for car-like vehicles under kinematic constraints, allowing both forward and backward motion. Despite the algorithm's robustness, its computational demand — evaluating 46 path types — remains a significant burden for real-time applications, particularly as modern environments grow increasingly dynamic and constrained.

Efficient and rapid path planning is vital in today's autonomous systems, such as urban autonomous driving, warehouse robotics, and automated valet parking, where real-time decision-making and quick adaptations to obstacles are crucial. To address these demands, we propose a significant improvement to the classical Reeds-Shepp algorithm by reducing the number of path types to consider from 46 to 20 and by introducing a partitioning scheme such that only one path type is evaluated per query instead of 20. This significantly simplifies the computational process, making the path planning algorithm more than ten times faster than current implementations. This speedup is crucial when the solver is used within wider planning schemes that call it as a subroutine thousands or millions of times, significantly enhancing real-time performance and safety. Moreover, our work introduces the underspecified Reeds-Shepp planning problem, which is of particular interest to autonomous vehicles. This capability finds critical application scenarios in warehouse robotics and agricultural autonomous vehicles, where the final orientation is not necessarily specified, but path optimality is required. It is also essential for situations where the robot must pass through specific waypoints starting from an initial pose without specified orientation through those waypoints. Additionally, it is highly relevant to grid-based applications, as it allows the robot to determine the shortest path to any grid cell efficiently. By solving this problem, we enhance the robot's navigational capabilities and operational efficiency. Our proposed method not only offers theoretical advancements but also practical improvements for real-time applications, directly translating to safer, more efficient, and higher-performing autonomous systems.


\subsection{Problem Statement}
\IEEEpubidadjcol{}
The Reeds-Shepp path planning problem focuses on finding the shortest path that a car-like vehicle can take from an initial configuration $p_{0} = \left(x_{0}, y_{0}, \theta_{0}\right)$ to a final configuration $p_{f} = \left(x_{f}, y_{f}, \theta_{f}\right)$ in the configuration space $\mathrm{SE}(2)$. The problem accounts for the vehicle's motion constraints, including a minimum turning radius and the ability to move both forward and backward. These constraints define the following simple car-like kinematic model~\cite{lavalle2006}:
\begin{equation}
	\begin{aligned}
		\dot{x}      & = v\cos{\theta}, \quad \dot{y} = v\sin{\theta},            \\
		v            & \in \{ -\vmax, \vmax \},                                   \\
		\dot{\theta} & = u, \quad u \in \bigg\{ -\thetamax, 0, \thetamax \bigg\}, \\
	\end{aligned}
	\label{eq:kinematic_model}
\end{equation}
where the vehicle's minimum turning radius is denoted by $r$, its maximum forward and backward velocities are denoted by $\pm v_{\mathrm{\max}}$, and its angular velocity is denoted by $u$. Here, $(x, y) \in \mathbb{R}^{2}$ represents the 2D position of the center of the vehicle's rear axle, and $\theta \in \rinterval{0}{2\pi}$ --- or $\theta \in \rinterval{-\pi}{\pi}$ --- represents the vehicle's orientation.

Model~\eqref{eq:kinematic_model} gives rise to the following set of admissible motions:
\begin{IEEEitemize}
	\item{$s^{+}$~: Forward straight line with $v = \vmax, u = 0$}
	\item{$s^{-}$~: Backward straight line with $v = -\vmax, u = 0$}
	\item{$l^{+}$~\hspace{1.5pt}: Forward left turn with $v = \vmax, u = \thetamax$}
	\item{$l^{-}$~\hspace{1.5pt}: Backward left turn with $v = -\vmax, u = -\thetamax$}
	\item{$r^{+}$~: Forward right turn with $v = \vmax, u = -\thetamax$}
	\item{$r^{-}$~: Backward right turn with $v = -\vmax, u = \thetamax$}
\end{IEEEitemize}

where $s^{+}$ and $s^{-}$ are also commonly referred to, in~\cite{Reeds1990}, as $S^{-}$ and $S^{+}$, respectively, $l^{+}$ and $r^{+}$ as $C^{+}$, and $l^{-}$ and $r^{-}$ as $C^{-}$.

As such, the challenge is to find the sequence of motions --- geodesics --- that lead from $p_{0}$ to $p_{f}$ at the minimum overall path length. Such a path can be defined as $P(s) = \left(x(s), y(s), \theta(s)\right)$, where $s \in \interval{0}{L}$, $L$ is the path length, $P(0) = p_{0}$, and $P(L) = p_{f}$. The path length $L$ is the sum of the lengths of the individual segments of the path --- straight line segments and arc segments. Such a path is continuous except at the points where the vehicle changes its motion direction, which are called cusps.

We define the circles that the vehicle's rear axle center traces when turning left and right as Left-Hand Circle (LHC) and Right-Hand Circle (RHC), respectively. Their centers are denoted by $c_{L}$ and $c_{R}$, respectively. The ``0'' and ``f'' subscripts stand for starting and final configurations respectively. We illustrate an example in Fig.~\ref{fig:global} with a turning radius of 30.

\begin{figure}[!ht]
	\centering
	\includegraphics[width=8.5cm,keepaspectratio]{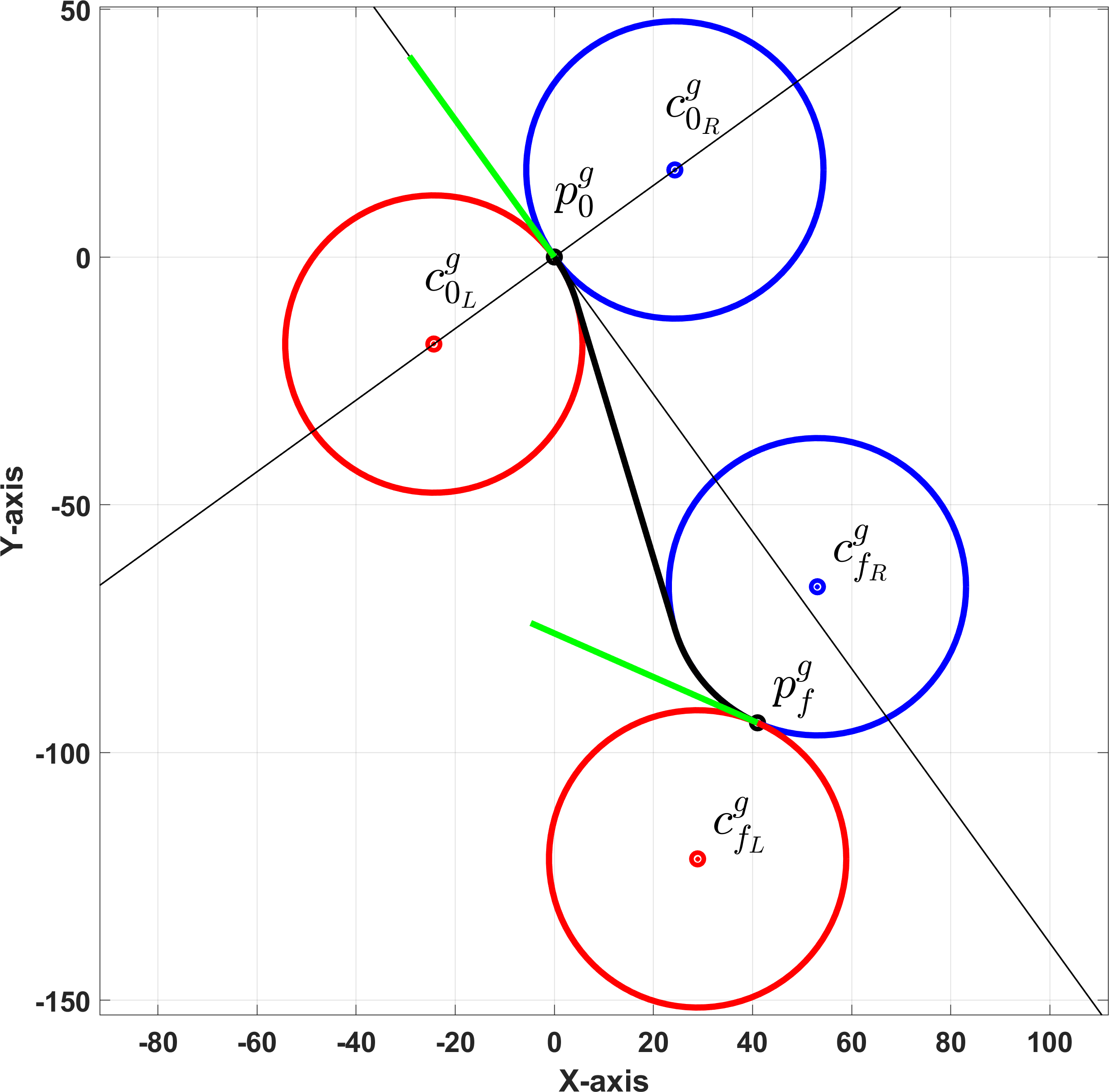}
	\caption{Shortest Reeds-Shepp path from $p_{0}^{g}$ to $p_{f}^{g}$, where $g$ stands for global coordinates. The local coordinate system centered at $p_{0}$ is illustrated in black lines. $c_{L}^{g}$ and $c_{R}^{g}$ correspond to the Left-Hand Circle (LHC) and Right-Hand Circle (RHC) centers defined by the maximum steering angle at the start and final configurations. LHC and RHC are illustrated as red and blue circles, respectively. Optimal path is $l^{-}s^{-}r^{-}$, or $C^{-}S^{-}C^{-}$. Green line = heading.}\label{fig:global}
\end{figure}

We introduce the under-specified Reeds-Shepp planning problem, which, to the best of our knowledge, has not been previously introduced or addressed in literature. This problem arises when the final configuration $p_{f}$ is not fully specified, more specifically, when the final orientation $\theta_{f}$ is not specified.

Given an initial configuration $p_{0} = (x_{0}, y_{0}, \theta_{0})$ and a final position $(x_{f}, y_{f})$ with an unspecified final orientation $\theta_{f}$, the goal is to find the final orientation $\Omega$ such that the resulting Reeds-Shepp path from $p_{0}$ to the final configuration $p_{f} = (x_{f}, y_{f}, \Omega)$ is the shortest possible. Formally, $\Omega$ is defined as:
\begin{equation}\label{eq:underspecified}
	\begin{aligned}
		\Omega = \argmin_{\theta_{f} \in \rinterval{-\pi}{\pi}} L(p_{0}, (x_{f}, y_{f}, \theta_{f})),
	\end{aligned}
\end{equation}
where $L(p_{0}, (x_{f}, y_{f}, \theta_{f}))$ represents the length of the Reeds-Shepp path from the initial configuration $p_{0}$ to the final configuration $(x_{f}, y_{f}, \theta_{f})$. The objective is to determine $\theta_{f}$ such that the path length is minimized.

\subsection{Related Work}\label{subsec:related_work}
The original cusp-free problem, i.e., the case where $v \in \big\{\vmax\big\}$ and $u \in \big\{ -\thetamax, 0, \thetamax \big\}$, was originally solved by Dubins~\cite{Dubins1957}. Dubins proved that optimal paths in such a scenario can be found among a minimal set of six path types, each having at most three motion segments that are either line segments $S$ or arc segments $C$. In 1972, Pecsvaradi~\cite{Pecsvaradi1972} provided an optimal horizontal guidance law for aircraft in the terminal area. By formulating the problem of guiding an aircraft in minimum time from an arbitrary initial position to the outer marker as a nonlinear optimal control problem, Pecsvaradi derived the optimal control law for the aircraft using the maximum principle. The optimal control synthesis for the Dubins problem was also proposed by Bui et al.~\cite{Xuan1994}, who partitioned the motion plane into regions, each associated with one optimal Dubins path.

In 1990, Reeds and Shepp~\cite{Reeds1990} extended Dubins' work to include cusps by describing a minimal set of 48 path types that is guaranteed to contain the shortest path between any two configurations. They also provided formulae to compute the length of each path's segments. Their approach is based on inducing first order necessary conditions from a cost function minimization that is subject to equality constraints. The same results were independently reached by other researchers~\cite{Sussmann1991, Boissonnat1992} who further reduced the number of path types to 46.

Soon after, methods that rely on repeatedly solving the Reeds-Shepp problem as part of larger motion planning schemes that account for additional constraints, such as obstacles, arose~\cite{Laumond1990, Jacobs1991}. Therefore, the need for fast and efficient algorithms also arose. Up until the early 1990s, algorithms requiring optimal Reeds-Shepp paths had to compute all 46 paths at every iteration.

In 1993, Souères and Laumond~\cite{Soueres1996} combined the necessary conditions given by Pontryagin's Maximum Principle, geometric reasoning, and Lie algebra to provide a global optimal control synthesis for Reeds-Shepp vehicles. In this work, the authors use geometrical arguments and symmetries to partition the configuration space into elements for all configurations having the same final orientation. While the authors associate each element with one of the 46 path types based on geometrical arguments and symmetries, they do not detail the specific algorithm used for this association.

In contrast, in 1995, Desaulniers and Soumis~\cite{Desaulniers1995}, building on previous works, provided an algorithm that is based on a new partition of the configuration space using a geometrical approach. This approach works in the space of the final position's Left-Hand Circle (LHC) center, $c_{f_{L}}$ (as defined in Fig.~\ref{fig:global}), and is based on five main steps (for radius $r = 1$):
\begin{IEEEenumerate}
	\item Read a final configuration $p_{f}$ and compute $c_{f_{L}}$.
	\item Find $Q \in \big\{Q_{1}, Q_{2}, Q_{3}, Q_{4}\big\}$: subdivision of the $c_{f_{L}}$'s space into four quadrants. The main quadrant from and to which transformations are made is $Q_{1} = \big\{\small(c_{f_{L_{x}}}, c_{f_{L_{y}}}, \theta_{f}\small): c_{f_{L_{x}}} < -1, c_{f_{L_{y}}} \geq{} 0 \big\}$.
	\item Find $R \in \big\{R_{1}, R_{2}, \ldots, R_{19}\big\}$, which represents the region within $Q_{1}$ containing $c_{f_{L}}$, ensuring that the boundaries between regions are also taken into account.
	\item Find $S$, the subdivision of $R$ containing $\theta_{f}$. If a single type is associated with $S$, then the optimal path is found. Otherwise, if two types are associated with $S$, obtain the path with the shortest distance among the two.
	\item If in step 2, $c_{f_{L}}$ was not in $Q_{1}$, then perform an inverse transformation to the shortest path obtained in step 4.
\end{IEEEenumerate}

More specifically, the authors in~\cite{Desaulniers1995} use the first-order necessary conditions that were initially developed in~\cite{Reeds1990} to identify 26 path types and 19 possible regions that $c_{f_{L_{x}}}$ can belong to within $Q_{1}$. They proceed to subdivide each region into subregions based on 35 $\theta$ equations that define potential ranges for $\theta_{f}$. They end up with a 161-element partition, 150 of which have a unique path type, and 11 of which have two possible path types. For the sake of brevity, we only provide, in Fig.~\ref{fig:regions}, a visual illustration of the 19-regions-subdivision of quadrant $Q_{1}$, which we reproduced with the help of the authors of the original work.

\begin{figure}[!ht]
	\centering
	\includegraphics[width=8.5cm,keepaspectratio]{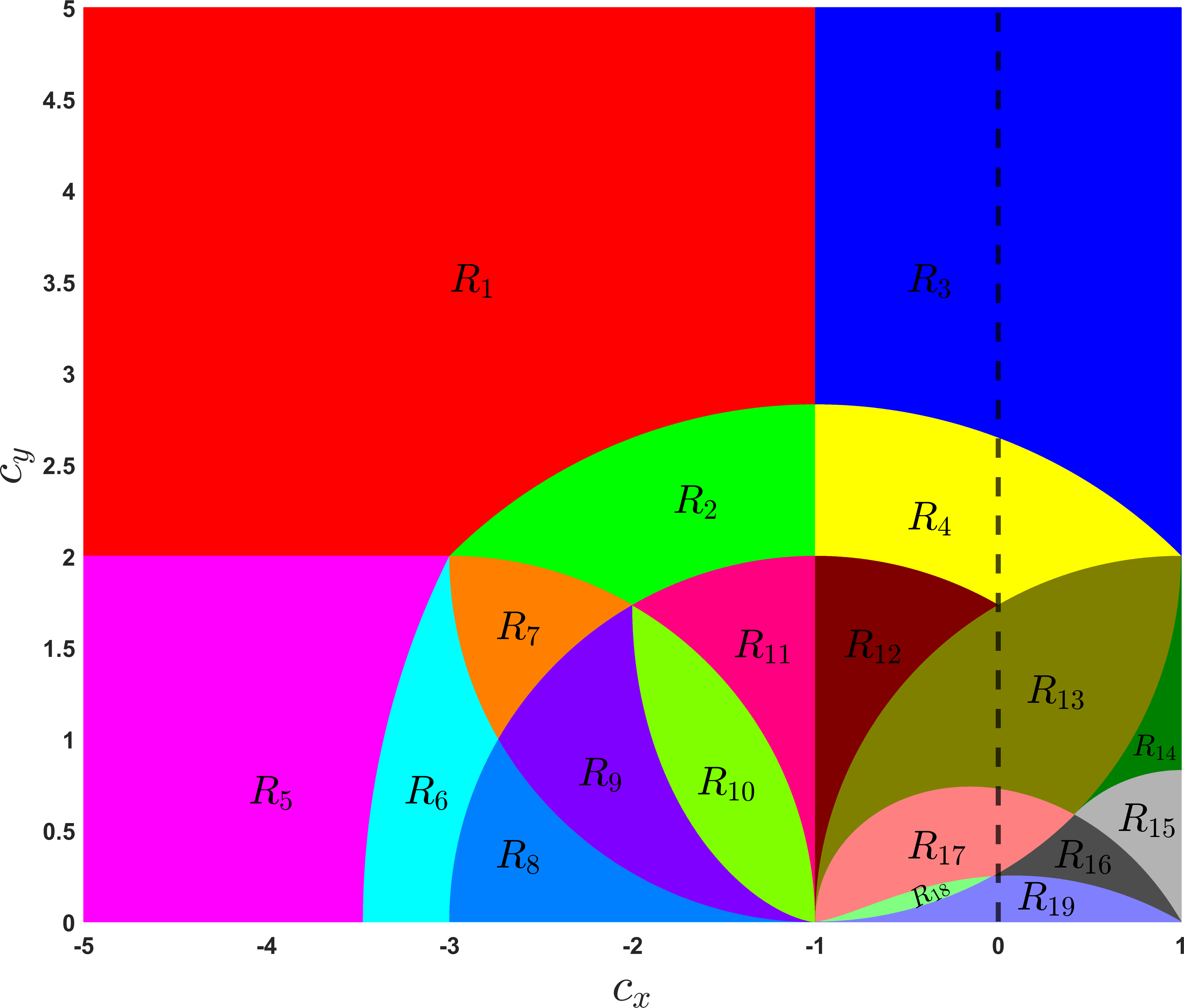}
	\caption{Subdivision of quadrant $Q_{1}$ of $c_{f_{L_{x}}}$'s space into 19 regions. Reproduced from~\cite{Desaulniers1995}. Each color refers to a specific region, enumerated as $R_{1}, R_{2}, \ldots, R_{19}$.}\label{fig:regions}
\end{figure}

Since then, and to the best of our knowledge, there has not been any significant work on improving the classical solutions to the Reeds-Shepp problem itself. There has been, however, major extensions for time optimal control, continuous-curvatures, and additional constraints, among others.

In 2004, Fraichard and Scheuer~\cite{Fraichard2004} extended Reeds-Shepp paths to continuous-curvature paths.
In 2009, Wang et al.~\cite{Wang2009} developed a new method to determine time optimal trajectories for nonholonomic bidirectional robots using switching vectors. More recently, in 2022, Ben-Asher and Rimon~\cite{Ben-Asher2022} extended previous results on time optimal synthesis by analyzing the problem as an optimal control problem by minimum principle and by singular control theory.

In 2010, Salaris et al.~\cite{Salaris2010} provided an optimal control synthesis for the case where constraints on the vehicle's field-of-view exist, such as in the case of visual servoing. Salaris et al.~\cite{Salaris2015} also provided an epsilon-optimal synthesis for vehicles with limited field-of-view sensors.

All of the aforementioned works are developed for free-space Reeds-Shepp path planning, since after works such as~\cite{Laumond1990 ,Jacobs1991}, traction died for methods that repeatedly solve the Reeds-Shepp problem to account for obstacles due to inefficiencies relative to other methods. For example, a recent work that was published in 2022 by Wen et al.~\cite{Wen2022} that solves multi-agent path finding for car-like robots with kinematic and spatiotemporal constraints uses a spatiotemporal hybrid-state $A^\ast$ instead, which proves to be more efficient in this scenario. Sampling based planners that obtain solutions for nonholonomic vehicles in occupied environments such as RRT* and PRM* are still popular as seen in~\cite{Karaman2013},~\cite{Yan2016},~\cite{Spano2022}, and~\cite{Li2023}. Such methods sample motion primitives repeatedly as subroutines. Based on the reasoning in this paragraph, having fast and efficient underlying algorithms for computing optimal Reeds-Shepp primitives is necessary in order to account for additional constraints such as obstacles and vehicle footprint.

Moreover, due to the lack of open-source implementations of most of the aforementioned methods and algorithms, the Open Motion Planning Library's implementation of the original Reeds-Shepp solution (OMPL)~\cite{Sucan2012} remains the de facto standard implementation for relevant motion planning in robotics. It still computes all 46 paths, but it asserts the necessary first order conditions for the optimality of each path. It is written in C++, making it a great benchmarking tool.

\subsection{Contributions \& Paper Structure}
Based on the discussion reported in~\ref{subsec:related_work}, we introduce a new solution that reduces the minimum set of paths to consider down to 20 essential types. We associate each of those types with a partition in the configuration space, allowing us to write an algorithm that outperforms the modern OMPL implementation by more than an order of magnitude. By proving our proposed algorithm's completeness and correctness, we show that those 20 types are sufficient to cover the entire configuration space. By performing exhaustive tests, including tests on edge cases, we demonstrate the reliability and effectiveness of our proposed solutions. We run billions of tests with random start and final configurations as well as random radii, spanning the entirety of a sampled space. We validate that each path reaches the final configuration with zero error.

Lastly, in this paper, we introduce and address the under-specified Reeds-Shepp planning problem, which, to the best of our knowledge, has not been previously addressed in literature. This problem arises when the final configuration is not fully specified, more specifically, when the final orientation $\theta_{f}$ is not specified. This is very relevant for grid-based planning applications. Consider the scenario where a car-like robot begins at a certain configuration $p_{0}$ in an $N\times M$ grid. By solving the under-specified Reeds-Shepp problem, we endow the robot with the knowledge of the final orientations $\theta_{f}^{N\times M}$ that produce the shortest paths to all grid cells. This allows the robot to move anywhere in the grid, always at a minimum-distance basis. Using the computed $\theta_{f}^{N\times M}$, we can compute a shortest-distance transform in the style of Euclidian distance transforms for the grid. The latter is symmetrical and rotationally invariant in its local frame for a fixed radius $r$, so it can be computed once and stored offline and it can be interpolated at any final position $(x_{f}, y_{f}) \in \mathbb{R}^{2}$.

Therefore, we introduce in this paper the following contributions:
\begin{itemize}
	\item A considerably simpler and more geometrically-intuitive partitioning of the configuration space that results in a further reduction in the minimal set of viable paths to 20.
	\item An implementation of the classical Reeds-Shepp solution as described in~\cite{Desaulniers1995} in C++ for benchmarking purposes.
	\item Introduction of the under-specified Reeds-Shepp planning problem and its solution.
	\item Exhaustive experiments to validate the effectiveness of our proposed solutions.
	\item Open-source C++ implementations of this paper's content\footnotemark[1]\textsuperscript{,}\footnotemark[2].
	      \footnotetext[1]{\url{https://github.com/IbrahimSquared/accelerated-RS-planner}}
	      \footnotetext[2]{\url{https://github.com/IbrahimSquared/underspecified-RS-planner}}
\end{itemize}

This paper is organized as follows:
\begin{itemize}

	\item In Section~\ref{sec:proposed}, we lay out the formal approach that we adopted to arrive to the accelerated Reeds-Shepp algorithm. We highlight the symmetries that our algorithm exploits as well as the partitioning of the configuration space that lead to our accelerated solution. We cover the geometric reasoning behind each of the 20 types, one at a time.

	\item In Section~\ref{sec:underspecified}, we discuss the under-specified Reeds-Shepp planning problem and its solution.

	\item In Section~\ref{sec:results}, we discuss the results of our exhaustive experiments and benchmark our proposed solution to~\cite{Desaulniers1995} and to OMPL's implementation. We also validate our proposed under-specified Reeds-Shepp problem solution.

	\item In Appendix~\ref{app:algs}, we provide the necessary pseudocode for the proposed accelerated Reeds-Shepp algorithm.
\end{itemize}

\section{Accelerated Reeds-Shepp Solution}
 {}\label{sec:proposed}
\subsection{Formal Proof Approach}
The proposed algorithm operates by partitioning the configuration space and producing a distinct solution within each subpartition. We argue for the correctness and the completness of the proposed algorithm by following the steps below:
\begin{IEEEenumerate}
	\item Define the two main partitions that distinguish the solution families.
	\item Within each main partition, establish subpartitions based on geometric invariants of each solution type, proven through geometric reasoning.
	\item For each subpartition, compute the solution using the analytical expressions from the original Reeds-Shepp solution specific to the path type.
	\item Ensure coverage of the configuration space using complementary predicates, leaving no region unaccounted for and avoiding overlap between subspaces.
\end{IEEEenumerate}
As such, we construct a solution that is both correct and complete --- the resulting algorithm handles every possible case correctly, and the partitions together address the entire problem space without omission, overlap, or redundancy. In Section~\ref{subsec:symmetries}, we highlight the symmetries exploited to simplify our proposed accelerated Reeds-Shepp solution. In Section~\ref{subsec:partitioning}, we define our main partitions and their corresponding subpartitions. More details on the non-uniqueness of the shortest path and the completeness and correctness of our solution are provided in Sections~\ref{subsec:nonuniqueness} and~\ref{subsec:completeness}. We summarize our findings in Algorithms~\ref{alg:ForwardProjection} to~\ref{alg:B_partitions} in Appendix~\ref{app:algs}.

\subsection{Symmetries \& Transformations}\label{subsec:symmetries}
We ground our reasoning in the local frame of reference of the starting configuration $p_{0}$. The solution is rotationally invariant with respect to $p_{0}$'s frame of reference and the start and final configurations themselves can we swapped without affecting the solution as was shown in the original work~\cite{Reeds1990}. This allows for simpler, consistent, and more intuitive reasoning/visualization and allows us to leverage the local symmetries that arise. In the original work~\cite{Reeds1990}, Section 8, the authors highlight those symmetries. They show that the underlying problem is invariant under certain transformations. We exploit those symmetries to transform the final configuration into the first quadrant of the local frame of reference of the starting configuration $p_{0}$. We then solve the Reeds-Shepp problem in the first quadrant, and project the solution back to the original quadrant using the inverse transformations.

\begin{figure}[!ht]
	\centering
	\includegraphics[width=8.5cm,keepaspectratio]{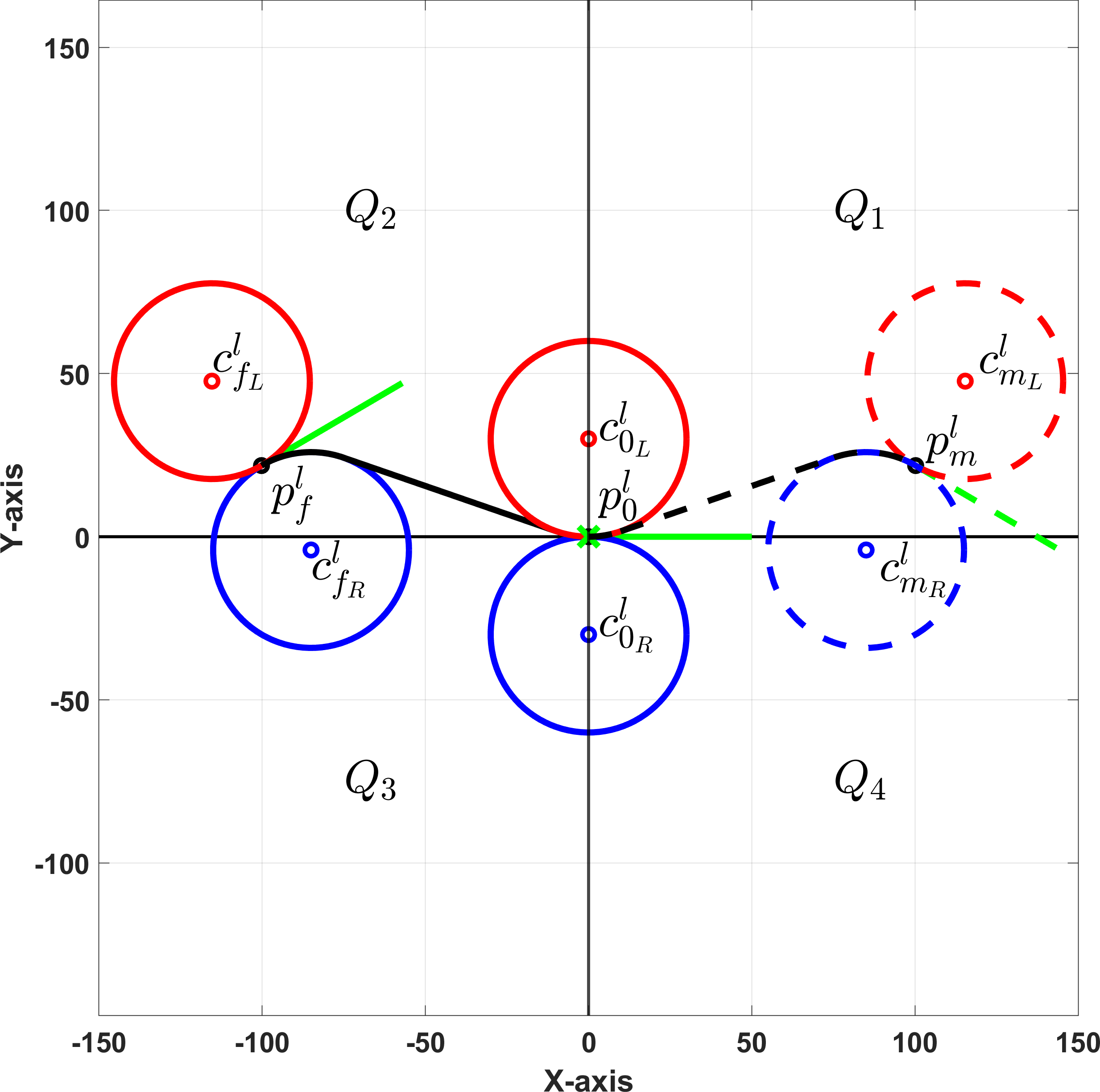}
	\caption{The local final configuration $p_{f}^{l}$ is mirrored from the second local quadrant $Q_{2}$ into $p_{m}^{l}$ in the first local quadrant $Q_{1}$. The mapped LHC and RHC are shown as dashed red and blue circles respectively, centered around the mapped centers $c_{m_{L}}^{l}$ and $c_{m_{R}}^{l}$ respectively. The mapped shortest Reeds-Shepp path is illustrated as the dashed black line. $l$ stand for local frame of reference, $m$ for mirrored, $R$ for right, and $L$ for left.}\label{fig:local}
\end{figure}

A similar approach was also adopted in~\cite{Desaulniers1995}. We illustrate an example $Q_{2}$-to-$Q_{1}$ projection in Fig.~\ref{fig:local}. In this case, where $p_{f}^{l} \in Q_{2}$, the solution from $p_{0}^{l}$ to $p_{f}^{l}$ --- $l^{-}s^{-}r^{-}$ --- is the same as the solution from $p_{0}^{l}$ and $p_{m}^{l}$ --- $l^{+}s^{+}r^{+}$ --- but with reverse directions of the segments. $p_{m}^{l}$ is the mirrored final configuration in $Q_{1}$, and is the result of the transformation of $p_{f}^{l}$ into $Q_{1}$.

All relevant transformations are summarized in Alg.~\ref{alg:ForwardProjection} and Alg.~\ref{alg:BackwardProjection} in Appendix~\ref{app:algs}. Alg.~\ref{alg:ForwardProjection} takes in the start position and the final configuration in $p_{0}$'s frame of reference, and returns the mapped configuration. Alg.~\ref{alg:BackwardProjection} takes in the Reeds-Shepp solution from $p_{0}^{l}$ to $p_{m}^{l}$ in $Q_{1}$ and projects it back to its original quadrant. The solution consists of up to five segments types $\mathcal{T}^{1\times5}$, each marked as `l', `r', `s', or `n', with `n' marking the end of the path. Each segment $i$ has a direction $\mathcal{D}^{1\times5}[i]$, and a length $\mathcal{L}^{1\times5}[i]$. Following the convention established in classical methods, we use T, U, V, and $\frac{\pi}{2}$ as potential lengths of segments. The $\frac{\pi}{2}$ marks the maximum central angle of an arc segment in the Reeds-Shepp problem as no arc segment within the minimal set of 48 paths exceeds $\frac{\pi}{2}r$. As such, an example path is $l_{T}^{+}s_{U}^{+}r_{V}^{+}$. Another example is $r_{T}^{-}l_{\frac{\pi}{2}}^{+}s_{U}^{+}r_{\frac{\pi}{2}}^{+}l_{V}^{-}$.

A \emph{path type}, however, is a classification of paths where each is represented as a triplet $\left(\mathcal{T}, \mathcal{D}, \mathcal{L}\right)$, and all paths of a given type are computed using the same analytic expressions and satisfy certain geometric invariant properties. For instance, all paths of the type $l_{T}^{+}s_{U}^{+}r_{V}^{+}$ share specific geometric constraints and can be derived using a consistent set of analytic formulas.

\subsection{Partitioning of the Configuration Space}\label{subsec:partitioning}
In the original Reeds-Shepp solution, each path in the minimal set of 48 paths is categorized into one of the following five: $CSC$, $CCC$, $CCCC$, $CCSC$, $CCSCC$\@. Each of those five constitutes a family of potential paths, with $CCSC$ including $CSCC$ paths. We first introduce geometrical reasoning in Lemmas 1 and 2 that lay the necessary foundation for Proposition 1 which allows us to classify the paths into two main sets: A and B. We then proceed to define the subpartitions of each set using Propositions~\ref{prop:p2} through~\ref{prop:p13_vs_p20}. Those subspace partitions allow us to classify the paths into 20 unique types with no overlap. Each proposition is proven through geometric reasoning.
\begin{lemma}\label{lem:first}
	Starting from an arbitrary configuration $p_{0}$, a left turn to $p_{f}$ rotates both the robot and its RHC about $c_{0_{L}}$ ($p_{0}$'s LHC center). Similarly, a right turn rotates the robot and its LHC about $c_{0_{R}}$ ($p_{0}$'s RHC center).
\end{lemma}
\begin{proof}
	The proof follows directly from the definition of the LHC and RHC\@.
\end{proof}
It follows that as a robot undertakes a left turn, its RHC also rotates ($c_{0_{R}}\neq{}c_{f_{R}}$), whereas its LHC remains fixed ($c_{0_{L}}=c_{f_{L}}$). Similarly, as a robot undertakes a right turn, its LHC rotates ($c_{0_{L}}\neq{}c_{f_{L}}$), whereas its RHC remains fixed ($c_{0_{R}}=c_{f_{R}}$). During a $C$ motion primitive, $c_{0_{L}}$ and $c_{f_{R}}$, or conversely $c_{0_{R}}$ and $c_{f_{L}}$, remain externally tangent. Therefore, the distance between the two centers during a turn is fixed and equal to $2r$. Any translation along a straight line segment $S$ breaks such external tangency, since $c_{f_{L}}$ and $c_{f_{R}}$ translate away from $c_{0_{R}}$ and $c_{0_{L}}$ respectively.
\begin{lemma}\label{lem:second}
	During a $C$ motion primitive, and at the limit of $\theta = \pm\frac{\pi}{2}$, the distance between the centers $c_{0_{R}}$ and $c_{f_{R}}$, or the distance between $c_{0_{L}}$ and $c_{f_{L}}$, can be at most $\mathcal{K}=\sqrt{(2r){}^2 + (2r){}^2} = 2r\sqrt{2}$.
\end{lemma}
\begin{proof}
	The proof follows directly from the Pythagorean theorem. The distance between the centers is the hypotenuse of a right triangle with sides of length $2r$.
\end{proof}

\begin{figure}[!ht]
	\centering
	\includegraphics[width=8.5cm,keepaspectratio]{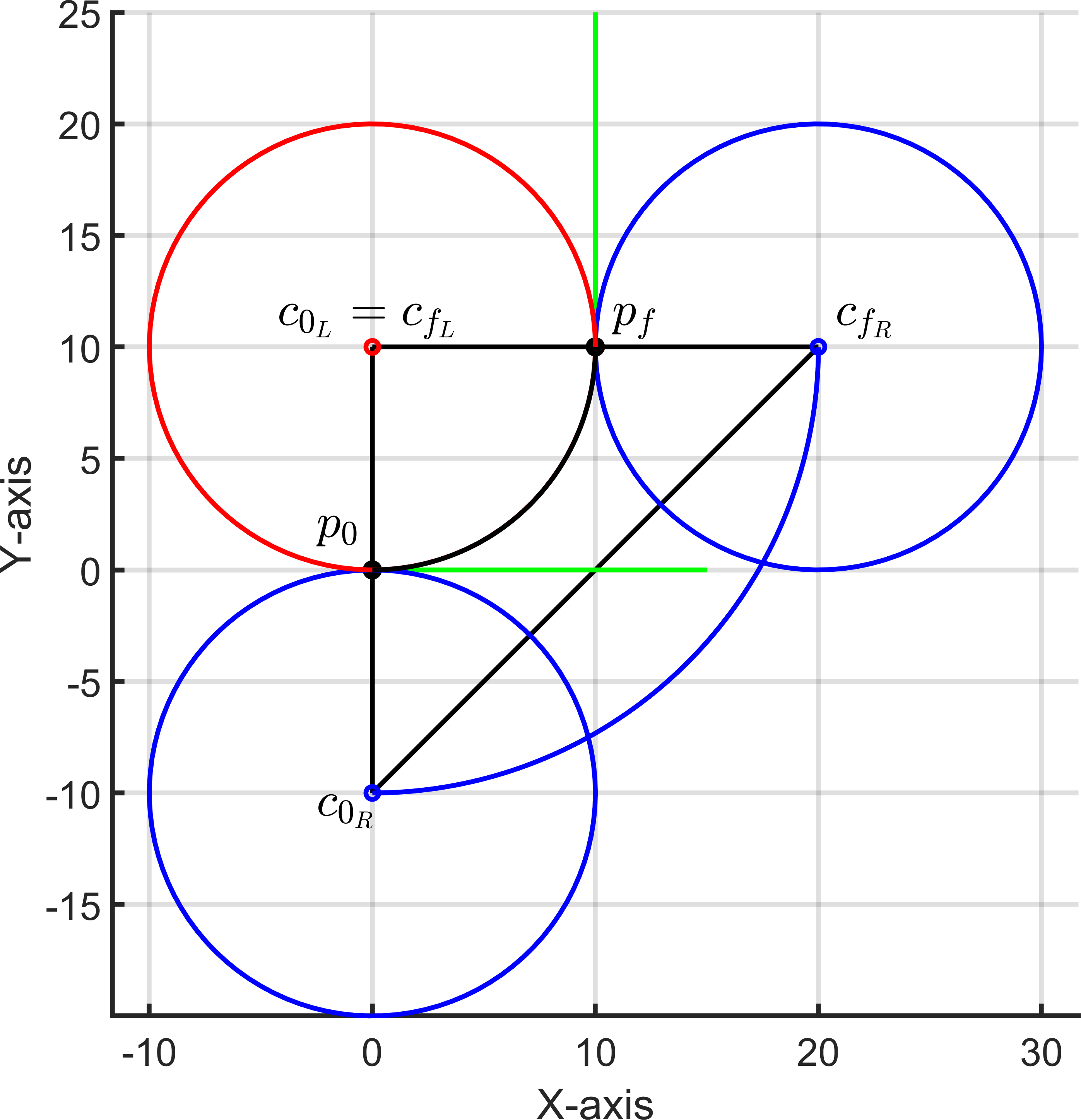}
	\caption{The right triangle with sides \(2r\) (where \(r = 10\)) is formed by a \(\frac{\pi}{2}\) left turn from \(p_0\) to \(p_f\). The hypotenuse, of length \(2r\sqrt{2}\), plays a crucial role in partitioning the solution space. The arc defined by the motion of RHC's center, as well as the RHC itself, is traced in blue.}\label{fig:sqrt_two}
\end{figure}

We illustrate a sample left turn of $\frac{\pi}{2}$ and the right triangle that is formed between $c_{0_{L}}$, $c_{0_{R}}$, and $c_{f_{R}}$ in Fig.~\ref{fig:sqrt_two}. The same reasoning that follows from Lemma~\ref{lem:first} and Lemma~\ref{lem:second} applies for the case when the optimal path consists of two, three, and even four $C$ motion primitives. Take for example the path $l_{\frac{\pi}{2}}^{+}r_{\frac{\pi}{2}}^{+}$. The distance between $c_{0_{L}}$ and $c_{f_{L}}$ is $\mathcal{K}$ and the distances between $c_{f_{R}}$ and $c_{f_{L}}$ and $c_{f_{R}}$ and $c_{0_{L}}$ is $2r$. When followed by an $S$ motion primitive, such distances are no longer maintained. Therefore, $CCC$ and $CCCC$ paths satisfy the condition defined in Algorithm~\ref{alg:booleanset} in Appendix~\ref{app:algs}, whereas $CSC$, $CCSC$ paths do not. We make the following proposition:
\begin{proposition}\label{prop:booleanset}
	Paths \(CCC\) and \(CCCC\) satisfy one or more of these conditions for the local start configuration \(p_{0}^{l} = (x_{0}^{l}, y_{0}^{l}, \theta_{0}^{l})\), the final mirrored configuration \(p_{m}^{l} = (x_{m}^{l}, y_{m}^{l}, \theta_{m}^{l})\), and the minimum turning radius \(r\):
	\begin{enumerate}
		\item The distances \(LL = \|c_{0_L} - c_{m_L}\|\), \(RR = \|c_{0_R} - c_{m_R}\|\), and \(LR = \|c_{0_L} - c_{m_R}\|\) satisfy:
		      \[
			      RR \leq \mathcal{K}, \quad LL \leq \mathcal{K}, \quad LR \leq 2r,
		      \]
		      where \(\mathcal{K} = 2r\sqrt{2}\).

		\item The distances \(LL = \|c_{0_L} - c_{m_L}\|\), \(RR = \|c_{0_R} - c_{m_R}\|\), and \(RL = \|c_{0_R} - c_{m_L}\|\) satisfy:
		      \[
			      RR \leq \mathcal{K}, \quad LL \leq \mathcal{K}, \quad RL \leq 2r.
		      \]

		\item The distances \(LL = \|c_{0_L} - c_{m_L}\|\), \(LR = \|c_{0_L} - c_{m_R}\|\), and \(RL = \|c_{0_R} - c_{m_L}\|\) satisfy:
		      \[
			      LR \leq 2r \quad LL \leq \mathcal{K}, \quad RL \leq 2r.
		      \]
	\end{enumerate}
	while paths \(CSC\) and \(CCSC\) do not. See Algorithm~\ref{alg:booleanset} in Appendix~\ref{app:algs} for computational details.
\end{proposition}

\begin{proof}
	The proof follows directly from Lemma~\ref{lem:first} and Lemma~\ref{lem:second}.
\end{proof}
Essentially, Algorithm~\ref{alg:booleanset} in Appendix~\ref{app:algs} is a boolean function that takes in the starting and final configurations and returns true if the proposed condition is satisfied, and false otherwise. Proposition~\ref{prop:booleanset} is a direct consequence of Lemma~\ref{lem:first} and Lemma~\ref{lem:second} and allows us to classify the paths into two main sets as follows:
\begin{itemize}
	\item Set A\@: $CSC$, $CCSC$, and $CCSCC$
	\item Set B\@: $CCC$, $CCCC$, and $CCSCC$
\end{itemize}
Path types $CCSCC$ are included in both sets since we make no proposition about them.

We proceed with the reasoning for each of the 20 cases. As mentioned before, all of the reasoning for all of our 20 cases is carried out in the first quadrant of the local frame of reference. The local starting configuration need not be fixed to $p_{0} = (0, 0, 0)$, but we do that for convenience. The same goes for the turning radius of 20 that we chose. In all upcoming figures, the LHC and the motion of its center $c_{m_{L}}$ are visualized in red, and the RHC and the motion of its center $c_{m_{R}}$ in blue. The optimal path for each case is illustrated in black.
Cases $P_{1}-P_{8}$, $P_{10}$, $P_{11}$ ($CSC$, $CCSC$), are exclusive to set A, cases $P_{13}-P_{20}$ ($CCC$, $CCCC$) are exclusive to set B, and cases $P_{9}$ and $P_{12}$ ($CCSCC$) are common to both sets. We combine all propositions in each of the sets A and B to produce the algorithms Alg.~\ref{alg:A_partitions} and Alg.~\ref{alg:B_partitions} in Appendix~\ref{app:algs}. \\

\subsubsection{Set A}
We begin with an illustration in Fig.~\ref{fig:cond_0} that shows a path $l_{\Omega}^{+}s_{U}^{+}$, or $C^{+}S^{+}$. In this case, the final orientation $\theta_{f}$ is equal to $\Omega$ --- the angle which produces the shortest path from $p_{0}$ to $(x_{f}, y_{f})$ --- see Equation~\eqref{eq:underspecified}. During the initial $l_{\Omega}^{+}$ primitive, and based on Lemma~\ref{lem:first}, LHC remains fixed. During the second primitive, $s_{U}^{+}$, LHC translates along the straight red line segment. The angle defined by the line connecting $c_{0_{L}}$ and $c_{f_{L}}$ is equal to $\Omega$ and may be computed as:
\begin{equation}
	\begin{aligned}
		\angle L_{f}L_{0} = \atantwo{\left(c_{f_{L_{y}}}-c_{0_{L_{y}}}, c_{f_{L_{x}}}-c_{0_{L_{x}}} \right)}.
	\end{aligned}
\end{equation}

We define the path type $P_{2}$ with the following proposition after introducing an additional turn $l_{V}^{+}$ at the end of $l_{\Omega}^{+}s_{U}^{+}$:
\begin{proposition}\label{prop:p2}
	For path type $P_{2}$ of the form $l_{T}^{+}s_{U}^{+}l_{V}^{+}$ with $V \neq 0$, $\theta_{f} > \angle L_{f}L_{0}$.
\end{proposition}
\begin{proof}
	Upon introducing an additional left turn primitive $l_{V}^{+}$ at the end of $l_{\Omega}^{+}s_{U}^{+}$, RHC rotates about LHC's center, $c_{f_{L}}$, without affecting the line connecting $c_{0_{L}}$ and $c_{f_{L}}$. Since the additional left turn increases $\theta_{f}$ without changing $\angle L_{f}L_{0}$, then, for an $l_{T}^{+}s_{U}^{+}l_{V}^{+}$ path, $\theta_{f} > \angle L_{f}L_{0}$.
\end{proof}
The condition proposed in Proposition~\ref{prop:p2} is invariant for the first case in set A\@, $P_{2}$. We provide an illustration of this case in the left side of Fig.~\ref{fig:cond_2_3}.

Upon saturation of the central angle of the second arc segment in $P_{2}$ to $\frac{\pi}{2}$, a new motion primitive is required to accomodate for a further increase in $\theta_{f}$. We define the path type $P_{3}$ with the following proposition that allows us to distinguish the partition of $P_{3}$ from $P_{2}$:
\begin{proposition}\label{prop:p3}
	For path type $P_{3}$ of the form $l_{T}^{+}s_{U}^{+}l_{\frac{\pi}{2}}^{+}r_{V}^{-}$ with $V \neq 0$, $\theta_{f} > \angle L_{f}L_{0} + \frac{\pi}{2}$.
\end{proposition}
\begin{proof}
	At $\theta_{f}$ equal to $\angle L_{f}L_{0} + \frac{\pi}{2}$, the last $C$ primitive in $P_{2}$ saturates to $\frac{\pi}{2}$. In order to accomodate for a further increase in $\theta_{f}$, a new motion $C$ motion primitive is required. Based on the original minimal set of 48 paths, the next motion primitive has to be a right turn of negative direction, since there cannot be two consecutive left turns, and since a positive direction right turn decreases $\theta_{f}$ instead. Therefore, the path must be $l_{T}^{+}s_{U}^{+}l_{\frac{\pi}{2}}^{+}r_{V}^{-}$.
\end{proof}
An example $P_{3}$ case is illustrated in the right side of Fig.~\ref{fig:cond_2_3}.

\begin{figure}[!ht]
	\centering
	\includegraphics[width=8.5cm,keepaspectratio]{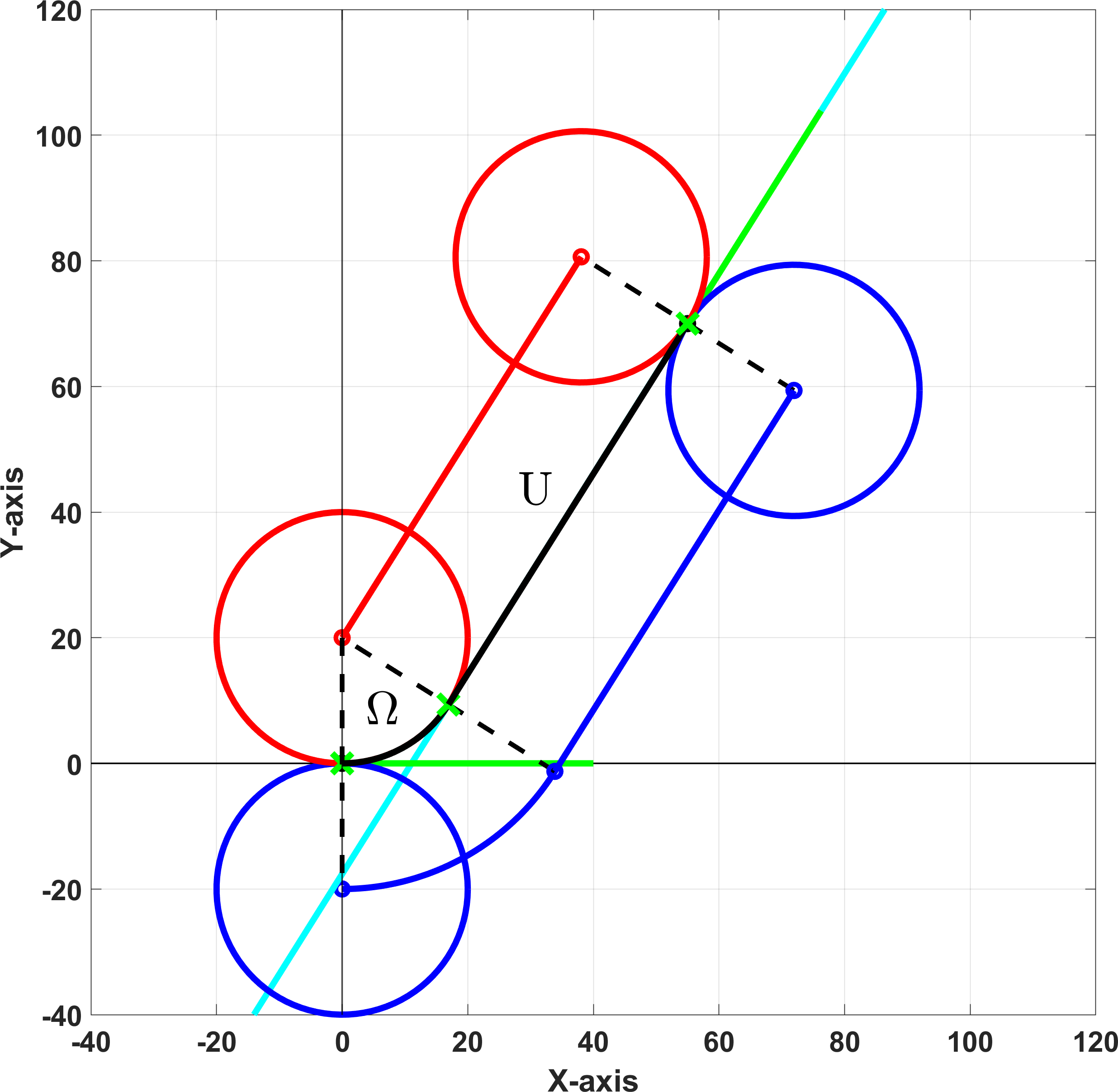}
	\caption{The optimal path between the start and final configurations is a degenerate/edge case in the CSC family, where the solution is $l_{\Omega}^{+}s_{U}^{+}$, or $C^{+}S^{+}$. The last arc segment has a length of zero. The cyan line marks $\theta_{f} = \Omega$. The angle $\Omega$ is measured as the central angle of the arclength undertaken during the first $l_{\Omega}^{+}$ motion pritimive starting from $p_{0}$.}\label{fig:cond_0}
\end{figure}

\begin{figure}[!ht]
	\centering
	\begin{subfigure}{}
		\includegraphics[width=4.25cm]{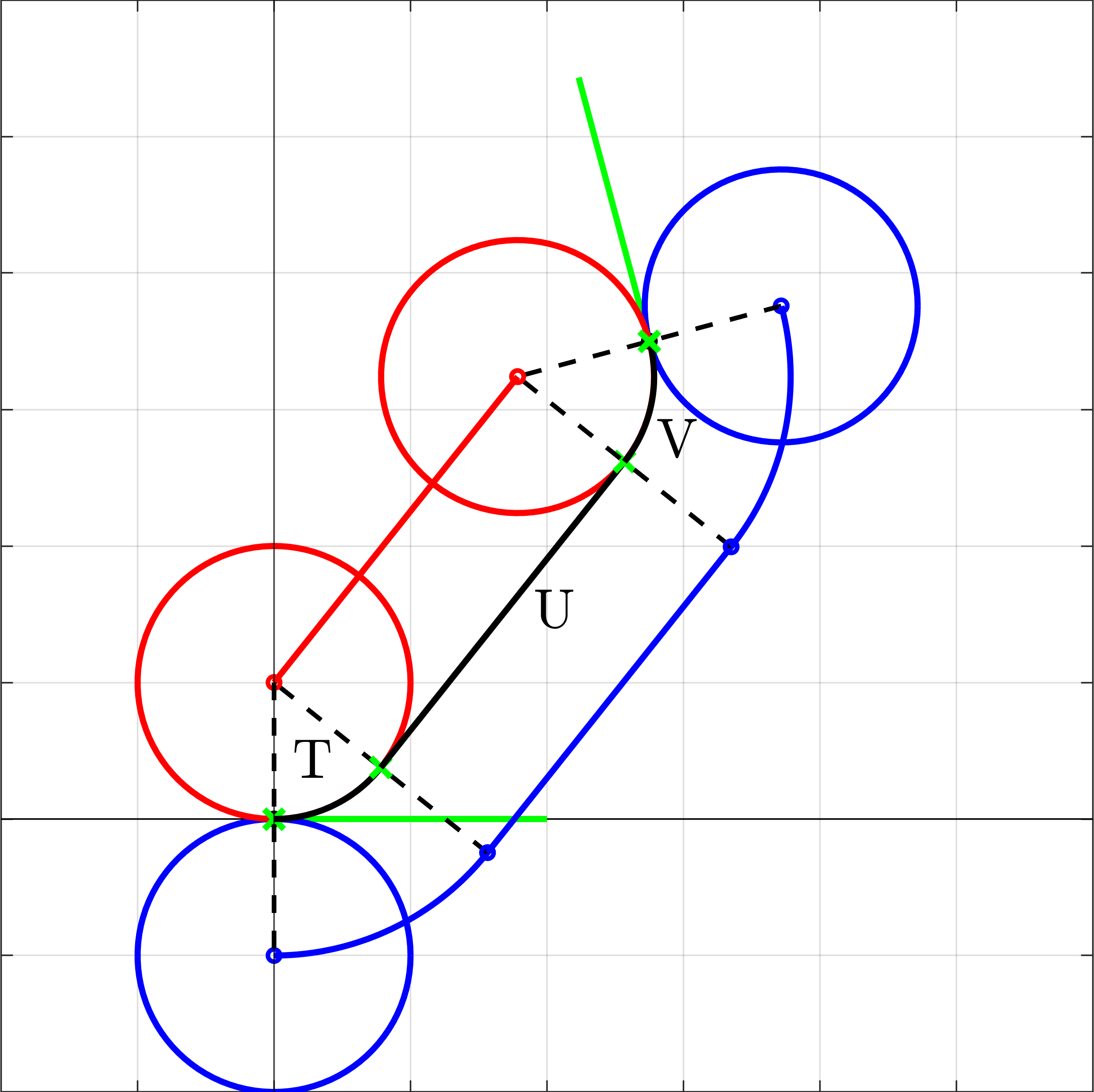}
	\end{subfigure}\hfil
	\begin{subfigure}{}
		\includegraphics[width=4.25cm]{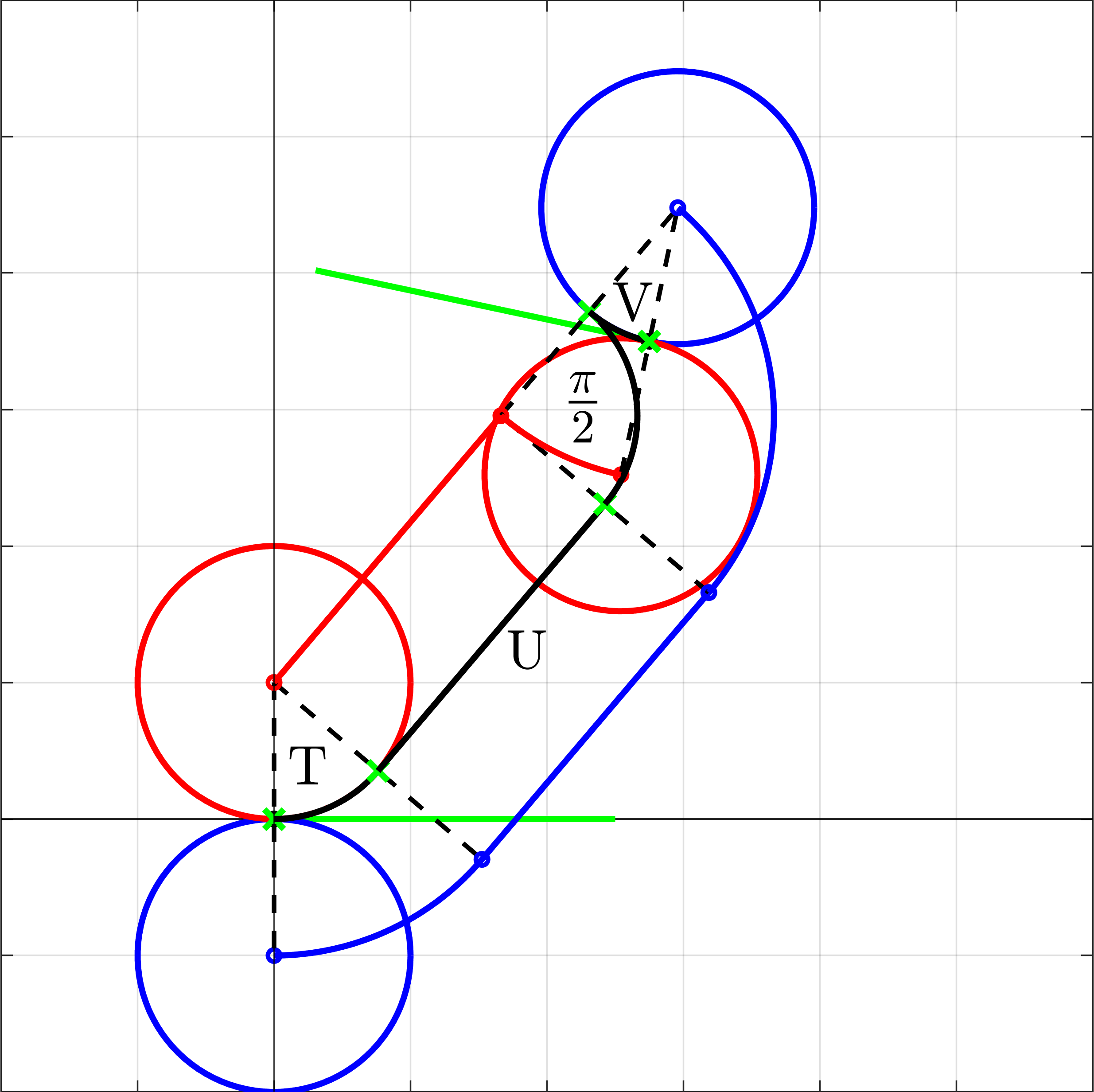}
	\end{subfigure}\hfil
	\caption{$P_{2}$: $l_{T}^{+}s_{U}^{+}l_{V}^{+}$ (left) where $\angle L_{f}L_{0} < \theta_{f} \leq{} \angle L_{f}L_{0} + \frac{\pi}{2}$, \& $P_{3}$: $l_{T}^{+}s_{U}^{+}l_{\frac{\pi}{2}}^{+}r_{V}^{-}$ (right) where $\theta_{f} > \angle L_{f}L_{0} + \frac{\pi}{2}$.}\label{fig:cond_2_3}
\end{figure}

Similar to $P_{2}$, $P_{1}$ is obtained by introducing an additional right turn primitive $r_{V}^{+}$ at the end of $l_{\Omega}^{+}s_{U}^{+}$.
\begin{proposition}\label{prop:p1}
	For path type $P_{1}$ of the form $l_{T}^{+}s_{U}^{+}r_{V}^{+}$ with $V \neq 0$, $\theta_{f} < |\angle L_{f}L_{0}|$.
\end{proposition}
\begin{proof}
	LHC rotates about RHC's center, $c_{f_{R}}$. As such, $\theta_{f}$ decreases below $\angle L_{f}L_{0}$.
\end{proof}

The absolute value is to account for the sign change in $\angle L_{f}L_{0}$ when $c_{f_{L_{y}}} < c_{0_{L_{y}}}$. We illustrate an example of $P_{1}$ in Fig.~\ref{fig:cond_1_4} to the left.

Similar to $P_{2}$ and $P_{3}$, as $\theta_{f}$ decreases, the central angle of the third arc segment in $P_{1}$ saturates to $\frac{\pi}{2}$, and a new motion primitive is required to accomodate for a further decrease in $\theta_{f}$. Following similar reasoning, the next motion primitive has to be a left turn of negative direction, since a positive direction left turn increases $\theta_{f}$ instead. Therefore, the path must be $l_{T}^{+}s_{U}^{+}r_{\frac{\pi}{2}}^{+}l_{V}^{-}$. We illustrate such a path in Fig.\ref{fig:cond_1_4} to the right.
Moreover, we illustrate the case when the third arc segment in $P_{1}$ is $\frac{\pi}{2}$ in Fig.~\ref{fig:cond_1_4_t1}.
We note that the line $L$ defined by the point $(x_{f}, y_{f})$ and direction $\theta_{f}$ is perpendicular to the straight line primitive $S$ that is tangent to $c_{0_{L}}$ and $c_{f_{R}}$. This is trivially due to the difference of $\frac{\pi}{2}$ between the two primitives. Let $t_{1}$ be the signed distance along $L$ between the projection of $c_{0_{L}}$ onto $L$, $c_{0_{L}}^{'}$, and $(x_{f}, y_{f})$. We make the following proposition that allows us to distinguish partitions $P_{1}$ and $P_{4}$ in the negative range of $\theta_{f}$:
\begin{proposition}\label{prop:p1_vs_p4}
	For path type $P_{4}$ of the form $l_{T}^{+}s_{U}^{+}r_{\frac{\pi}{2}}^{+}l_{V}^{-}$ with $V \neq 0$, $\left(t_{1} > -2r\right) \land \left(\angle R_{f}L_{0} = \atantwo{}\left(c_{f_{R_{y}}}-c_{0_{L_{y}}}, c_{f_{R_{x}}}-c_{0_{L_{x}}}\right) > \theta_{f}\right)$.
\end{proposition}
\begin{proof}
	When the last arc segment in $P_{1}$ saturates to $\frac{\pi}{2}$, the distance between $c_{0_{L}}^{'}$ and $(x_{f}, y_{f})$ is equal to $2r$. The distance between $c_{f_{R}}$ and $(x_{f}, y_{f})$ is always $r$ and the line defined by the two points is always perpendicular to $L$. Moreover, when $\theta_{f}$ further decreases, $c_{f_{R}}$ rotates clockwise and the distance between $c_{f_{R}}$ and $c_{0_{L}}^{'}$ decreases. Therefore, following the Pythagorean theorem, $t_{1}$ must decrease in magnitude. The angle $\angle R_{f}L_{0}$ is the angle defined by the two points $c_{f_{R}}$ and $c_{0_{L}}$. For a path of the form $l_{T}^{+}r_{\frac{\pi}{2}}^{+}$, $\angle R_{f}L_{0}$ is always aligned with $\theta_{f}$. Both the straight line primitive $S$ in between the two arc segments and the arc segment $l_{V}^{-}$ at the end of $P_{4}$ contribute to increasing $\angle R_{f}L_{0}$ beyond $\theta_{f}$.
\end{proof}

\begin{figure}[!ht]
	\centering
	\begin{subfigure}{}
		\includegraphics[width=4.25cm]{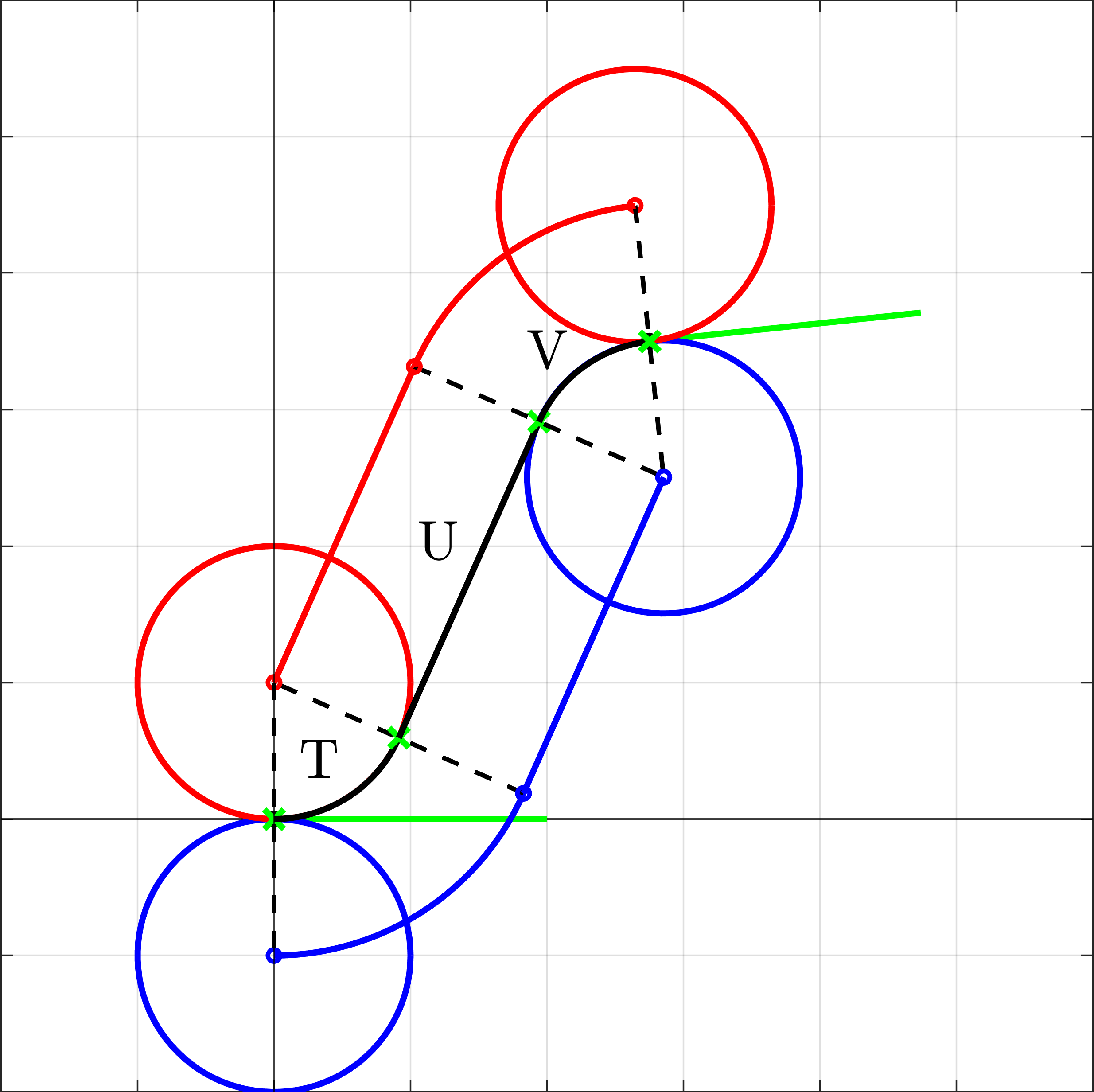}
	\end{subfigure}\hfil
	\begin{subfigure}{}
		\includegraphics[width=4.25cm]{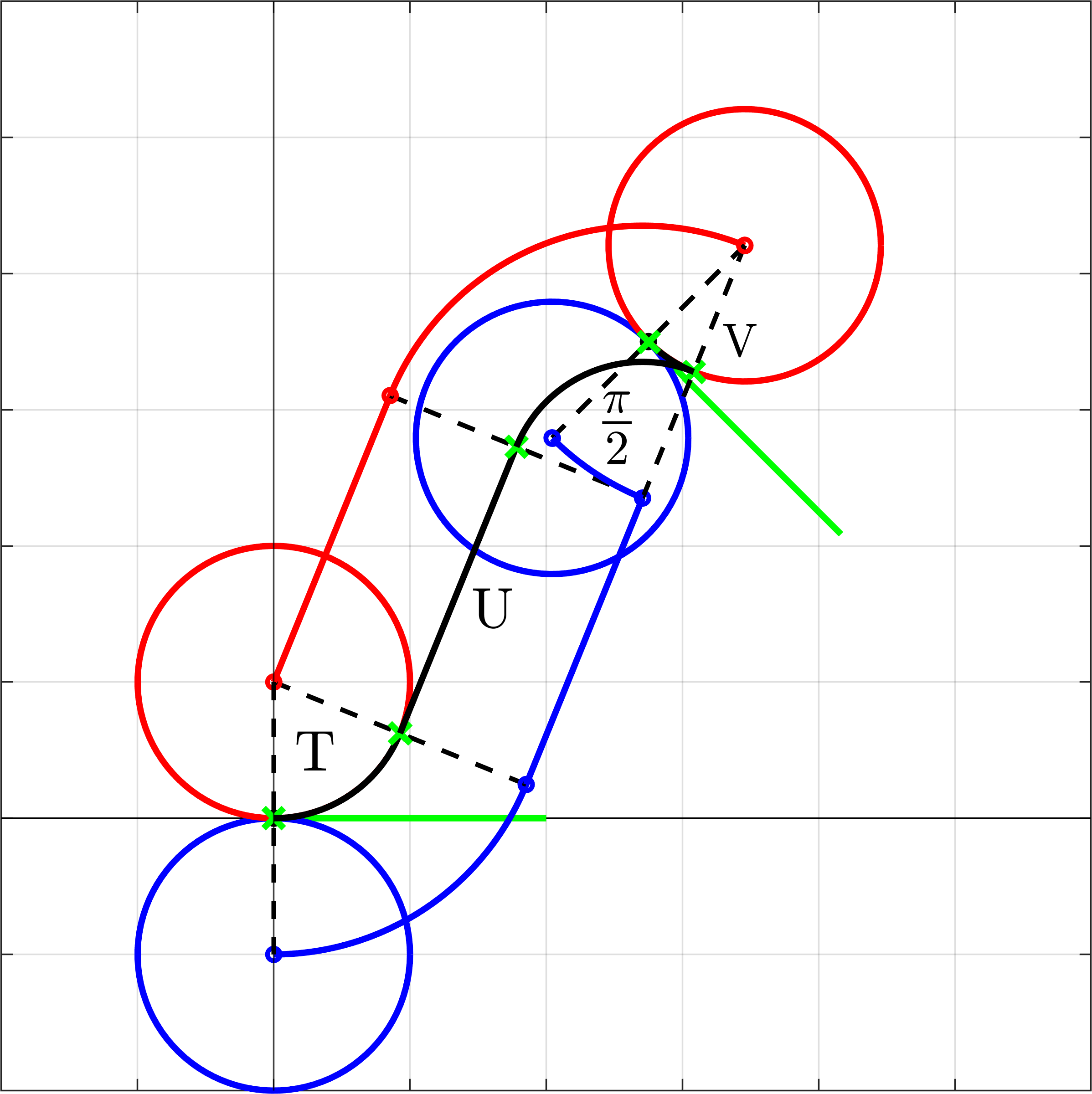}
	\end{subfigure}\hfil
	\caption{$P_{1}$: $l_{T}^{+}s_{U}^{+}r_{V}^{+}$ (left), \& $P_{4}$: $l_{T}^{+}s_{U}^{+}r_{\frac{\pi}{2}}^{+}l_{V}^{-}$ (right).}\label{fig:cond_1_4}
\end{figure}

\begin{figure}[!ht]
	\centering
	\includegraphics[width=8.5cm,keepaspectratio]{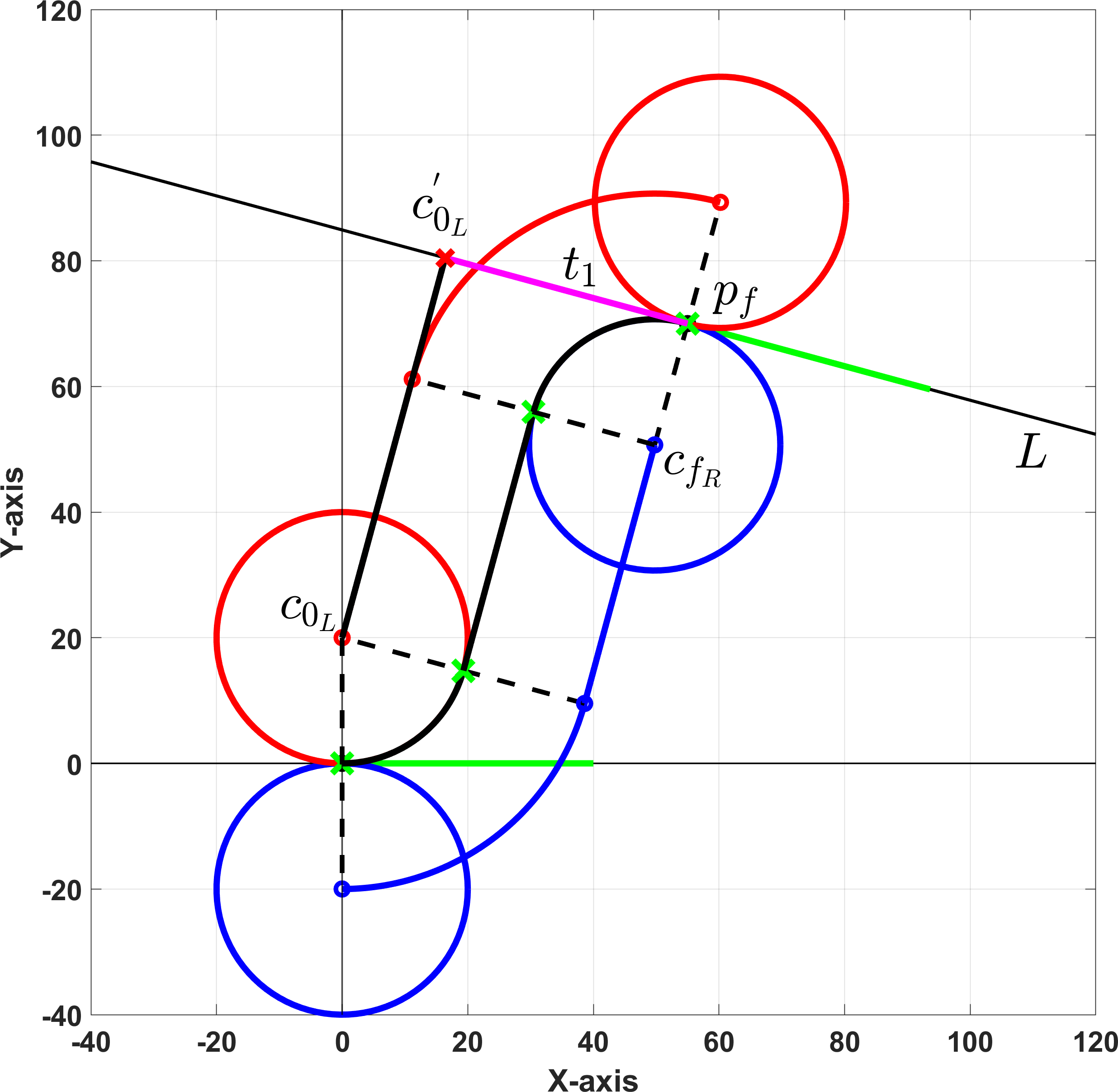}
	\caption{Plot showing the path $P_{1}$ with the third primitive saturated to $\frac{\pi}{2}$, $l_{T}^{+}s_{U}^{+}r_{\frac{\pi}{2}}^{+}$. $L$ is the line defined by the point $(x_{f}, y_{f})$ and angle $\theta_{f}$. $c_{0_{L}}^{'}$ is the projection of $c_{0_{L}}$ onto line $L$ and $t_{1}$ is the distance between $c_{0_{L}}^{'}$ and $(x_{f}, y_{f})$. $t_{1}$ is marked in magenta.}\label{fig:cond_1_4_t1}
\end{figure}

The conditions that define the subpartitions for path types $P_{7}$ and $P_{8}$ follow from the reasoning that we present in Lemma~\ref{lem:third}.
\begin{lemma}\label{lem:third}
	Given a path $l_{T}^{+}s_{U}^{+}$ that starts from $p_{0} = (0, 0, 0)$, $\forall{}~(0 < T \leq{} \frac{\pi}{2})$ and $\forall U$,  the final position $(x_{f}, y_{f}) \in Q_{1}$. Similarily, given a path $r_{T}^{+}s_{U}^{+}$ that starts from $p_{0}$, $\forall{}~(0 < T \leq{} \frac{\pi}{2})$ and $\forall U$, the final position $(x_{f}, y_{f}) \in Q_{4}$.
\end{lemma}
\begin{proof}
	Let the initial configuration be \( p_0 = (0, 0, 0) \). For \( l_T^+ s_U^+ \), a left turn \( l_T^+ \) changes the orientation to \( T \) with \( 0 < T \leq \frac{\pi}{2} \). Straight motion \( s_U^+ \) then results in \( (x_f, y_f) = (U \cos T, U \sin T) \), where \( \cos T > 0 \) and \( \sin T > 0 \), so \( (x_f, y_f) \in Q_1 \). For \( r_T^+ s_U^+ \), a right turn \( r_T^+ \) changes the orientation to \( -T \). Straight motion \( s_U^+ \) then results in \( (x_f, y_f) = (U \cos T, -U \sin T) \), where \( \cos T > 0 \) and \( -\sin T < 0 \), so \( (x_f, y_f) \in Q_4 \). Thus, the final position lies in the respective quadrant as stated in the lemma.
\end{proof}

As such, path types $P_{2}$ and $P_{3}$ beginning with $l_{T}^{+}s_{U}^{+}$ have the end point of the first two primitives in $Q_{1}$. The y-abscissa of such a point must be greater than or equal to $c_{f_{L_{y}}}-r$ and $c_{f_{R_{y}}}-r$ respectively so that it remains in $Q_{1}$. Otherwise, this point crosses into $Q_{4}$, and based on Lemma~\ref{lem:third}, the optimal paths must then begin with $r_{T}^{+}s_{U}^{+}$.

We introduce the two paths types $P_{7}$ and $P_{8}$ that begin with $r_{T}^{+}s_{U}^{+}$. $P_{7}$ and $P_{8}$ differ from $P_{2}$ and $P_{3}$, respectively, only in the first primitive, which is a positive right turn instead of a positive left turn. Similar to $P_{2}$ and $P_{3}$, there is a transition from $P_{7}$ to $P_{8}$ when the third primitive saturates to $\frac{\pi}{2}$ so as to accomodate a further increase in $\theta_{f}$. Sample $P_{7}$ and $P_{8}$ paths are illustrated in the left and right sides of Fig.~\ref{fig:cond_7_8} respectively.

\begin{figure}[!ht]
	\centering
	\begin{subfigure}{}
		\includegraphics[width=4.25cm]{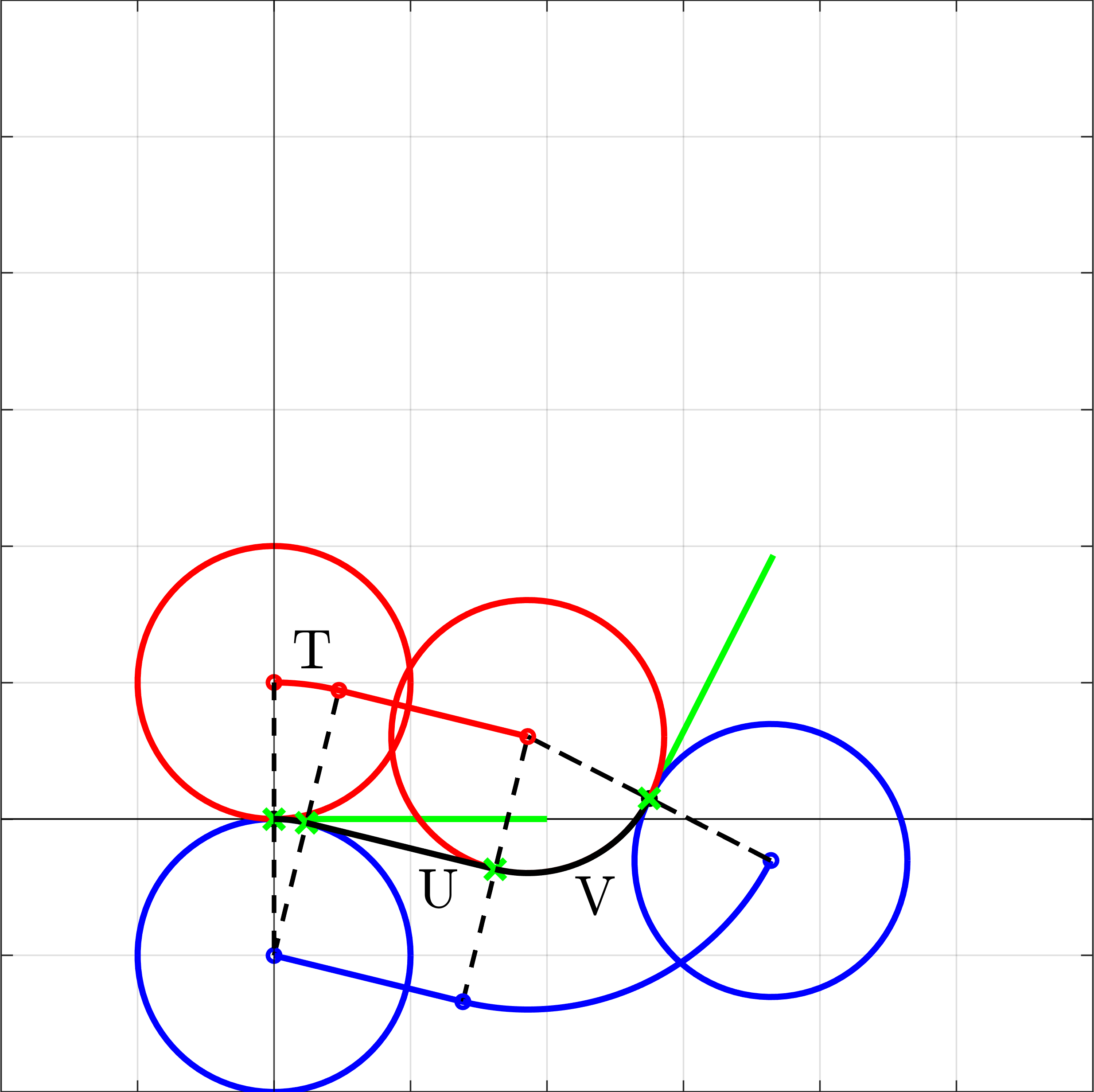}
	\end{subfigure}\hfil
	\begin{subfigure}{}
		\includegraphics[width=4.25cm]{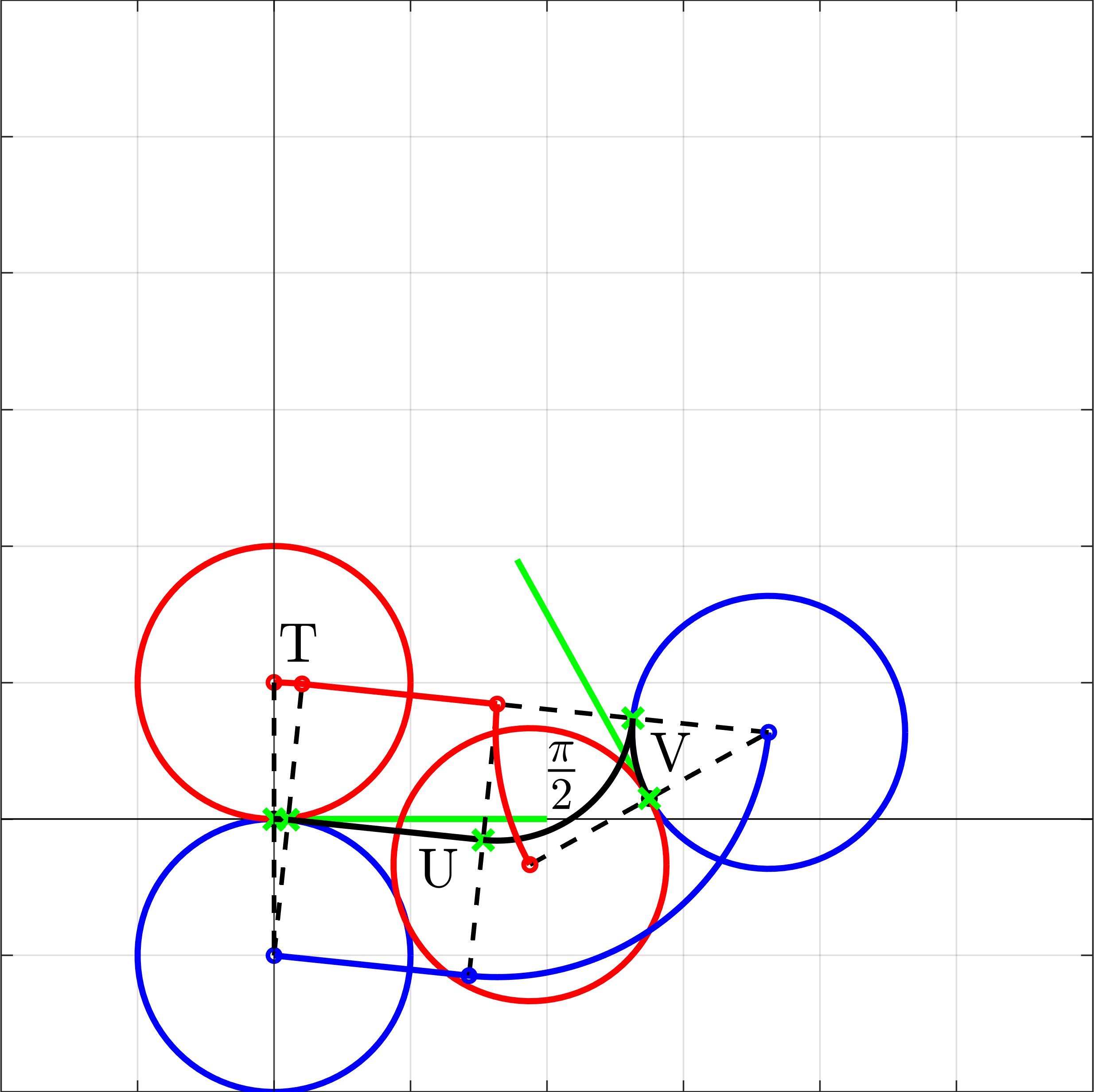}
	\end{subfigure}\hfil
	\caption{$P_{7}$: $r_{T}^{+}s_{U}^{+}l_{V}^{+}$ (left), \& $P_{8}$: $r_{T}^{+}s_{U}^{+}l_{\frac{\pi}{2}}^{+}r_{V}^{-}$ (right).}\label{fig:cond_7_8}
\end{figure}

A similar approach is adopted to subpartition between $P_{7}$ and $P_{8}$ as was done between $P_{1}$ and $P_{4}$. In particular, if the signed distance $t_{2}$ between the projection of $c_{0_{R}}$ onto $L$, $c_{0_{R}}^{'}$, and $(x_{f}, y_{f})$ is greater than $-2r$, then $P_{8}$ is the optimal path. We illustrate the case when $t_{2} = -2r$ in Fig.~\ref{fig:cond_7_8_t2}.
Let $d_{1}$ be the distance between $c_{0_{R}}^{'}$ and $c_{0_{R}}$. Following the reasoning introduced in Proposition~\ref{prop:p1_vs_p4}, we make the proposition:
\begin{proposition}\label{prop:p7_vs_p8}
	For path type $P_{8}$ of the form $r_{T}^{+}s_{U}^{+}l_{\frac{\pi}{2}}^{+}r_{V}^{-}$ with $V \neq 0$, $t_{2} > -2r$ and $d_{1} > r$.
\end{proposition}
\begin{proof}
	The proof follows the same reasoning as the proof of Proposition~\ref{prop:p1_vs_p4} with the addition of $d_{1} > r$. The latter comes from the fact that if $d_{1} \leq{} r$, then the distance between $c_{0_{R}}$ and $c_{f_{R}}$ is less than or equal to $\mathcal{K}$, making the optimal path fall in set B instead of set A\@ according to Proposition~\ref{prop:booleanset}. $d_{1}$ can be deduced from the similar triangles formed by the four points $c_{0_{R}}$, $c_{0_{R}}^{'}$, $c_{f_{R}}$, and $(x_{f}, y_{f})$.
\end{proof}

\begin{figure}[!ht]
	\centering
	\includegraphics[width=8.5cm,keepaspectratio]{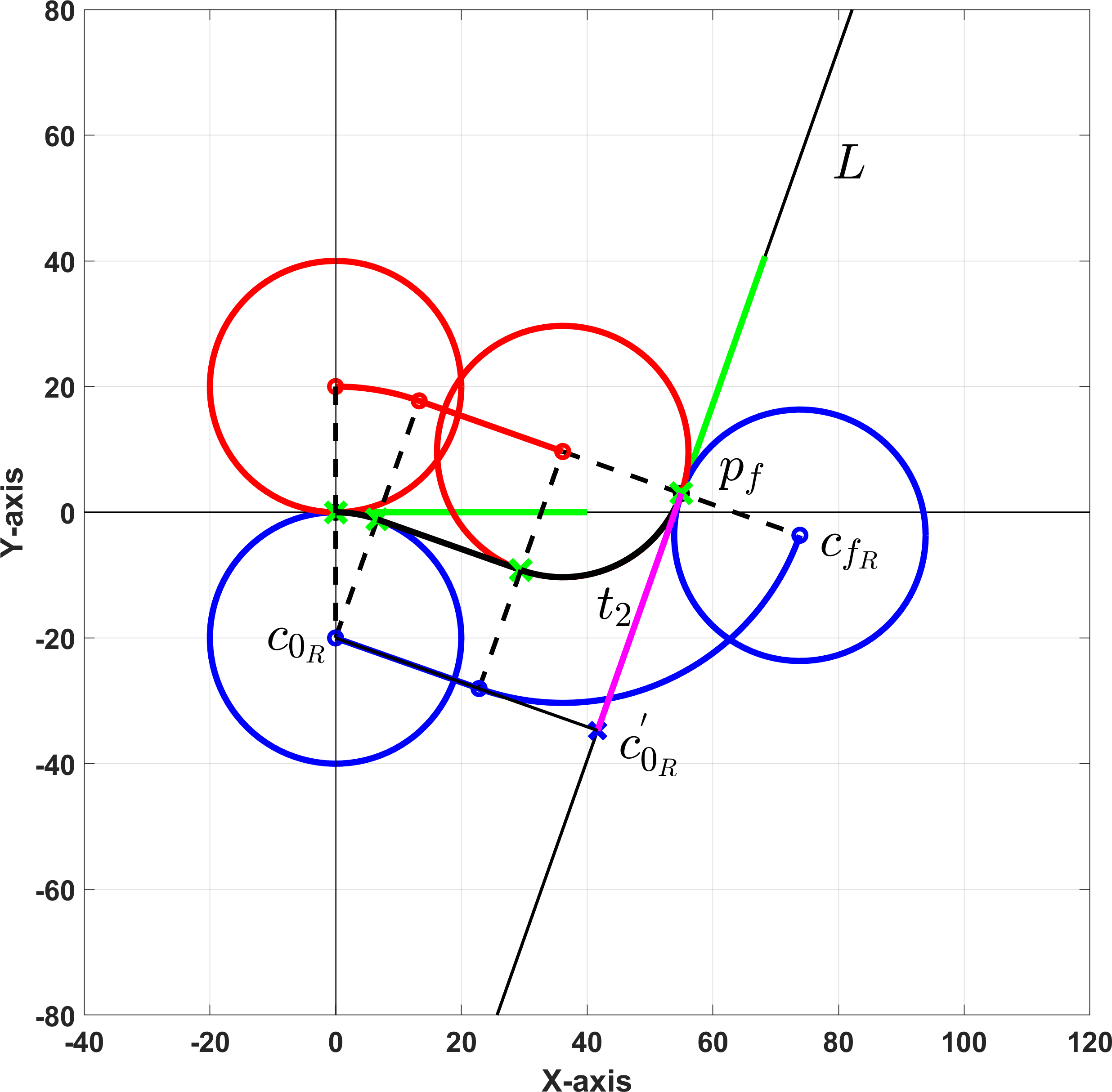}
	\caption{Plot showing the path $P_{7}$ with the third primitive saturated to $\frac{\pi}{2}$, $r_{T}^{+}s_{U}^{+}l_{\frac{\pi}{2}}^{+}$. $L$ is the line defined by the point $(x_{f}, y_{f})$ and angle $\theta_{f}$. $c_{0_{R}}^{'}$ is the projection of $c_{0_{R}}$ onto line $L$ and $t_{2}$ is the distance between $c_{0_{R}}^{'}$ and $(x_{f}, y_{f})$. $t_{2}$ is marked in magenta.}\label{fig:cond_7_8_t2}
\end{figure}

Path types $P_{5}$, $r_{T}^{+}l_{\frac{\pi}{2}}^{-}s_{U}^{-}r_{V}^{-}$, and $P_{6}$, $r_{T}^{+}l_{\frac{\pi}{2}}^{-}s_{U}^{-}l_{V}^{-}$, follow after $P_{4}$ as $\theta_{f}$ keeps decreasing in the negative range $\theta_{f} < 0$. We write the following proposition:
\begin{proposition}\label{prop:p5_vs_p6}
	For path types $P_{5}$ and $P_{6}$, $\theta_{f} < 2\beta_{0}-\pi$, where $\beta_{0} = \atantwo{\left(y_{f}-y_{0}, x_{f}-x_{0}\right)}$.
\end{proposition}
\begin{proof}
	It can be shown using symbolic solvers that at $\theta_{f} = 2\beta_{0}-\pi$, $d_{4}-d_{5}$ evaluates to zero, where $d_{4}$ and $d_{5}$ are the total path distances obtained with equations for $P_{4}$ and $P_{5}$ respectively. For $\theta_{f}$ less than $2\beta_{0}-\pi$, $d_{4}-d_{5}$ is positive, making $P_{5}$ the optimal path. Therefore, the point of transition from $P_{4}$ to $P_{5}$ is at $\theta_{f} = 2\beta_{0}-\pi$. Moreover, $2\beta_{0}-\pi$ is a solution to the equality $\angle L_{f}R_{0} = \angle R_{f}L_{0}$, where $\angle L_{f}R_{0} = \atantwo{\left(c_{f_{L_{y}}}-c_{0_{R_{y}}}, c_{f_{L_{x}}}-c_{0_{R_{x}}}\right)}$. Extensive numerical simulations further support Proposition~\ref{prop:p5_vs_p6}.
\end{proof}

We illustrate $P_{5}$ and $P_{6}$ in Fig.~\ref{fig:cond_5_6}. Reasoning to subpartition between the two is similar to the reasoning between $P_{1}$ and $P_{2}$:

\begin{figure}[]
	\centering
	\begin{subfigure}{}
		\includegraphics[width=4.25cm]{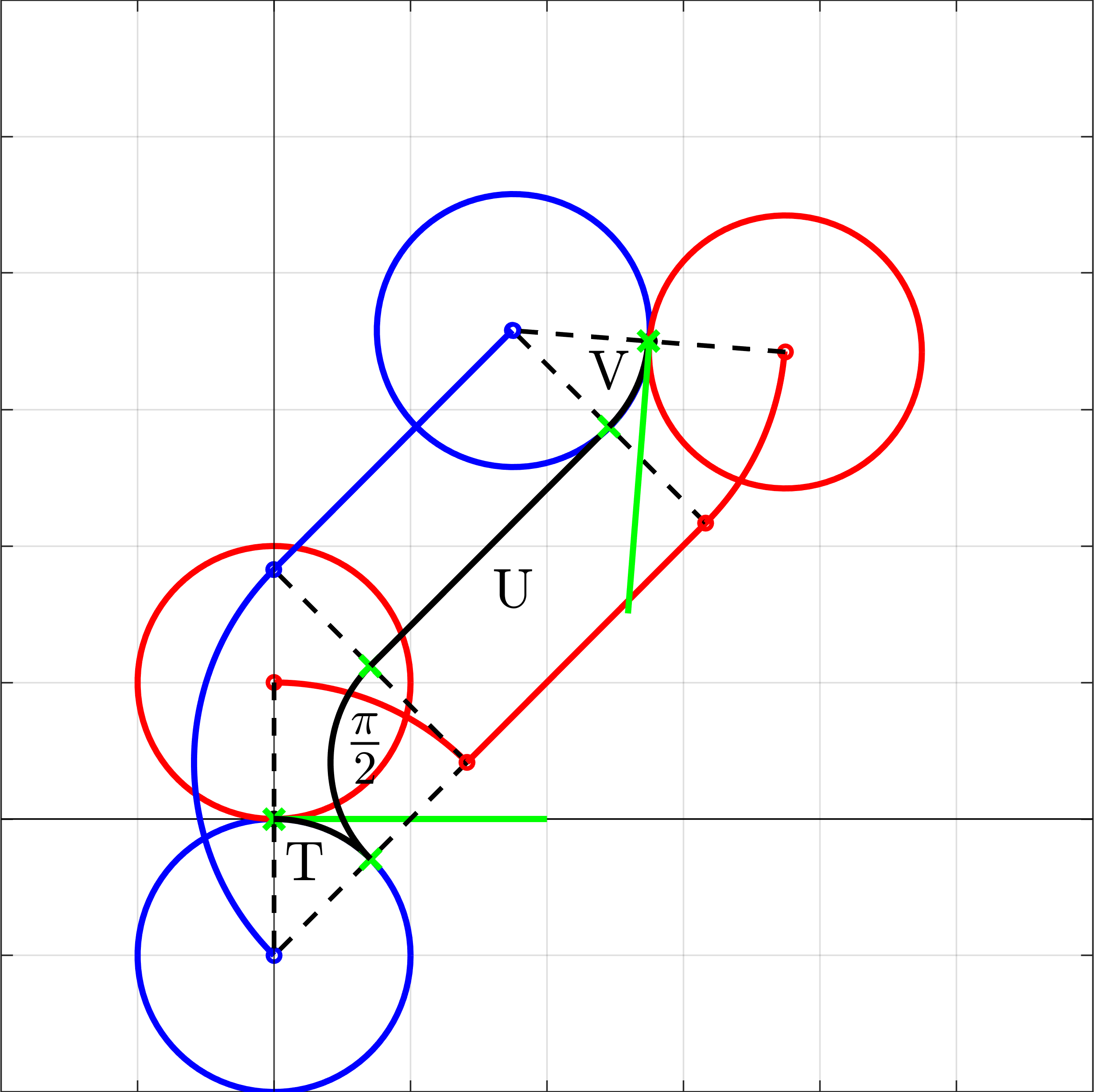}
	\end{subfigure}\hfil
	\begin{subfigure}{}
		\includegraphics[width=4.25cm]{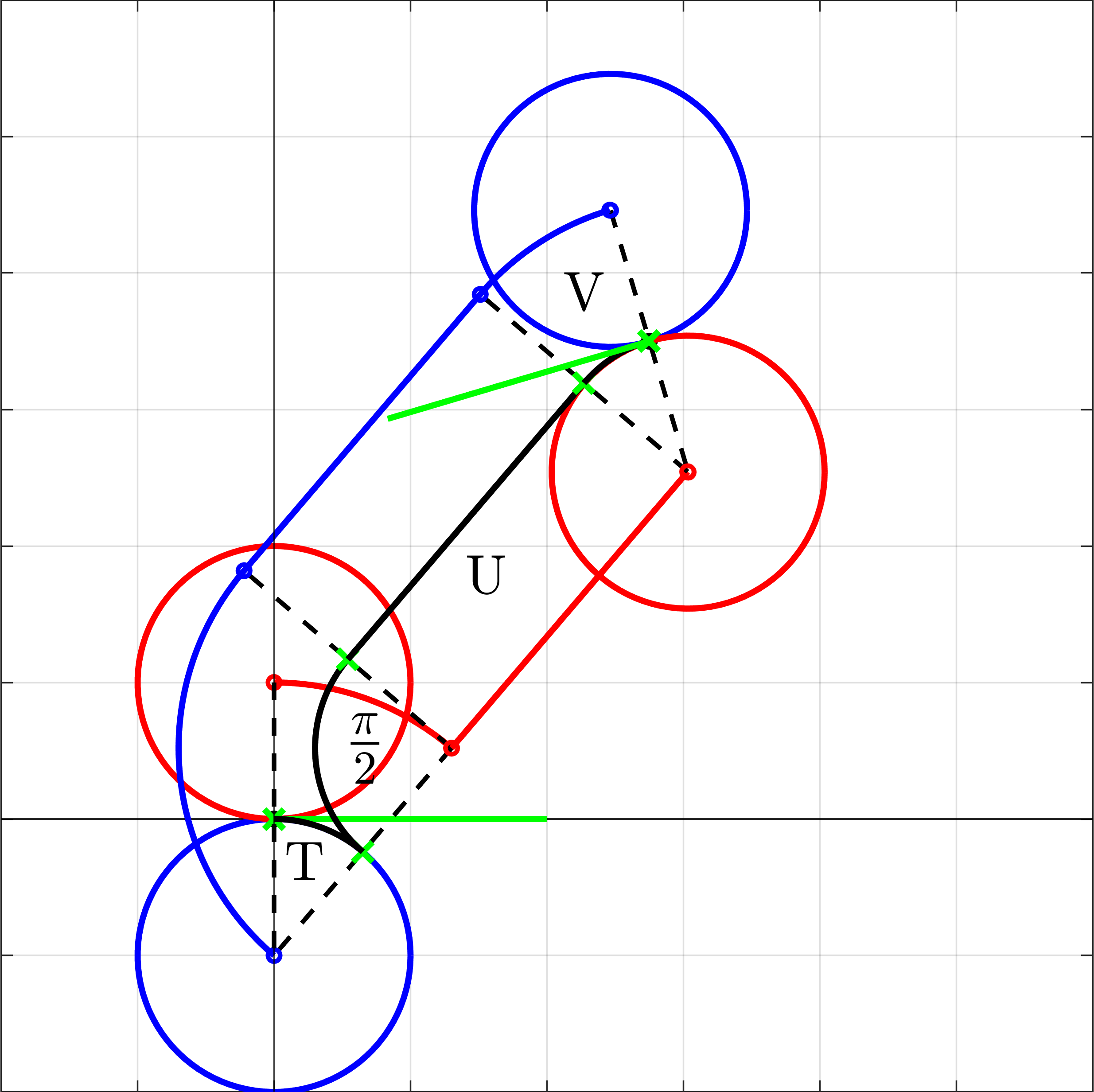}
	\end{subfigure}\hfil
	\caption{$P_{5}$: $r_{T}^{+}l_{\frac{\pi}{2}}^{-}s_{U}^{-}r_{V}^{-}$ (left), \& $P_{6}$: $r_{T}^{+}l_{\frac{\pi}{2}}^{-}s_{U}^{-}l_{V}^{-}$ (right).}\label{fig:cond_5_6}
\end{figure}

\begin{proposition}\label{prop:p6}
	For path type $P_{6}$, $\theta_{f} < \angle R_{0}L_{f} = \atantwo{\left(c_{0_{R_{y}}}-c_{f_{L_{y}}}, c_{0_{R_{x}}}-c_{f_{L_{x}}}\right)}$.
\end{proposition}
\begin{proof}
	We illustrate in Fig.~\ref{fig:cond_5_6_split} the case where the last primitive differentiating $P_{5}$ and $P_{6}$ has a length of zero, i.e., $V = 0$, making the path $r_{T}^{+}l_{\frac{\pi}{2}}^{-}s_{U}^{-}$. In this case, the line connecting $c_{f_{L}}$ and $c_{0_{R}}$ is parallel to the $S$ primitive in the path, making it also aligned with $\theta_{f}$. $\theta_{f}$ is therefore equal to $\angle R_{0}L_{f}$. An additional left turn primitive $l_{V}^{-}$ at the end of this path results in $\theta_{f}$ decreasing below $\angle R_{0}L_{f}$.
\end{proof}

\begin{figure}[!ht]
	\centering
	\includegraphics[width=8.5cm,keepaspectratio]{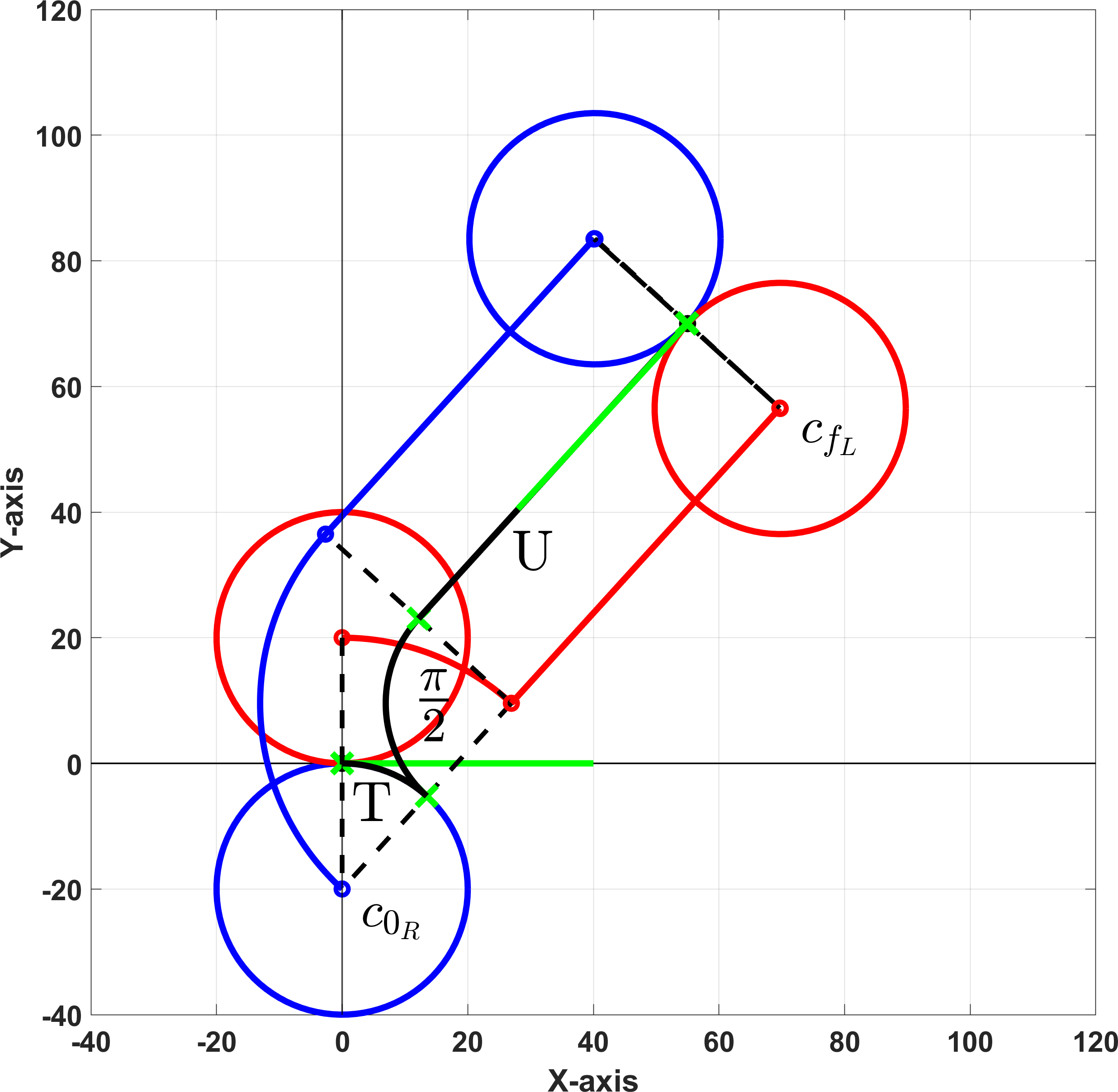}
	\caption{The path consisting of the first three primitives that are common to path types $P_{5}$ and $P_{6}$, with the last primitive having a length $V = 0$. $\theta_{f}$ is aligned with the line connecting $c_{f_{L}}$ and $c_{0_{R}}$.}\label{fig:cond_5_6_split}
\end{figure}

The last two path types that are exclusive to set A are $P_{10}$, $r_{T}^{-}l_{\frac{\pi}{2}}^{+}s_{U}^{+}r_{V}^{+}$, and $P_{11}$, $r_{T}^{-}l_{\frac{\pi}{2}}^{+}s_{U}^{+}l_{V}^{+}$. We illustrate them in Fig.~\ref{fig:cond_10_11} to the left and right respectively. Those two path types are the same as $P_{5}$ and $P_{6}$ but with the directions of primitives reversed. Therefore, similar reasoning that follows from Proposition~\ref{prop:p6} and its proof applies to $P_{10}$ and $P_{11}$. We make the following proposition:
\begin{proposition}\label{prop:p11}
	For path type $P_{11}$, $\theta_{f} > \angle L_{f}R_{0}$.
\end{proposition}
\begin{proof}
	The proof follows similar reasoning as the proof of Proposition~\ref{prop:p6} and has been omitted for brevity.
\end{proof}

\begin{figure}[]
	\centering
	\begin{subfigure}{}
		\includegraphics[width=4.25cm]{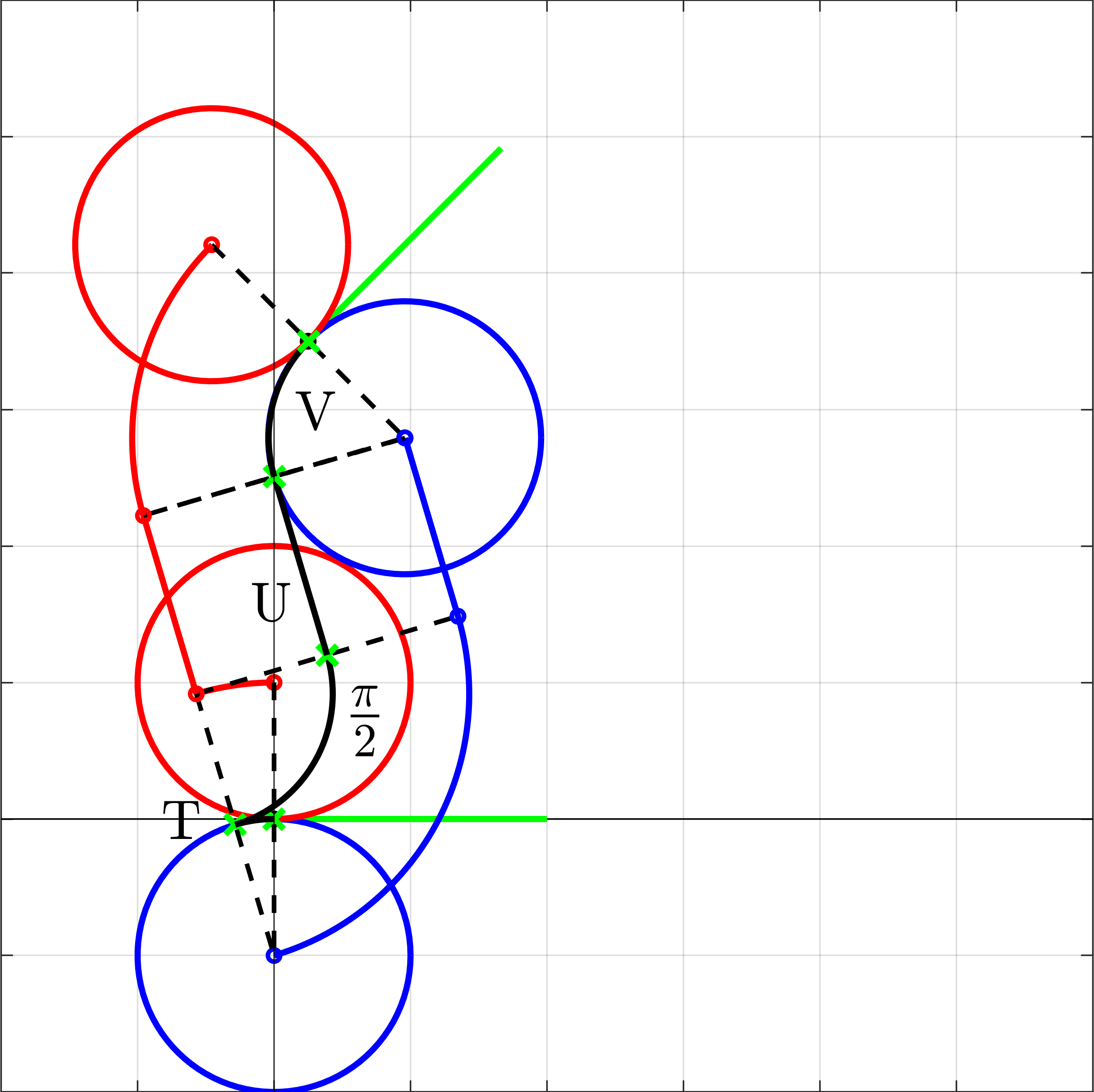}
	\end{subfigure}\hfil
	\begin{subfigure}{}
		\includegraphics[width=4.25cm]{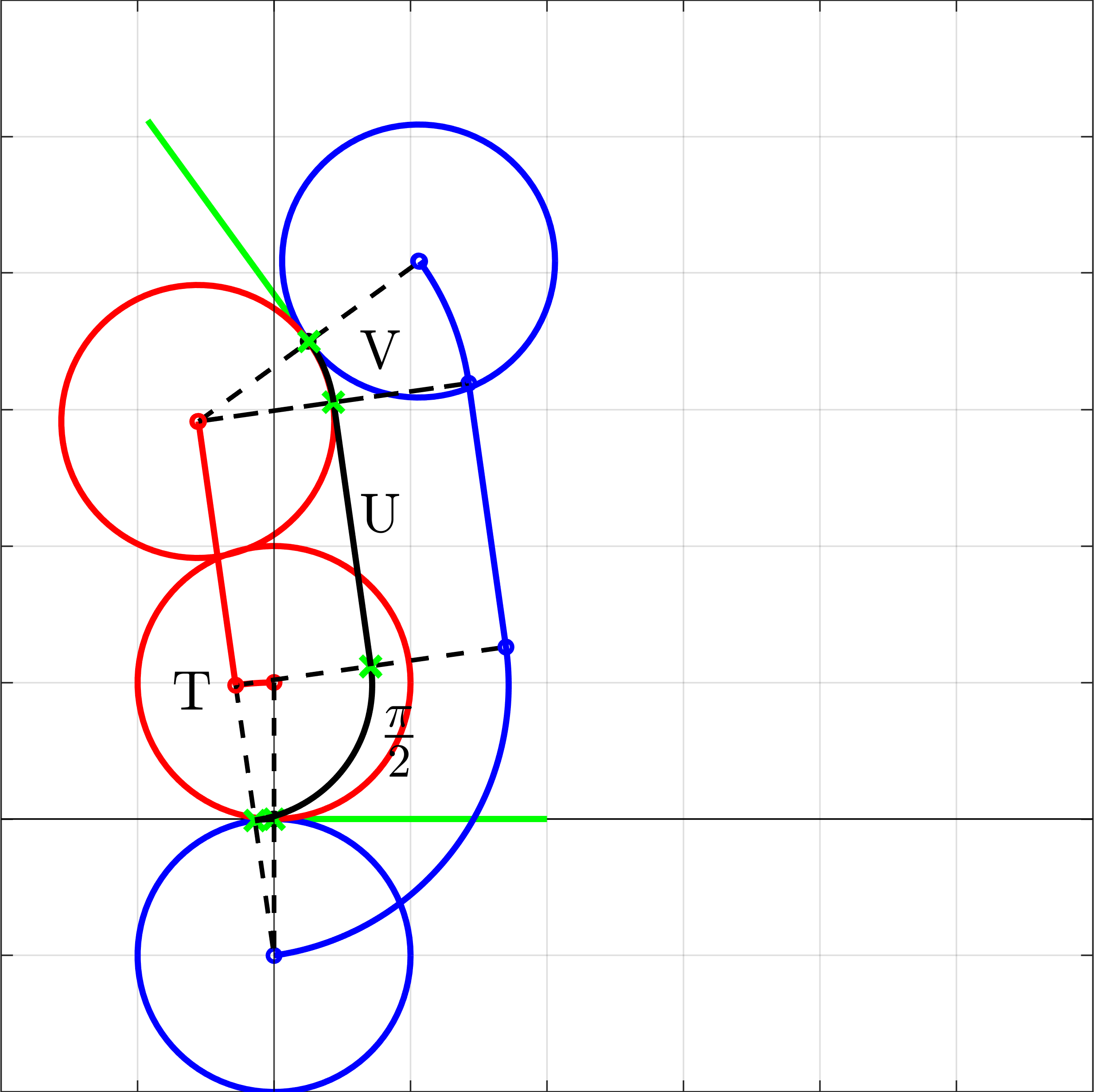}
	\end{subfigure}\hfil
	\caption{$P_{10}$: $r_{T}^{-}l_{\frac{\pi}{2}}^{+}s_{U}^{+}r_{V}^{+}$ (left), \& $P_{11}$: $r_{T}^{-}l_{\frac{\pi}{2}}^{+}s_{U}^{+}l_{V}^{+}$ (right).}\label{fig:cond_10_11}
\end{figure}

Path type $P_{10}$, $r_{T}^{-}l_{\frac{\pi}{2}}^{+}s_{U}^{+}r_{V}^{+}$, differs from $P_{1}$, $l_{T}^{+}s_{U}^{+}r_{V}^{+}$, only in the first primitive. In order to avoid any overlap between the two subpartitions, we make the following proposition:
\begin{proposition}\label{prop:p10_vs_p1}
	For path type $P_{10}$, $\left(c_{f_{R_{x}}} < 2r\right) \land{} \left(c_{f_{R_{y}}} > c_{0_{L_{y}}}\right)$.
\end{proposition}
\begin{proof}
	Since the second primitive in $P_{10}$ is $l_{\frac{\pi}{2}}^{+}$, based on Lemma~\ref{lem:first}, and based on the fact that the last primitive in $P_{10}$ is a positive right turn, then $\forall$ $(0 < T \leq{} \frac{\pi}{2})$ in $P_{10}$, $c_{f_{R_{x}}} < 2r$ and $c_{f_{R_{y}}} > c_{0_{L_{y}}}$.
\end{proof}

Similarily, path type $P_{11}$, $r_{T}^{-}l_{\frac{\pi}{2}}^{+}s_{U}^{+}l_{V}^{+}$, differs from $P_{2}$, $l_{T}^{+}s_{U}^{+}l_{V}^{+}$, only in the first primitive. In order to avoid any overlap between the two subpartitions, we make the following proposition:
\begin{proposition}\label{prop:p11_vs_p2}
	For path type $P_{11}$, $c_{f_{L_{x}}} < 0$.
\end{proposition}
\begin{proof}
	Since the second primitive in $P_{11}$ is $l_{\frac{\pi}{2}}^{+}$, based on Lemma~\ref{lem:first}, and based on the fact that the last primitive in $P_{11}$ is a positive left turn, then $\forall$ $(0 < T \leq{} \frac{\pi}{2})$ in $P_{11}$, $c_{f_{L_{x}}} < 0$.
\end{proof}

Path types $P_{9}$ and $P_{12}$ are common to both sets A and B\@. We illustrate them in Fig.~\ref{fig:cond_9_12} to the left and right respectively.
$P_{9}$, $r_{T}^{-}l_{\frac{\pi}{2}}^{+}s_{U}^{+}r_{\frac{\pi}{2}}^{+}l_{V}^{-}$, differs from $P_{10}$, $r_{T}^{-}l_{\frac{\pi}{2}}^{+}s_{U}^{+}r_{V}^{+}$, only in the last primitive. We make the following proposition to subpartition between the two path types:
\begin{proposition}\label{prop:p9_vs_p10}
	For path type $P_{9}$, $|t_{2}| < 2r$.
\end{proposition}
\begin{proof}
	The proof follows similar reasoning as the proof of Proposition~\ref{prop:p1_vs_p4} and has been omitted for brevity.
\end{proof}
$P_{12}$, $r_{T}^{+}l_{\frac{\pi}{2}}^{-}s_{U}^{-}r_{\frac{\pi}{2}}^{-}l_{V}^{+}$, differs from $P_{5}$, $r_{T}^{+}l_{\frac{\pi}{2}}^{-}s_{U}^{-}r_{V}^{-}$, only in the last primitive. We make the following proposition to subpartition between the two path types:
\begin{proposition}\label{prop:p12_vs_p5}
	For path type $P_{12}$, $|t_{2}| < 2r$.
\end{proposition}
\begin{proof}
	The proof follows similar reasoning as the proof of Proposition~\ref{prop:p1_vs_p4} and has been omitted for brevity.
\end{proof}

The last overlap in set A that we need to address is between $P_{4}$, $l_{T}^{+}s_{U}^{+}r_{\frac{\pi}{2}}^{+}l_{V}^{-}$, and $P_{9}$, $r_{T}^{-}l_{\frac{\pi}{2}}^{+}s_{U}^{+}r_{\frac{\pi}{2}}^{+}l_{V}^{-}$, since the two paths differ only in the first primitive. We make the following proposition:
\begin{proposition}\label{prop:p4_vs_p9}
	For path type $P_{4}$, $c_{f_{L_{x}}} \geq{} 2r$.
\end{proposition}
\begin{proof}
	Since $P_{4} \in A$, the condition $LL = \twonorm{c_{0_{L}}-c_{f_{L}}} > 2r$ holds. Moreover, since the first primitive in $P_{4}$ is $l_{T}^{+}$, and the third primitive is $r_{\frac{\pi}{2}}^{+}$, then $c_{f_{L_{x}}} \geq{} 2r$ must hold $\forall$ $(0 < T \leq{} \frac{\pi}{2}),~U,~(0 < V \leq{} \frac{\pi}{2})$ in $l_{T}^{+}s_{U}^{+}r_{\frac{\pi}{2}}^{+}l_{V}^{-}$.
\end{proof}

\begin{figure}[]
	\centering
	\begin{subfigure}{}
		\includegraphics[width=4.25cm]{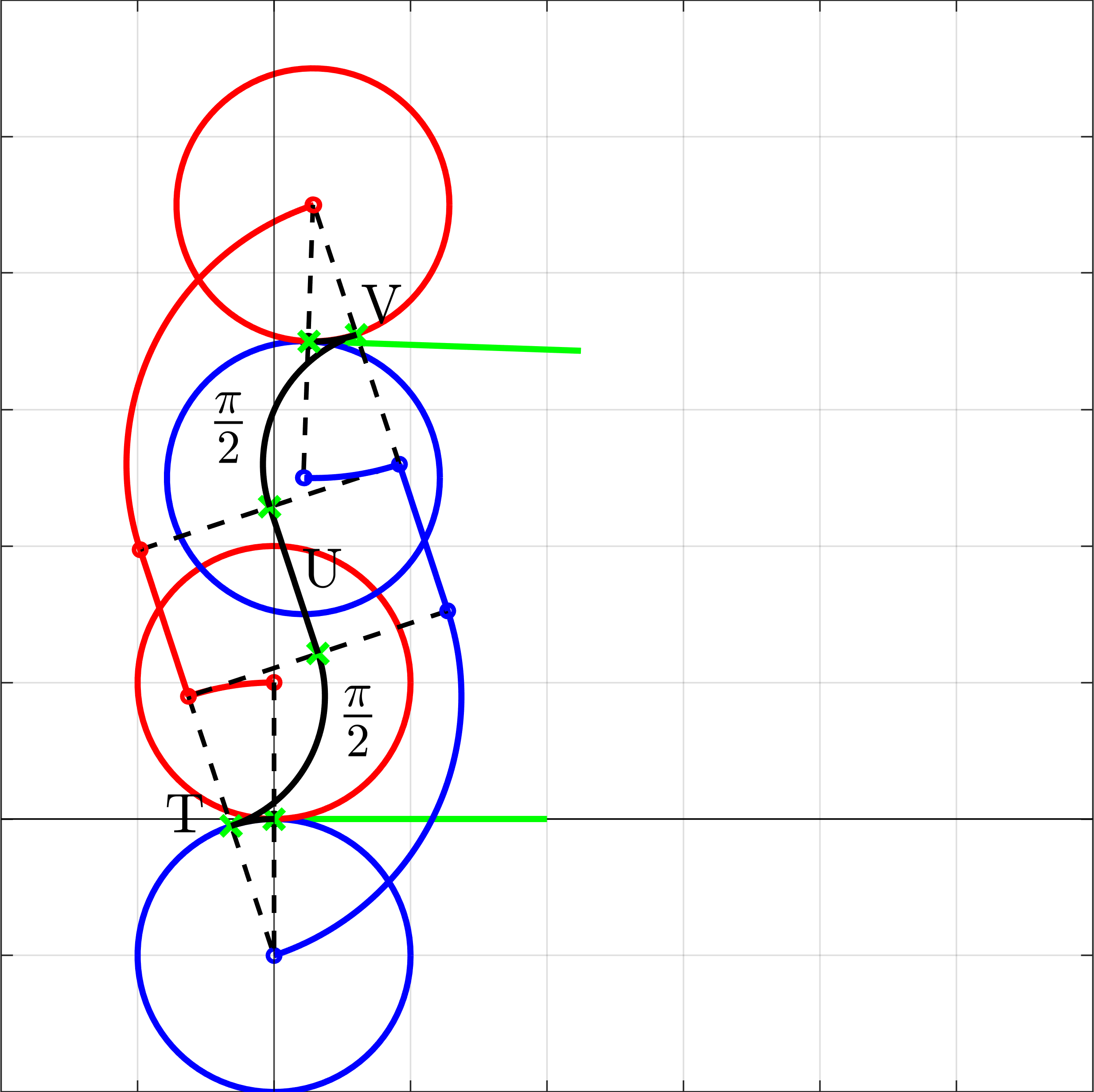}
	\end{subfigure}\hfil
	\begin{subfigure}{}
		\includegraphics[width=4.25cm]{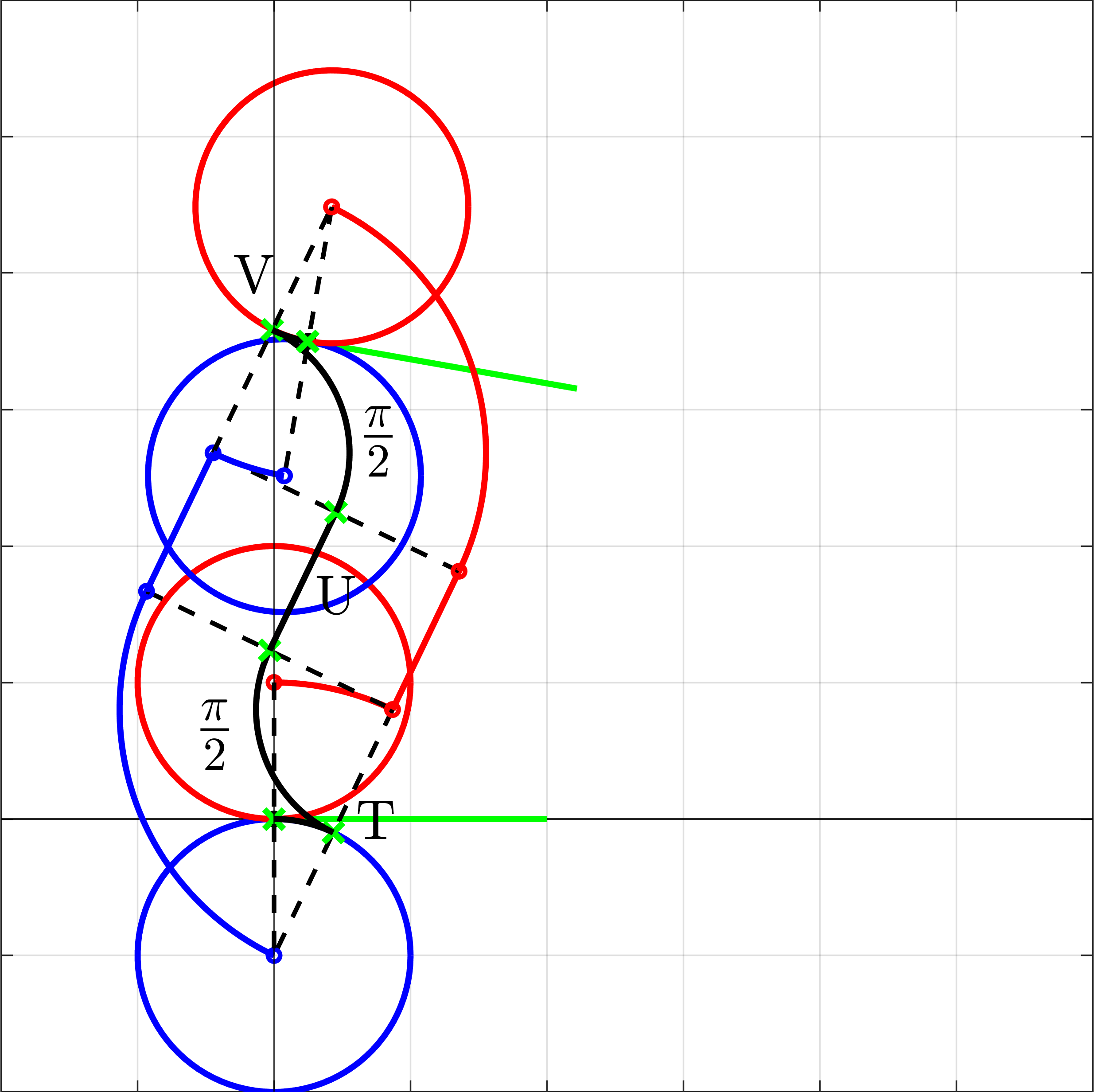}
	\end{subfigure}\hfil
	\caption{$P_{9}$: $r_{T}^{-}l_{\frac{\pi}{2}}^{+}s_{U}^{+}r_{\frac{\pi}{2}}^{+}l_{V}^{-}$ (left), \& $P_{12}$: $r_{T}^{+}l_{\frac{\pi}{2}}^{-}s_{U}^{-}r_{\frac{\pi}{2}}^{-}l_{V}^{+}$ (right).}\label{fig:cond_9_12}
\end{figure}

\subsubsection{Set B}
\begin{figure}[!ht]
	\centering
	\includegraphics[width=8.5cm,keepaspectratio]{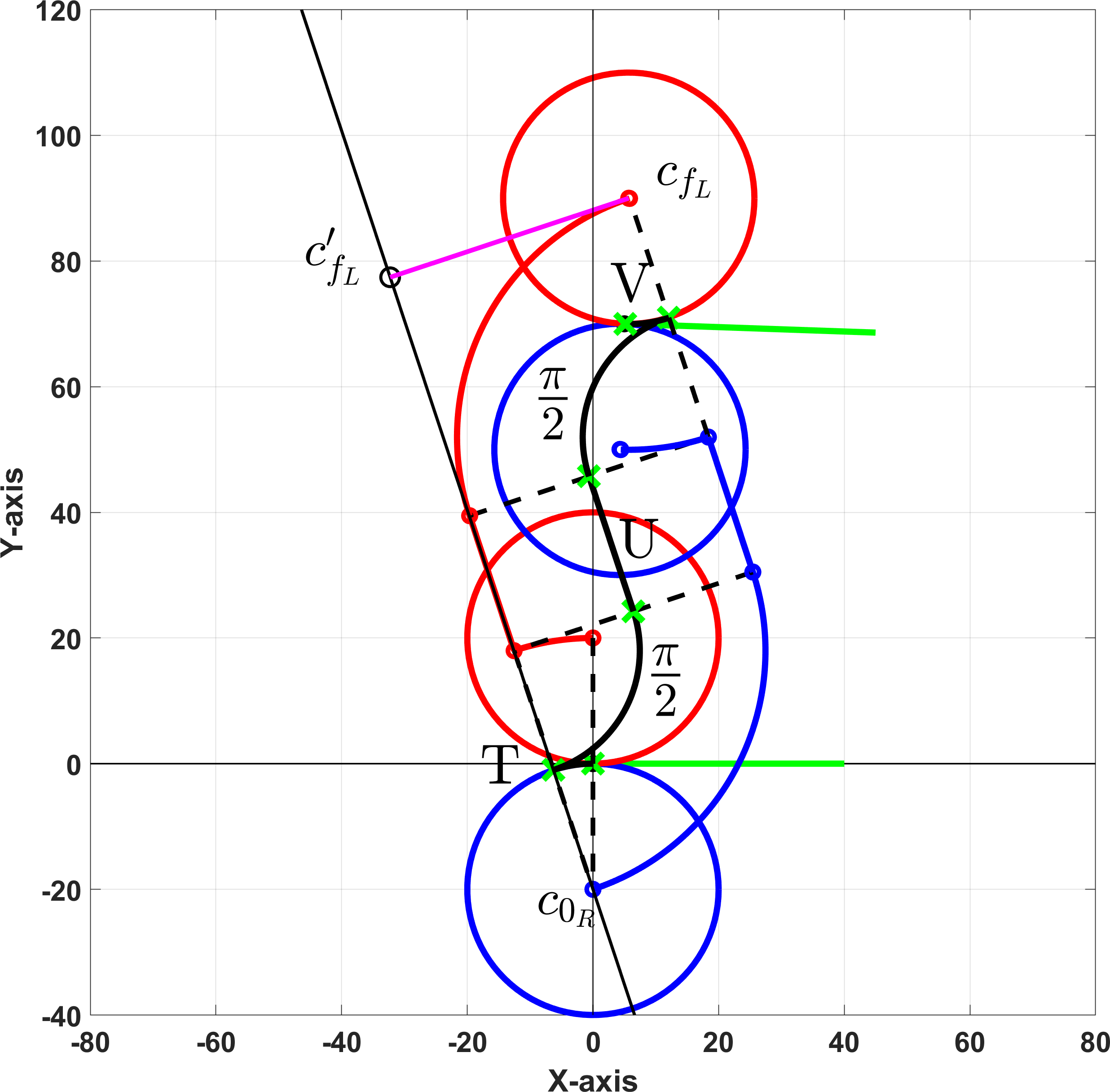}
	\caption{$P_{9}$ consisting of $r_{T}^{-}l_{\frac{\pi}{2}}^{+}s_{U}^{+}r_{\frac{\pi}{2}}^{+}l_{V}^{-}$. $c_{f_{L}}'$ is the projection of $c_{f_{L}}$ onto the line formed by $c_{0_{R}}$ and the angle $\frac{\pi}{2} + T$.}\label{fig:new_9_12_proof}
\end{figure}
Set B contains the families $CCC$, $CCCC$, and $CCSCC$. Only the last one, $CCSCC$ involves an $S$ primitive. The S curve formed by the middle three primtives in $P_{9}$ and $P_{12}$, $l_{\frac{\pi}{2}}^{+}s_{U}^{+}r_{\frac{\pi}{2}}^{+}$ and $r_{\frac{\pi}{2}}^{-}s_{U}^{-}l_{\frac{\pi}{2}}^{-}$ respectively, makes it such that the distance $RL$ between $c_{0_{R}}$ and $c_{f_{L}}$ is at least $\sqrt{20}r$. We make the following proposition to isolate $P_{9}$ and $P_{12}$ from the rest of the families in set B\@:
\begin{proposition}\label{prop:p9_p12_setB}
	For path types $P_{9}$ and $P_{12}$, $RL \geq{} \sqrt{20}r$.
\end{proposition}
\begin{proof}
	We refer to Fig.~\ref{fig:new_9_12_proof} for this proof. The proof follows from the fact that in both $P_{9}$ and $P_{12}$ the last primitive is a left turn, during which $c_{f_{L}}$ does not change its position. $\forall U \geq 0$, the distance between $c_{0_{R}}$ and $c_{f_{L}}'$ is $\geq 4r$. Since the line formed by $c_{f_{L}}$ and $c_{f_{L}}'$ is of length $2r$ and is perpendicular to the line formed by $c_{0_{R}}$ and $c_{f_{R}}$, we use the Pythagorean theorem to deduce that $RL \geq{} \sqrt{16r^{2}+4r^{2}} = \sqrt{20}r$.
\end{proof}
$P_{9}$, however, can be subpartitioned from $P_{12}$ by the following proposition:
\begin{proposition}\label{prop:p9_vs_p12_setB}
	For path type $P_{9}$, $\theta_{f} > 2\beta_{0} - \pi$.
\end{proposition}
\begin{proof}
	The proof follows similar reasoning as the proof of Proposition~\ref{prop:p5_vs_p6} and has been omitted for brevity.
\end{proof}

We proceed with the families $CCC$ and $CCCC$. We start with the case that logically follows from $P_{9}$ when $RL < \sqrt{20}r$, which is $P_{13}$, $r_{T}^{-}l_{U}^{+}r_{U}^{+}l_{V}^{-}$. We illustrate $P_{13}$ in Fig.~\ref{fig:cond_13_14} to the left. In order to differentiate between $P_{13}$ and $P_{14}$, $r_{T}^{-}l_{U}^{+}r_{V}^{+}$, we make the following proposition:
\begin{proposition}\label{prop:p13_vs_p14}
	For path type $P_{13}$, $\alpha \geq{} \beta$, where $\alpha = \acos{\left(\frac{3r^{2}+\frac{RL^{2}}{4}}{2rRL}\right)}$ and $\beta = \theta_{f}-\frac{\pi}{2}-\angle R_{0}L_{f}$.
\end{proposition}
\begin{proof}
	We refer to Fig.~\ref{fig:cond_13_14_split} for this proof. Angle $\alpha = \angle EDC$ can be computed using the cosine rule in the $\triangle EDC$, where $EC = r$, $DE = \frac{RL}{2}$, and $DC = 2r$. Angle $\beta = \angle EDF$ can be readily computed as the angle between the lines $AD$ and $DF$. If the last primitive in $P_{13}$ has a lenght of $V = 0$, then the two angles are equal. If $V > 0$, then $\alpha > \beta$, and $P_{13}$ is the optimal path. Otherwise, $P_{14}$ is the optimal path.
\end{proof}

\begin{figure}[]
	\centering
	\begin{subfigure}{}
		\includegraphics[width=4.25cm]{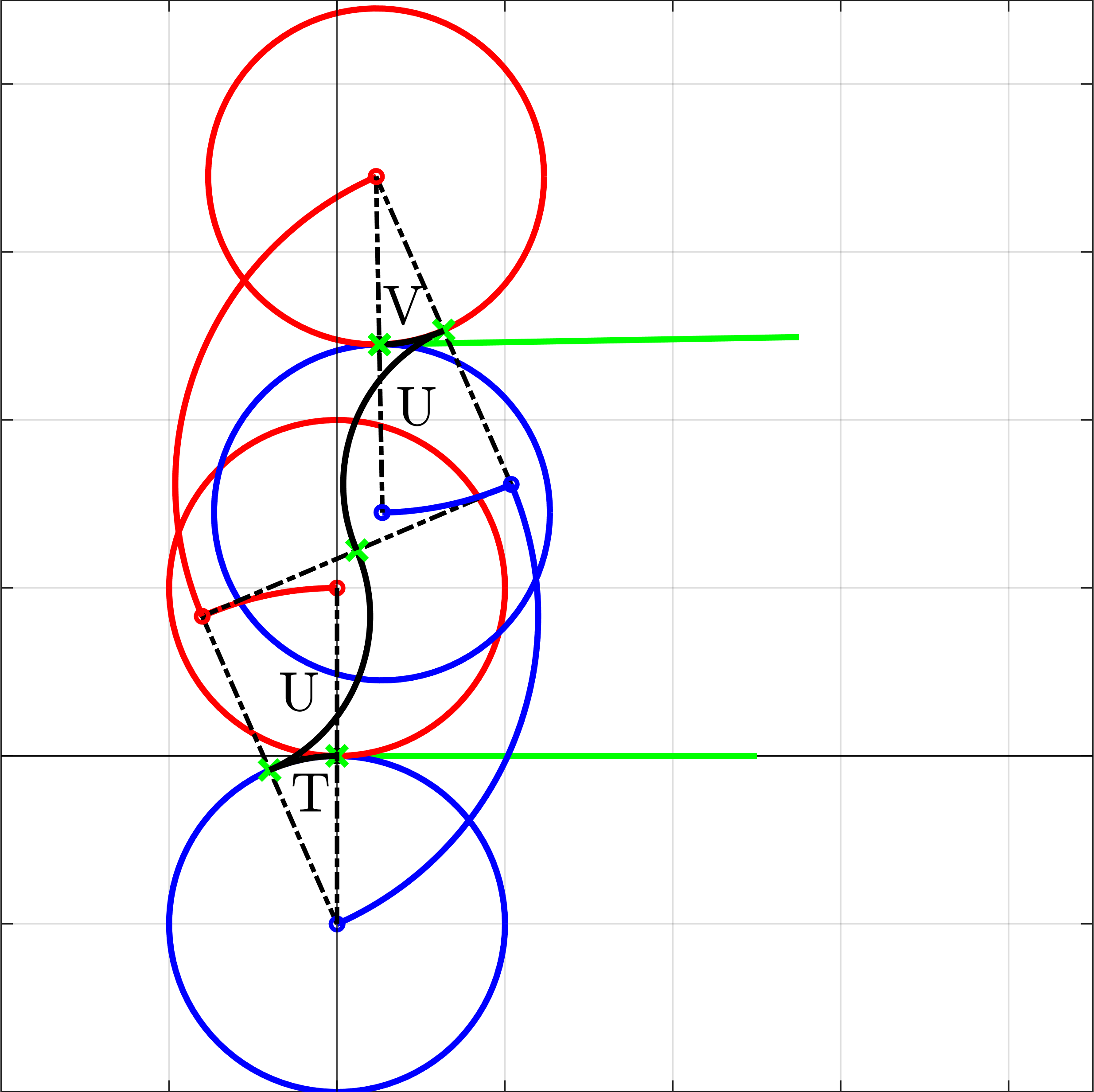}
	\end{subfigure}\hfil
	\begin{subfigure}{}
		\includegraphics[width=4.25cm]{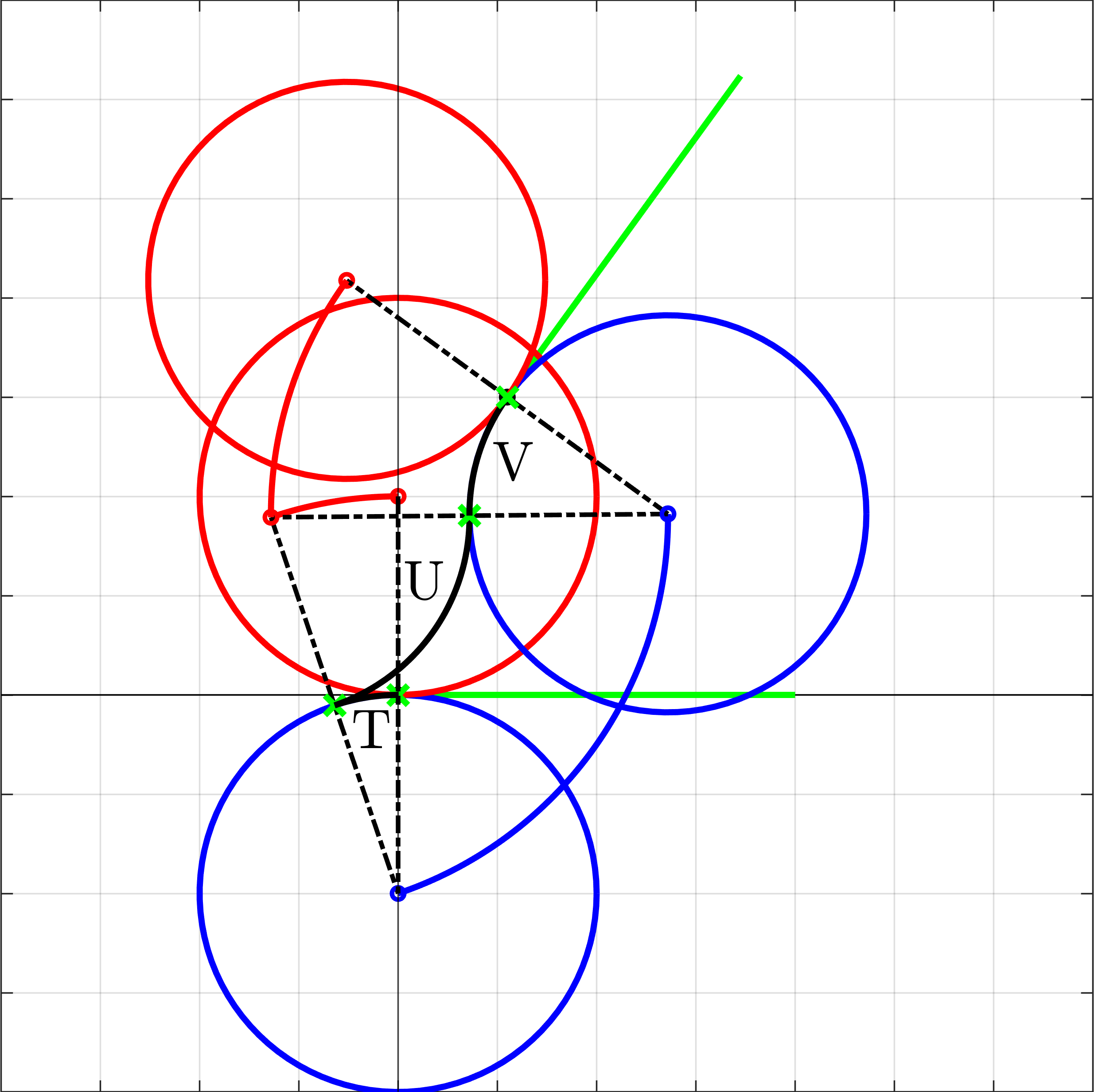}
	\end{subfigure}\hfil
	\caption{$P_{13}$: $r_{T}^{-}l_{U}^{+}r_{U}^{+}l_{V}^{-}$ (left), \& $P_{14}$: $r_{T}^{-}l_{U}^{+}r_{V}^{+}$ (right).}\label{fig:cond_13_14}
\end{figure}

\begin{figure}[!ht]
	\centering
	\includegraphics[width=8.5cm,keepaspectratio]{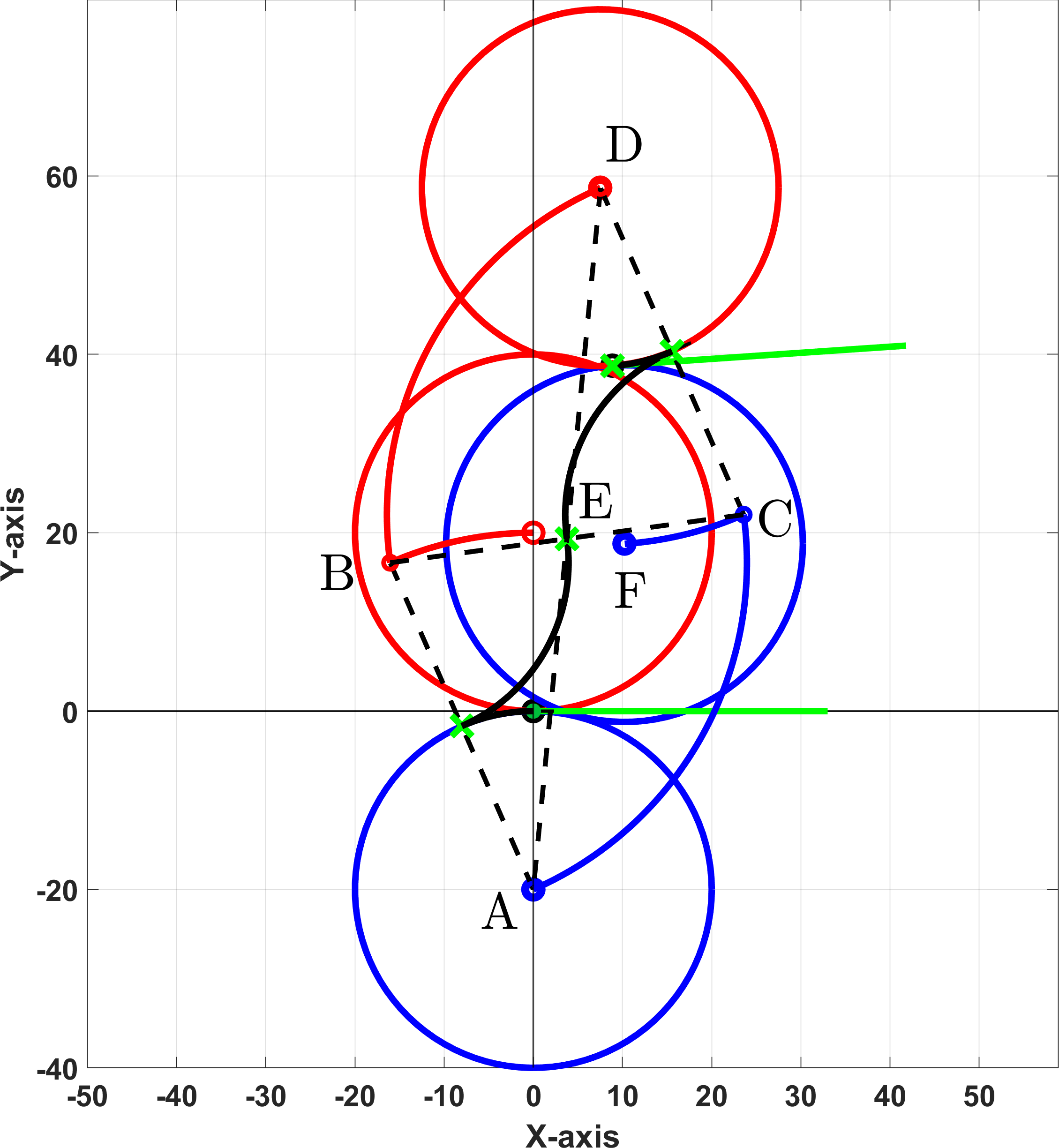}
	\caption{The triangles $\triangle EAB$ and $\triangle EDC$ are similar due to the symmetry in the path $P_{13}$, $r_{T}^{-}l_{U}^{+}r_{U}^{+}l_{V}^{-}$.}\label{fig:cond_13_14_split}
\end{figure}

\begin{figure}[]
	\centering
	\begin{subfigure}{}
		\includegraphics[width=4.25cm]{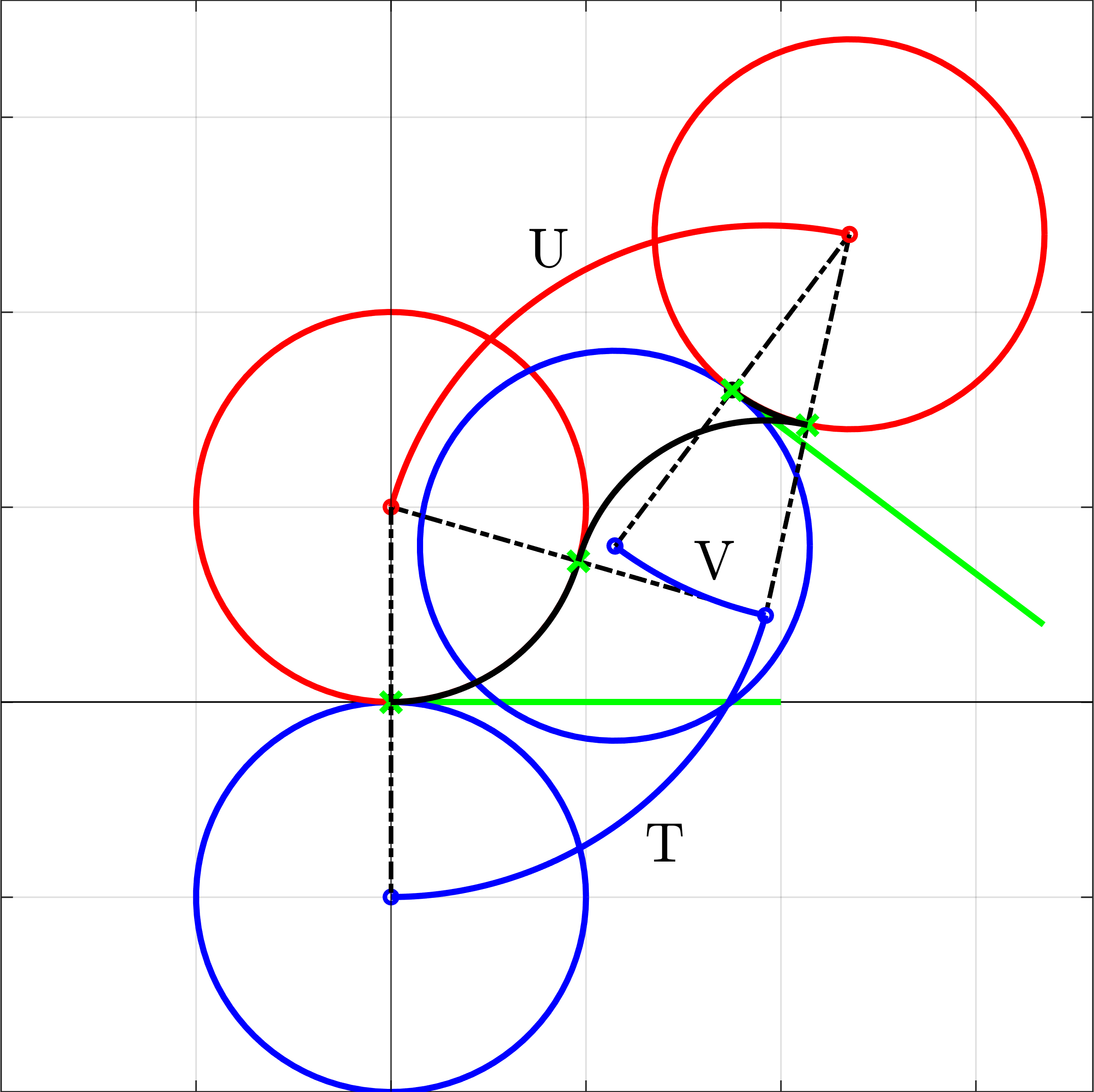}
	\end{subfigure}\hfil
	\begin{subfigure}{}
		\includegraphics[width=4.25cm]{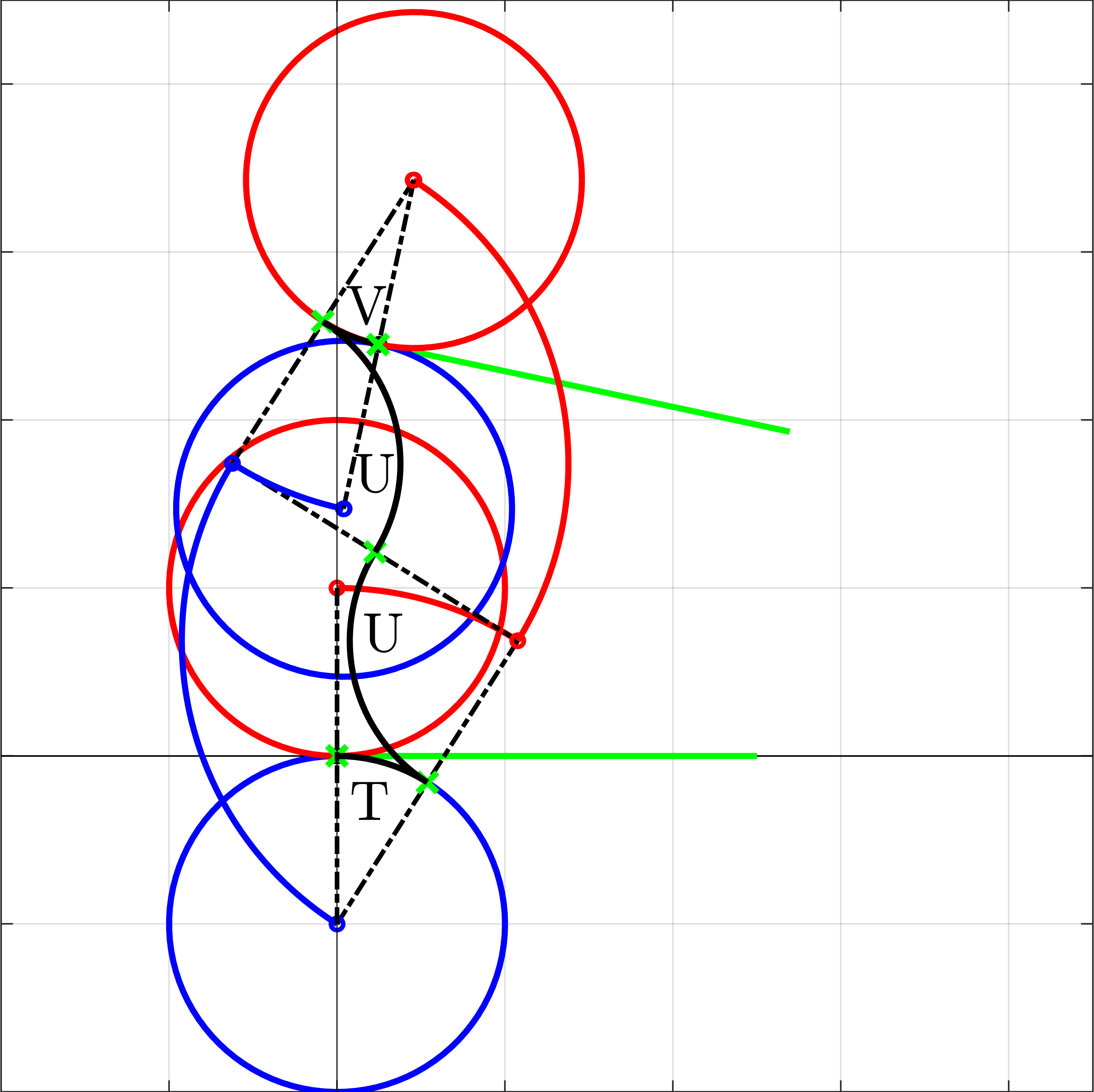}
	\end{subfigure}\hfil
	\caption{$P_{16}$: $l_{T}^{+}r_{U}^{+}l_{V}^{-}$ (left), \& $P_{17}$: $r_{T}^{+}l_{U}^{-}r_{U}^{-}l_{V}^{+}$ (right).}\label{fig:cond_17_16}
\end{figure}

The proof for the proposition to subpartition $P_{16}$, $l_{T}^{+}r_{U}^{+}l_{V}^{-}$, from $P_{13}$ is closely related to the proof of Proposition~\ref{prop:p13_vs_p14}. We illustrate an example $P_{16}$ path to the left in Fig.~\ref{fig:cond_17_16} and we make the following proposition:
\begin{proposition}\label{prop:p16_vs_p13}
	For path type $P_{16}$, $\alpha \leq \beta_{1} = \frac{\pi}{2} - \angle L_{f}R_{0}$.
\end{proposition}
\begin{proof}
	We also refer to Fig.~\ref{fig:cond_13_14_split} for this proof. When the first primitive $r_{T}^{-}$ in $P_{13}$ has a length of zero, the line $AB$ must be aligned with the y-axis. The angle $\beta_{1}$ is the angle between the line $AD$ and the y-axis, whereas $\alpha$ is the angle between the lines $AB$ and $AD$. If the first primitive in $P_{13}$ has a non-zero length, then $\alpha > \beta_{1}$, and $P_{13}$ is the optimal path. Otherwise, $P_{16}$ is the optimal path.
\end{proof}

Path type $P_{17}$, $r_{T}^{+}l_{U}^{-}r_{U}^{-}l_{V}^{+}$, is the mirror of $P_{13}$ with respect to the line connecting $c_{0_{R}}$ and $c_{f_{L}}$. We illustrate $P_{17}$ in Fig.~\ref{fig:cond_17_16} to the right. We make the following proposition to subpartition $P_{17}$ from $P_{13}$:
\begin{proposition}\label{prop:p17_vs_p13}
	For path type $P_{17}$, $\theta_{f} < 2 L_{f}R_{0} - \pi$.
\end{proposition}
\begin{proof}
	The proof is obtained by comparing the analytical expressions for the length of each primitive in $P_{13}$ and $P_{17}$. The two have the same exact expression for $U$. From the expressions $T$ and $V$, one may construct the following inequality: $T_{13} - T_{17} \geq \theta_{f}$ that holds for $d_{13} - d_{17} \geq 0$, where $d_{13}$ and $d_{17}$ are the total path distances for $P_{13}$ and $P_{17}$ respectively and $T_{13}$ and $T_{17}$ are the lengths of the first primitives in $P_{13}$ and $P_{17}$ respectively. A solution for $T_{13} - T_{17} = \theta_{f}$ is found when $\theta_{f} = 2 L_{f}R_{0} - \pi$.
\end{proof}

Path type $P_{15}$, $l_{T}^{+}r_{U}^{-}l_{V}^{+}$ as shown in Fig.~\ref{fig:cond_15_18} to the left, can be subpartitioned from $P_{14}$ as $\theta_{f}$ keeps increasing in the range $\theta_{f} \geq \frac{\pi}{2}$ by the following proposition:
\begin{proposition}\label{prop:p15_vs_p14}
	For path type $P_{15}$, $(LR \leq 2r) \land (RL \leq 2r)$.
\end{proposition}
\begin{proof}
	The proof follows from Lemma.~\ref{lem:first} and the same line of reasoning that was used to obtain Proposition~\ref{prop:booleanset}.
\end{proof}

A similar proposition to Proposition~\ref{prop:p13_vs_p14} can be made to subpartition $P_{18}$ from $P_{17}$:
\begin{proposition}\label{prop:p18_vs_p17}
	For path type $P_{18}$, $\alpha \leq \beta_{2} = -\theta_{f} - \beta_{1}$.
\end{proposition}
\begin{proof}
	The proof follows the same reasoning as the proof of Proposition~\ref{prop:p13_vs_p14} and has been omitted for brevity.
\end{proof}

\begin{figure}[]
	\centering
	\begin{subfigure}{}
		\includegraphics[width=4.25cm]{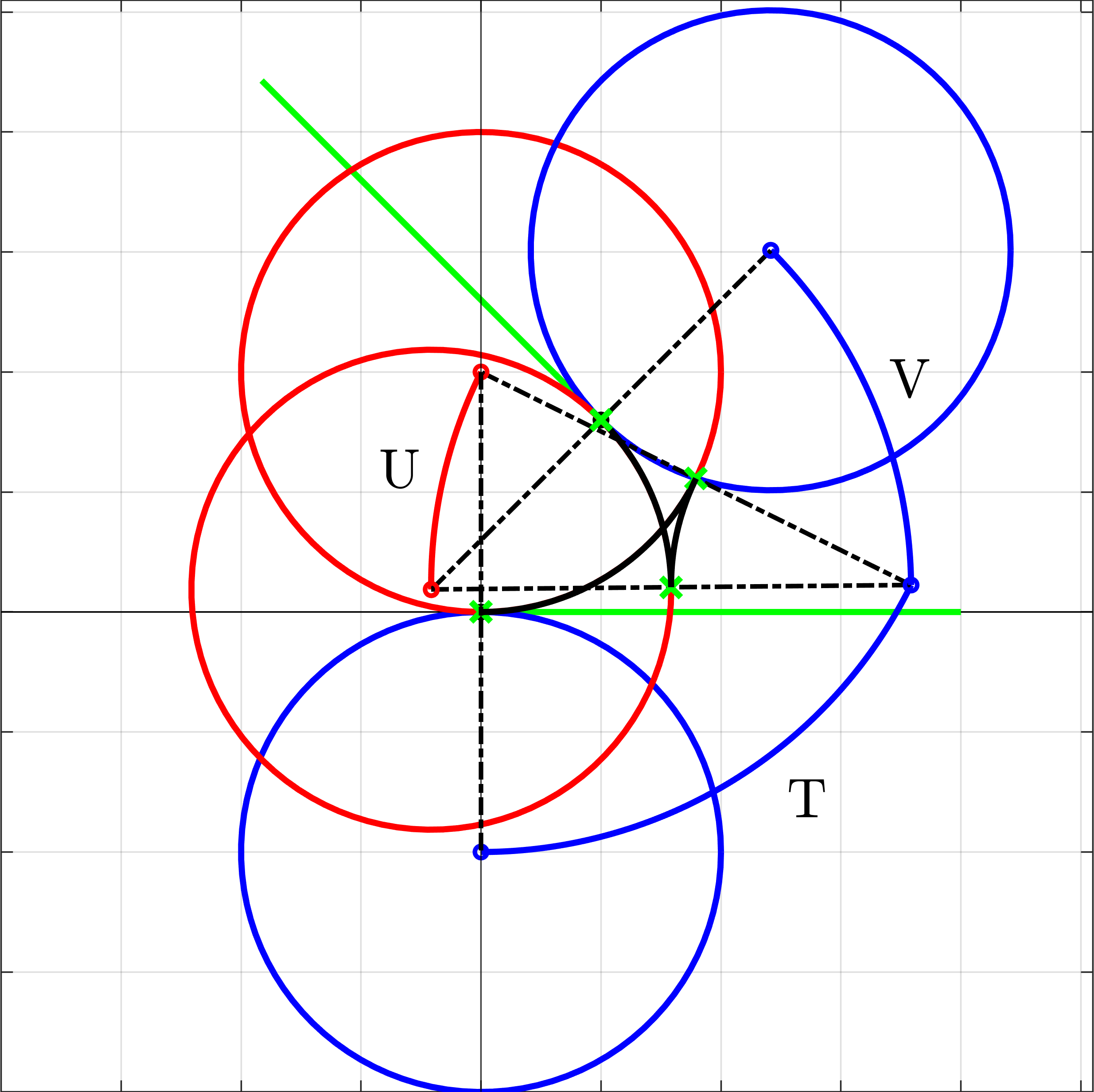}
	\end{subfigure}\hfil
	\begin{subfigure}{}
		\includegraphics[width=4.25cm]{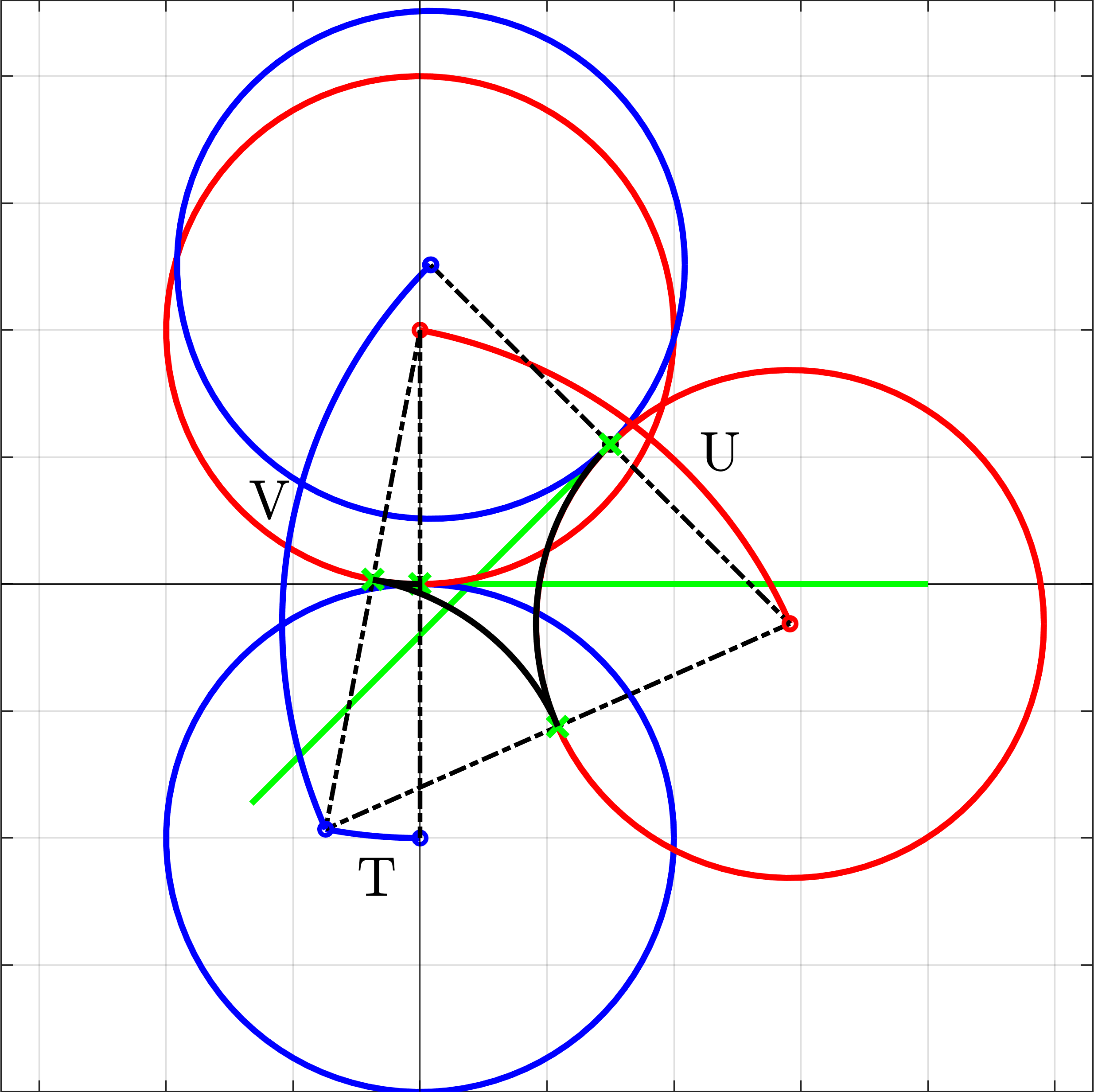}
	\end{subfigure}\hfil
	\caption{$P_{15}$: $l_{T}^{+}r_{U}^{-}l_{V}^{+}$ (left), \& $P_{18}$: $l_{T}^{-}r_{U}^{+}l_{V}^{-}$ (right).}\label{fig:cond_15_18}
\end{figure}

\begin{figure}[]
	\centering
	\begin{subfigure}{}
		\includegraphics[width=4.25cm]{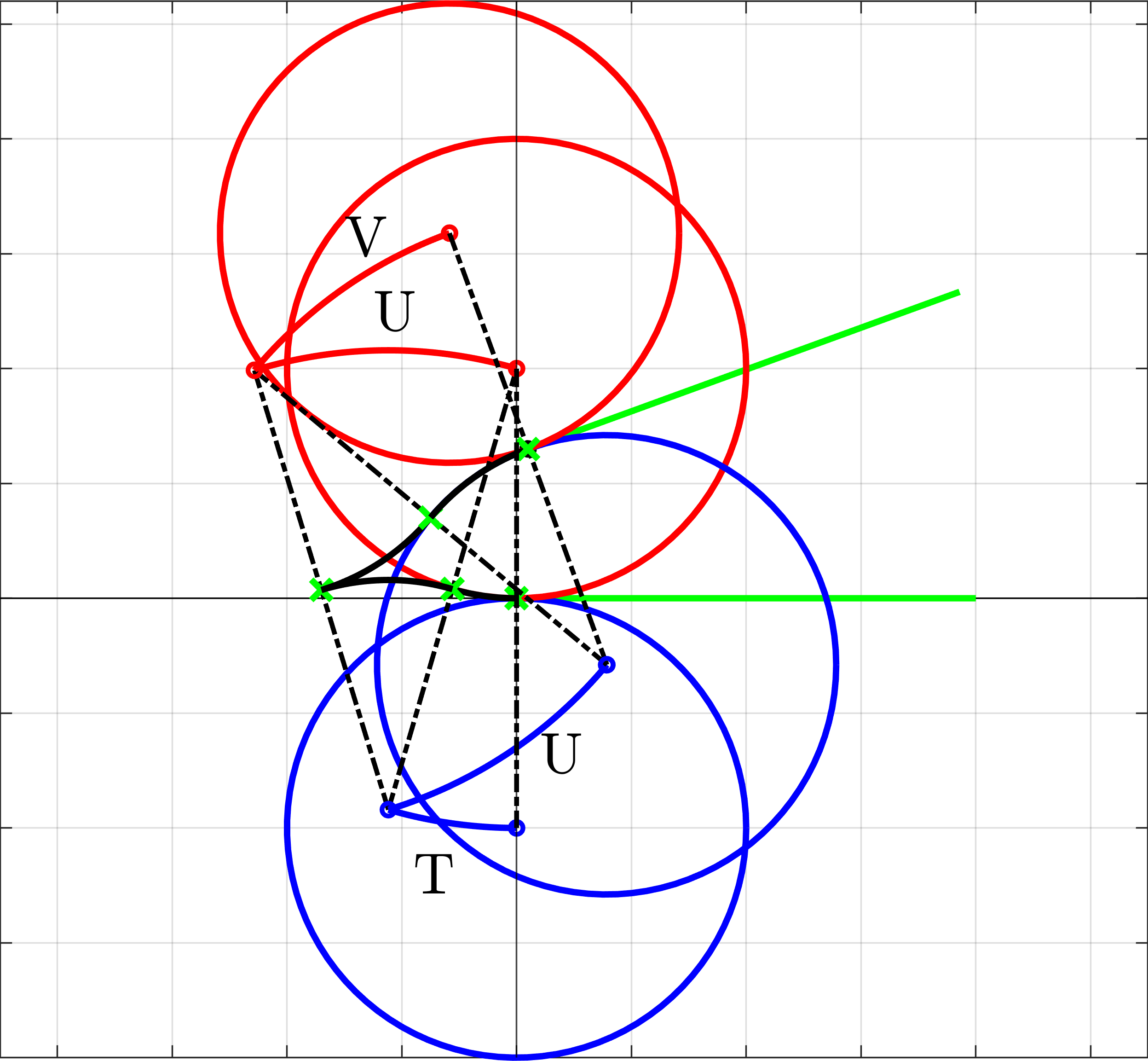}
	\end{subfigure}\hfil
	\begin{subfigure}{}
		\includegraphics[width=3.95cm]{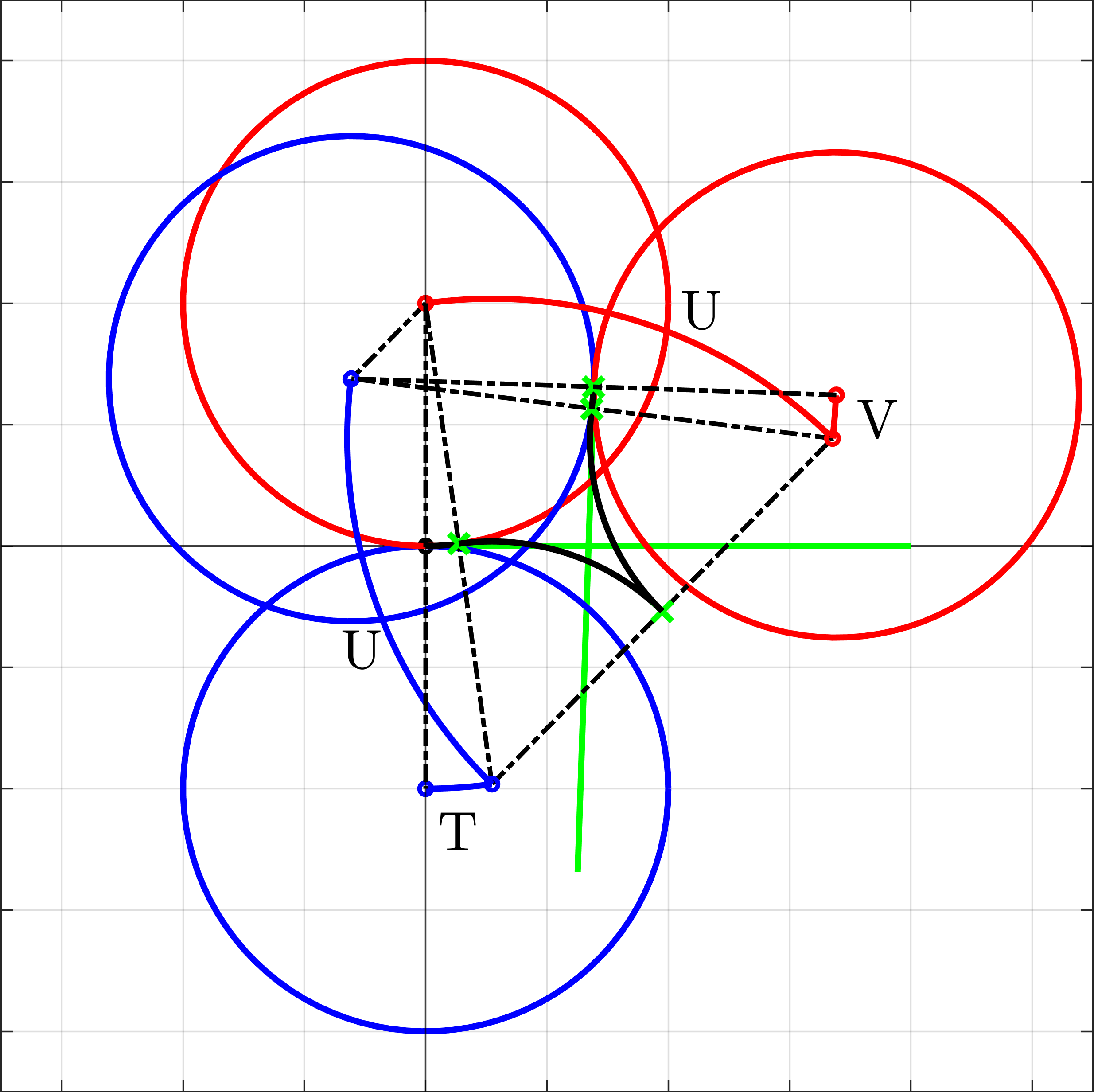}
	\end{subfigure}\hfil
	\caption{$P_{19}$: $l_{T}^{-}r_{U}^{-}l_{U}^{+}r_{V}^{+}$ (left), \& $P_{20}$: $l_{T}^{+}r_{U}^{+}l_{U}^{-}r_{V}^{-}$ (right).}\label{fig:cond_19_20}
\end{figure}

Path type $P_{19}$, $l_{T}^{-}r_{U}^{-}l_{U}^{+}r_{V}^{+}$ as shown in Fig.~\ref{fig:cond_19_20} to the left, differs from $P_{14}$ only in the first primitive. We make the following proposition to subpartition $P_{19}$ from $P_{14}$:
\begin{proposition}\label{prop:p19_vs_p14}
	For path type $P_{19}$, $(RL > 2r)~\lor~(\gamma = \acos{\left( \frac{ \frac{LR}{2} + r}{2r} \right)} > \beta_{3} = R_{f}L_{0} + \frac{\pi}{2})$.
\end{proposition}
\begin{proof}
	The proof follows similar reasoning as the proof of Proposition~\ref{prop:p13_vs_p14}. We refer to Fig.~\ref{fig:cond_14_19_split} for this proof. Due to the symmetry in path $P_{19}$, two similar triangles $\triangle ABE$ and $\triangle EDC$ are formed, where $AB$ and $CD$ are parallel. $AD$ intersects $BC$ at $E$. The length of $AD$ may be computed as the sum of two hypotenuses $h_{1} = AE$ and $h_{2} = ED$. Each hypotenuse is computed using trigonometry in the right triangles formed by the perpendicular bisector crossing $E$. $h_{1} = \frac{\frac{LR}{2}}{\cos{\gamma}}$ and $h_{2} = \frac{r}{\cos{\gamma}}$. Since $h_{1} + h_{2} = 2r$, then $\gamma = \acos{\left( \frac{ \frac{LR}{2} + r}{2r} \right)}$. $\beta_{3}$ constitutes the angle between the $DC$ and the y-axis. When the first primitive in $P_{19}$ has a length of zero, then $\gamma = \beta_{3}$. Otherwise, $\gamma > \beta_{3}$, and $P_{19}$ is the optimal path. Moreover, following the reasoning presented after Lemma~\ref{lem:first}, we obtain $RL > 2r$ for $P_{19}$.
\end{proof}

\begin{figure}[!ht]
	\centering
	\includegraphics[width=8.5cm,keepaspectratio]{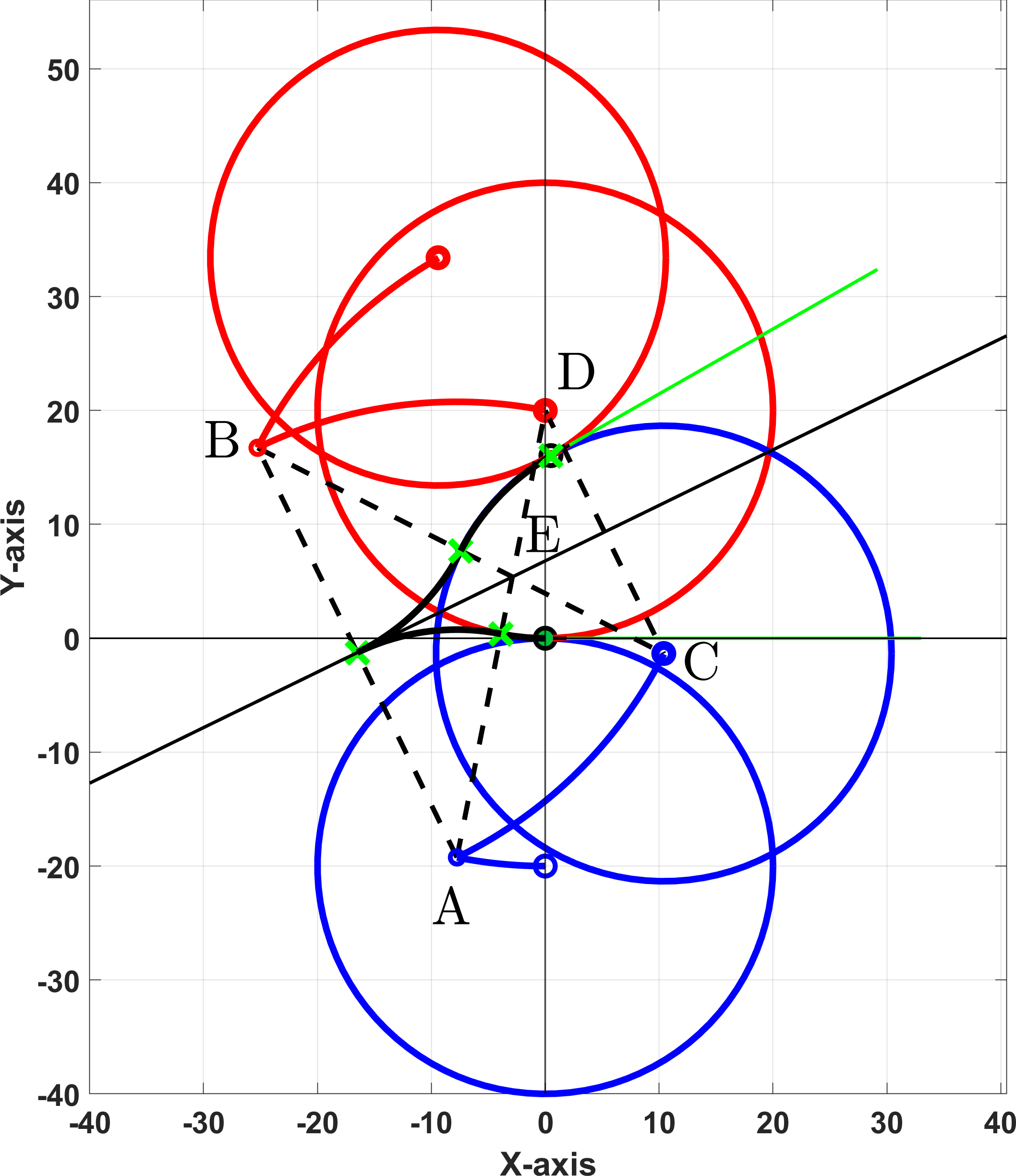}
	\caption{The triangles $\triangle EAB$ and $\triangle EDC$ are similar due to the symmetry in the path $P_{19}$, $l_{T}^{-}r_{U}^{-}l_{U}^{+}r_{V}^{+}$.}\label{fig:cond_14_19_split}
\end{figure}

Path $P_{20}$ differs from $P_{16}$ only in the last primitive. We make the following proposition to subpartition $P_{20}$ from $P_{16}$:
\begin{proposition}\label{prop:p20_vs_p16}
	For path type $P_{20}$, $(O > LL)~\lor~(O > RR)$ where $O = 4r \sin{\left(\frac{\gamma}{2}\right)}$.
\end{proposition}
\begin{proof}
	The proof follows similar reasoning as the proof of Proposition~\ref{prop:p19_vs_p14}. $P_{20}$ forms two similar triangles due to its symmetry. The chord $O$ connecting the end points of the arc $r_{U}^{+}$ and $l_{U}^{-}$ is computed as $O = 4r \sin{\left(\frac{\gamma}{2}\right)}$. When either the first or last primitive in $P_{20}$ have a length of zero, then $O = RR$ and $O = LL$ respectively. If $(O \leq LL)~\lor~(O \leq RR)$, then $P_{16}$ is the optimal path. Otherwise, $P_{20}$ is the optimal path.
\end{proof}

Similar argument can be used to establish the following proposition that subpartitions $P_{18}$ from $P_{20}$:
\begin{proposition}\label{prop:p18_vs_p20}
	For path type $P_{18}$, $(O \leq RR) \lor (RL \leq 2r)$.
\end{proposition}
\begin{proof}
	The proof follows similar reasoning as the proof of Proposition~\ref{prop:p20_vs_p16} as well as Lemma~\ref{lem:first}.
\end{proof}

Lastly, we address the persisting problem in literature where $P_{20}$ could not be subpartitioned from $P_{17}$. In order to do so, we make the following proposition:
\begin{proposition}\label{prop:p17_vs_p20}
	For path type $P_{17}$, $(T_{17} \leq T_{20}) \lor (T_{17} + U_{17} \leq 2 U_{20})$, where $T_{17}$ and $U_{17}$ are the lengths of the first and second primitives in $P_{17}$ respectively, and $T_{20}$ and $U_{20}$ are the lengths of the first and second primitives in $P_{20}$ respectively.
\end{proposition}
\begin{proof}
	The proof for $T_{17} + U_{17} \leq 2 U_{20}$ can be obtained from the inequality $d_{17} \leq d_{20}$ formed by the analytical expressions of the total path distances for $P_{17}$ and $P_{20}$ respectively. $T_{17} \leq T_{20}$ can directly be observed and such an argument has been previously made in~\cite{Desaulniers1995} in the last paragraph of Section IV\@.
\end{proof}

Similar arguments can be established to subpartition $P_{13}$ from $P_{19}$ and $P_{20}$ respectively. We omit the proofs for brevity, but the respective propositions are as follows:
\begin{proposition}\label{prop:p13_vs_p19}
	For path type $P_{13}$, $(T_{13} \leq V_{19}) \lor (T_{13} + U_{13} \leq 2 U_{19})$.
\end{proposition}
\begin{proposition}\label{prop:p13_vs_p20}
	For path type $P_{13}$, $(V_{13} \leq T_{20}) \lor (T_{13} + U_{13} \leq \theta_{f} + 2 U_{20})$.
\end{proposition}

We have now provided a complete set of propositions that subpartition each path type in each set, A and B\@. The algorithms that combine the propositions that subpartition the path types in each set, A and B, are provided as commented pseudocode in Alg.~\ref{alg:A_partitions} and Alg.~\ref{alg:B_partitions} in Appendix~\ref{app:algs}. The final algorithm that combines the two sets, Alg.~\ref{alg:main_partition} is also provided in Appendix~\ref{app:algs}.

\subsection{On the Non-Uniqueness of the Optimal Solution}\label{subsec:nonuniqueness}
Path solutions obtained with types $P_{1}, \ldots, P_{20}$ are not unique in terms of shortest distance optimality. It is sufficient to provide a simple example to illustrate this. Consider the starting configuration $p_{0} = \left(x_{0}, y_{0}, \theta_{0}\right) = \left(0, 0, 0\right)$ and the final configuration $p_{f} = \left(x_{f}, y_{f}, \theta_{f}\right) = \left(0.05, 0.12, -1.5\right)$ with a minimum turning radius of $r = 1$. The following three paths are all optimal and reach the final configuration $p_{f}$ with zero error with the shortest distance $d = 1.5$:
\begin{itemize}
	\small
	\item $l_{T}^{-}r_{U}^{+}l_{V}^{-}$, with $\{T, U, V\} = \{0.32051, 0.67456, 0.50493\}$.
	\item $l_{T}^{-}r_{U}^{+}l_{U}^{-}r_{V}^{+}$, with $\{T, U, V\} = \{0.2021, 0.5816, 0.1347\}$.
	\item $r_{T}^{+}l_{U}^{-}r_{V}^{+}$, with $\{T, U, V\} = \{0.4751, 0.7225, 0.3024\}$.
\end{itemize}

\subsection{On the Completeness and Correctness of the Solution}\label{subsec:completeness}
\setcounter{thm}{0}
\begin{thm}
	The accelerated Reeds-Shepp solution is both complete and correct.
\end{thm}

\begin{proof}
	We now prove the completeness, correctness, and termination of the proposed accelerated Reeds-Shepp solution. \\
	1. \textbf{Completeness:} The predicates in Algorithms~\ref{alg:A_partitions} and~\ref{alg:B_partitions} divide the problem space into non-overlapping subspaces, with no input belonging to more than one subspace (mutual exclusivity). Additionally, the union of all subspaces covers the entire problem domain, ensuring no gaps (collective exhaustiveness). Propositions~\ref{prop:booleanset} through~\ref{prop:p13_vs_p20} establish these properties for all possible inputs.

	2. \textbf{Correctness:} Within each subspace, the solution is derived using analytic expressions that satisfy the necessary conditions of the Reeds-Shepp problem, as established in~\cite{Reeds1990}. These expressions are proven to produce feasible and optimal paths for the given constraints, ensuring correctness in every subspace.

	3. \textbf{Termination:} Each branch of the algorithm terminates after a finite sequence of decisions, as ensured by the structure of Algorithms~\ref{alg:A_partitions} and~\ref{alg:B_partitions}. The decision tree in these algorithms has a bounded depth and a deterministic mapping, guaranteeing that every input converges to a single solution.

\end{proof}

This formal proof, combined with the extensive numerical validation, verification, and edge-case testing described in Section~\ref{sec:results}, confirms the algorithm's completeness and correctness.

\section{Under-Specified Reeds-Shepp Solution}\label{sec:underspecified}
In this section, we address finding a solution to the problem defined in Equation~\eqref{eq:underspecified}. Starting with the proposed accelerated Reeds-Shepp solution as well as given OMPL's implementation of the original Reeds-Shepp solution, one can find degenerate cases where the final orientation $\theta_{f}$ lies at the edge of two subpartitions. We have previously illustrated an example in Fig.~\ref{fig:cond_0}, where $\theta_{f}$ is exactly aligned with $\Omega$, and where the optimal Reeds-Shepp solution can be computed using both $P_{1}$ and $P_{2}$ since the last primitive differentiating $P_{1}$ from $P_{2}$ has a length of zero and the first two primitives have matching lengths and types. In that case, the angle $\Omega$ that minimizes the path length from the starting configuration $p_{0}$ to the final position $(x_{f}, y_{f})$ is trivially aligned with the tangent to the LHC at the starting configuration $p_{0}$.

\begin{figure}[!ht]
	\centering
	\includegraphics[width=8.5cm,keepaspectratio]{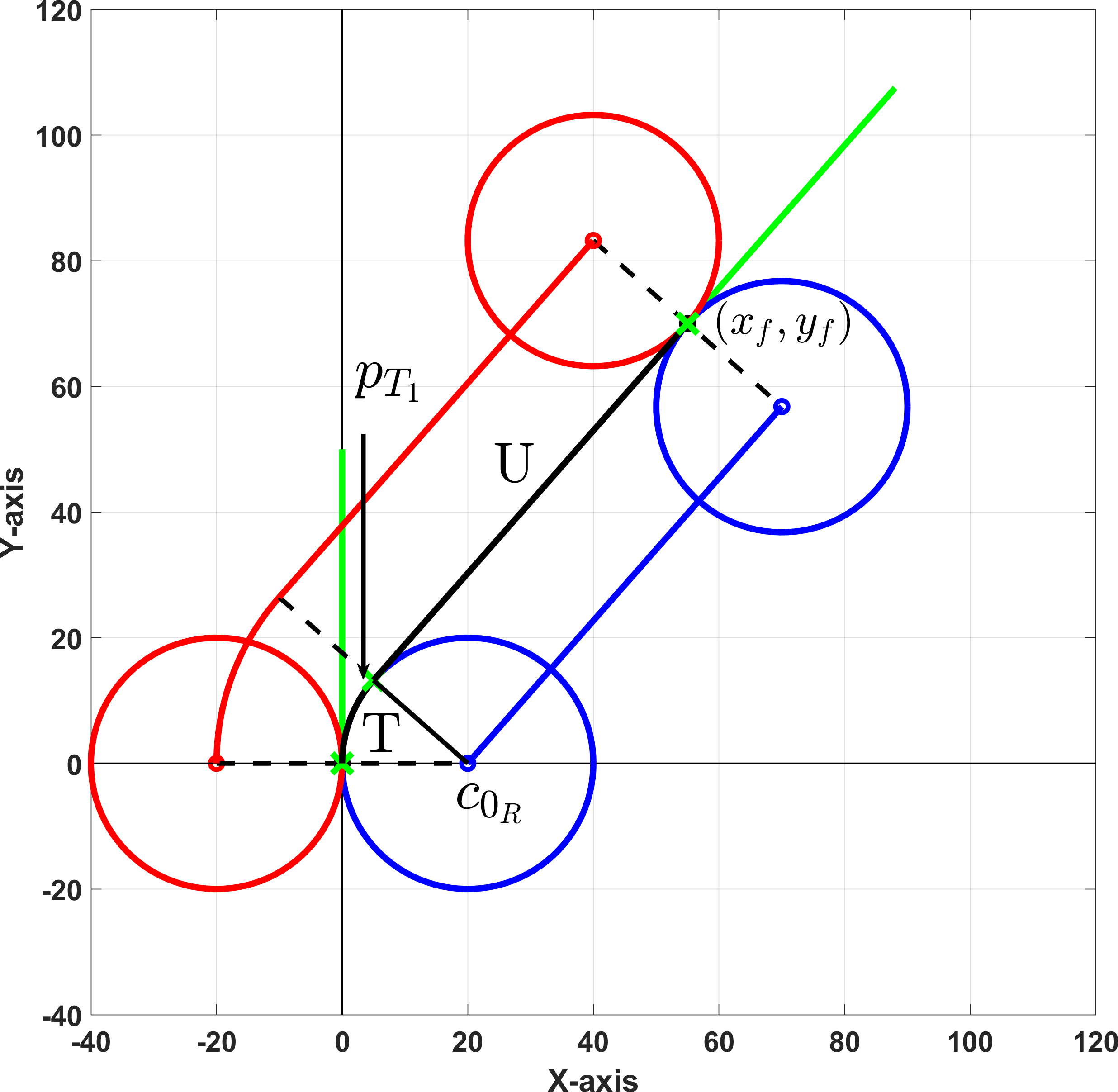}
	\caption{Illustration of $(x_{f},y_{f}) \in R_{1}$ and $p_{0} = (0,0,\frac{\pi}{2})$ with $r = 20$. The optimal path lies on the boundary between $P_{1}$ and $P_{2}$ and constitutes a degenerate/edge case where the last segment of the CSC path has zero length. $T = \Omega_{1}$.}\label{fig:omega_cond_1}
\end{figure}

Another example is illustrated in Fig.~\ref{fig:omega_cond_1} for $p_{0} = (0,0,\frac{\pi}{2})$ and $r = 20$. The final orientation $\Omega_{1}$ that minimizes the path length for the given $\left(x_{f}, y_{f}\right)$ can be obtained by solving the equality equation:
\begin{equation}
	\begin{aligned}
		\frac{ p_{T_{1_{y}}}-c_{0_{R_{y}}} } { p_{T_{1_{x}}}-c_{0_{R_{x}}} } \cdot \frac{ p_{T_{1_{y}}}-y_{f} } { p_{T_{1_{x}}}-x_{f} } = -1,
	\end{aligned}\label{eq:omega_11}
\end{equation}
where
\begin{equation}
	\begin{aligned}
		p_{T_{1}} = \begin{bmatrix} p_{T_{1_{x}}} \\ p_{T_{1_{y}}} \end{bmatrix} = \begin{bmatrix} c_{0_{R_{x}}} + r\cos{\left(\pi-\Omega_{1}\right)} \\ c_{0_{R_{y}}} + r\sin{\left(\pi-\Omega_{1}\right)} \end{bmatrix},
	\end{aligned}
\end{equation}
meaning that the final orientation $\theta_{f}$ is aligned with the tangent to the circle $c_{0_{R}}$ at $p_{T_{1}}$.

We identify three different regions in the $x_{goal}-y_{goal}$ plane where $\Omega$ admits a unique solution and illustrate them in Fig.~\ref{fig:inverse_regions}. The three regions are symmetric across all four quadrants with respect to the x and y axes. This three-region-partitioning of the $x_{goal}-y_{goal}$ plane is independent of the turning radius $r$. The three regions are defined as follows:
\begin{equation*}
	\begin{aligned}
		\begin{cases}
			R_{1} \text{ if } (x_{f},y_{f}) \notin c_{0_{R}} \land{} \big( (y_{f} \geq{} y_{0}+r) \lor{} (x_{f}-c_{0_{R_{x}}} < 0) \big) \\
			R_{2} \text{ if } (x_{f},y_{f}) \notin c_{0_{L}}^{*} \land{} \big( (y_{f} < y_{0}+r) \lor{} (x_{f}-c_{0_{R_{x}}} > 0) \big)  \\
			R_{3} \text{ if } (x_{f},y_{f}) \in c_{0_{L}}^{*} \land{} (x_{f}, y_{f}) \in c_{0_{R}},
		\end{cases}\label{eq:regions}
	\end{aligned}
\end{equation*}
where $c_{0_{R}}$ and $c_{0_{L}}^{*}$ are the circles centered at $(c_{0_{R_{x}}}, c_{0_{R_{y}}})$ and $(c_{0_{L_{x}}}, c_{0_{L_{y}}})$, respectively, with radii $r$ and $r\sqrt{5}$, respectively.
The reasoning behind this partitioning for $R_{1}$ is trivial. $R_{1}$ is the region to which a path $r_{T}^{+}s_{U}^{+}$ is the optimal solution $\forall~(0 \leq T \leq \frac{\pi}{2}), U \geq 0$.

\begin{figure}[!ht]
	\centering
	\includegraphics[width=8.5cm,keepaspectratio]{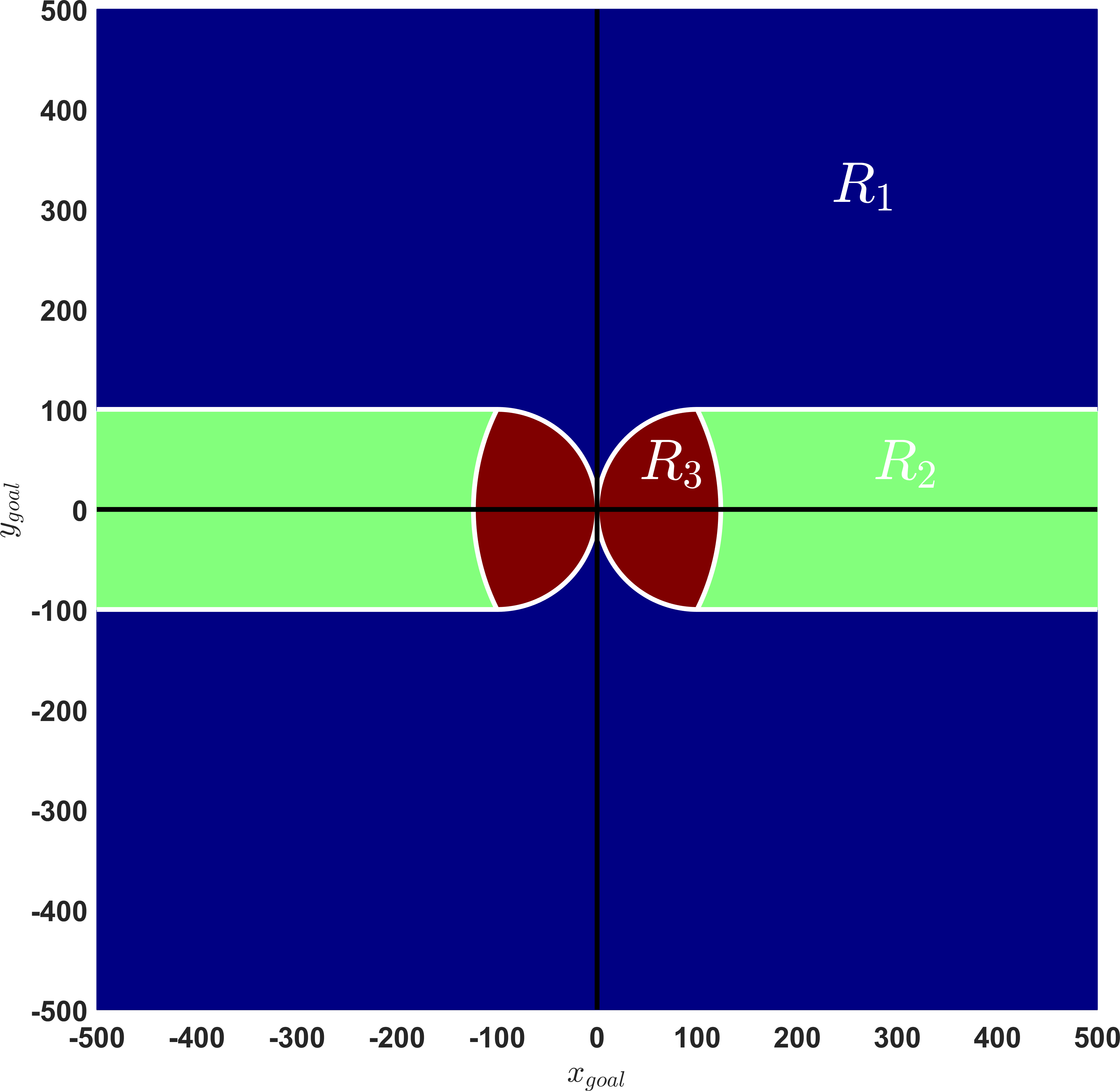}
	\caption{Partitioning of the $x_{goal}-y_{goal}$ plane for the underspecified Reeds-Shepp problem given $p_{0} = (0, 0, \frac{\pi}{2})$ and $r = 100$. Three regions are noted in the first quadrant which are symmetric across all four quadrants with respect to x and y axes. In each region, $\Omega$ admits a unique solution.}\label{fig:inverse_regions}
\end{figure}

We refer to Fig.~\ref{fig:omega_cond_2}. Based on the reasoning introduced in Proposition~\ref{prop:p3}, as $\left(x_{f}, y_{f}\right)$ crosses into region $R_{2}$, an additional $l^{-}$ primitive is introduced. $\theta_{f}$ thus maintains its alignment with the tangent to the circle $c_{0_{L}}$, since $r^{+}$ has an arc length of $\frac{\pi}{2}r$. We may obtain $\Omega_{2}$ by solving the equality equation:
\begin{equation}
	\begin{aligned}
		\frac{ p_{T_{2_{y}}}-c_{0_{L_{y}}} } { p_{T_{2_{x}}}-c_{0_{L_{x}}} } \cdot \frac{ p_{T_{2_{y}}}-y_{f} } { p_{T_{2_{x}}}-x_{f} } = -1,
	\end{aligned}\label{eq:omega_12}
\end{equation}
where
\begin{equation}
	\begin{aligned}
		p_{2_{T}} = \begin{bmatrix} p_{T_{2_{x}}} \\ p_{T_{2_{y}}} \end{bmatrix} = \begin{bmatrix} c_{0_{L_{x}}} + r\cos{\Omega_{2}} \\ c_{0_{L_{y}}} + r\sin{\Omega_{2}} \end{bmatrix}.
	\end{aligned}
\end{equation}

\begin{figure}[!ht]
	\centering
	\includegraphics[width=8.5cm,keepaspectratio]{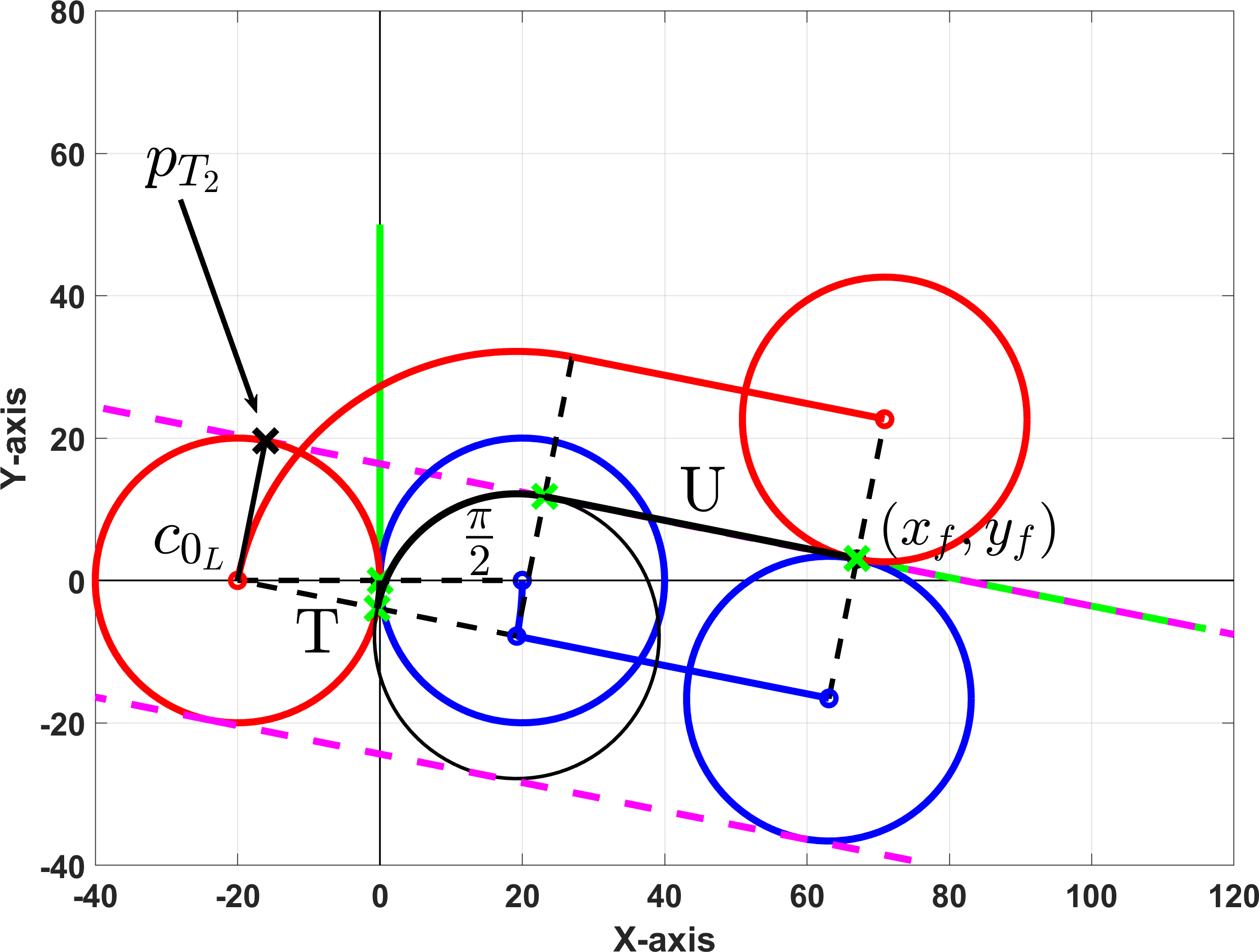}
	\caption{Illustration of $(x_{f},y_{f}) \in R_{2}$ for $p_{0} = (0,0,\frac{\pi}{2})$ with $r = 20$. The optimal path in this case constitutes a degenrate/edge case where the last segment of the CCSC path has zero length.}\label{fig:omega_cond_2}
\end{figure}

For the last case, we refer to Fig.~\ref{fig:omega_cond_3}. Based on the reasoning introduced in Lemma~\ref{lem:first}, $c_{f_{R}}$ is fixed after the first $l^{-}$ primitive. As such, $\Omega_{3}$ is a solution to the following equality equation:
\begin{equation}
	\begin{aligned}
		\sqrt{{\left( p_{T_{3_{x}}}-x_{f} \right)}^{2} + {\left( p_{T_{3_{y}}}-y_{f} \right)}^{2}} = r,
	\end{aligned}\label{eq:omega_13}
\end{equation}
where
\begin{equation}
	\begin{aligned}
		p_{T_{3}} = \begin{bmatrix} p_{T_{3_{x}}} \\ p_{T_{3_{y}}} \end{bmatrix} = \begin{bmatrix} c_{0_{L_{x}}} + 2r\cos{\Omega_{3}} \\ c_{0_{L_{y}}}-2r\sin{\Omega_{3}} \end{bmatrix} = \begin{bmatrix} c_{f_{R_{x}}} \\ c_{f_{R_{y}}} \end{bmatrix}.
	\end{aligned}
\end{equation}

\begin{figure}[!ht]
	\centering
	\includegraphics[width=8.5cm,keepaspectratio]{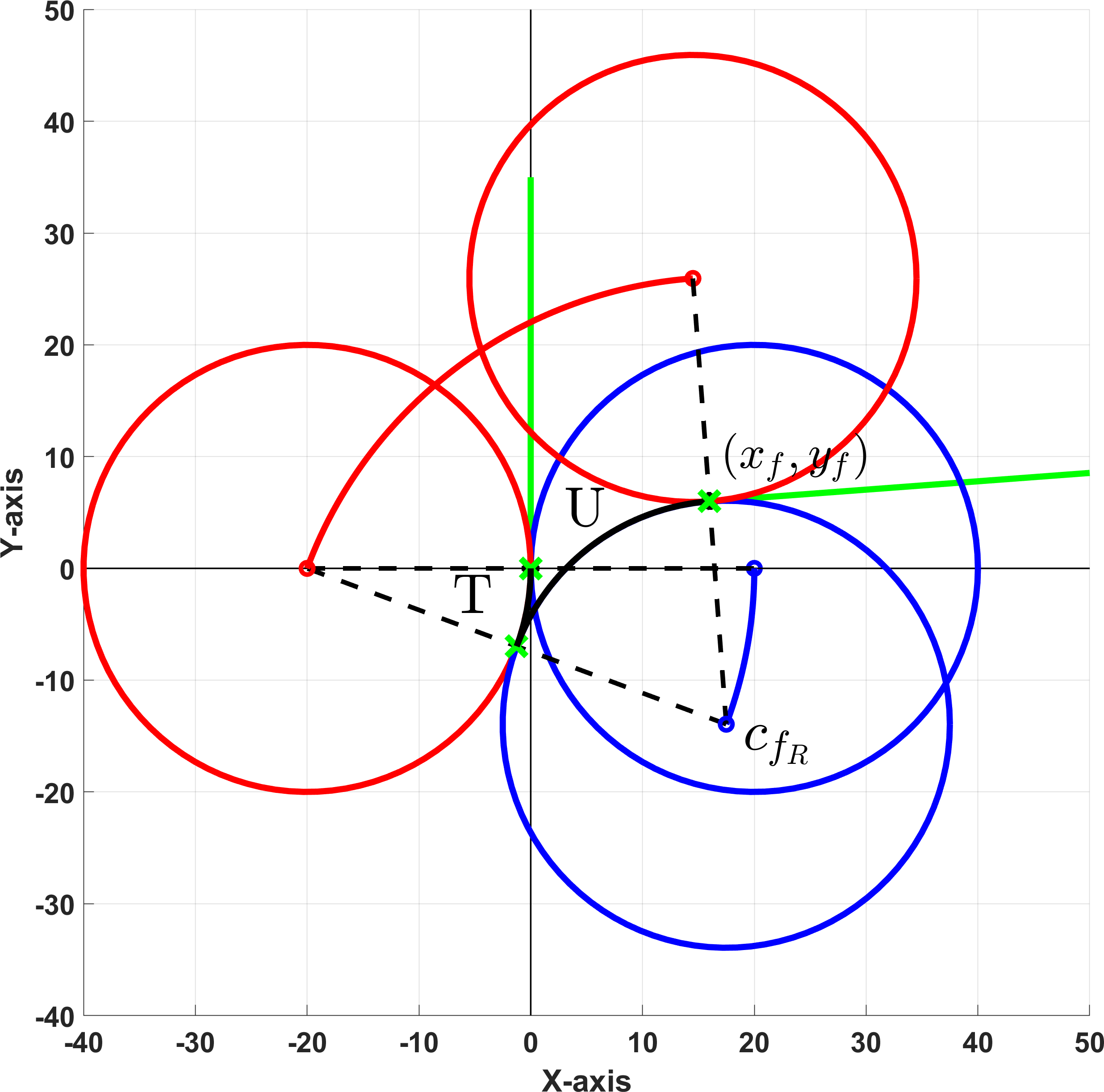}
	\caption{Illustration of $(x_{f},y_{f}) \in R_{3}$ for $p_{0} = (0,0,\frac{\pi}{2})$ with $r = 20$. The optimal path constitutes a degenerate/edge case made up of only two arc segments --- CC.}\label{fig:omega_cond_3}
\end{figure}

The three solutions are documented in Section~\ref{sec:results} for the three regions $R_{1}$, $R_{2}$, and $R_{3}$.

\section{Benchmarking, Validation, and Results}\label{sec:results}
\subsection{Benchmarking}
{}\label{sec:benchmarking}
We compare the performance of our proposed algorithm against OMPL's implementation of the original Reeds-Shepp algorithm and against our implementation of~\cite{Desaulniers1995}. We generate a uniformly distributed sample of \num{1e9} final configurations such that $c_{f_{L_{x}}} \in Q_{1}$, with $Q_{1} = \big\{\small(c_{f_{L_{x}}}, c_{f_{L_{y}}}, \theta_{f}\small): c_{f_{L_{x}}} > 1, c_{f_{L_{y}}} \geq{} 0 \big\}$ (see Fig.~\ref{fig:regions}). The starting configuration is fixed. All algorithms are written in C++. We present sample results obtained on an idle Linux operating system using an Intel\textregistered~Core\texttrademark~i9--13980HX in Table.~\ref{tab:benchmarks}. All evaluations are ran sequentially, with no parallelization or multi-threading. More specifically, we note the time taken by each algorithm to evaluate all configurations, then we divide that time by the number of states (\num{1e9}) to obtain the average compute time per state in \SI{}{\micro{}\second{}}. Moreover, we mark the error difference in path length with respect to the path length computed by OMPL\@. Our method outperforms OMPL by a factor of 15, whereas~\cite{Desaulniers1995} outperforms OMPL by a factor of 4.21. Both methods achieve a machine precision level of error in path length with the maximum path length error being in the order of \num{1e-15}.

Some types are evaluated slighly faster than others since there are more computations involved, which is why we report average compute times. Further details on the variability of the computation speedups are reported in the relevant open-source repository.

An important implementation note is that, according to its original authors, the algorithm reported in~\cite{Desaulniers1995} is based on a dichomatic approach, where in order to determine which region $c_{f_{L}}$ is in (see Fig.~\ref{fig:regions}), one has to solve a series of at most five inequality tests. Moreover, we do not perform the last step of the original algorithm proposed in~\cite{Desaulniers1995}. Instead, we stick to benchmarking paths against~\cite{Desaulniers1995} with final configurations randomly generated in $c_{f_{L_{x}}}$'s $Q_{1}$ directly.

\begin{table*}[htbp]
	\centering
	\caption{Sample benchmarking results comparing our proposed planner with OMPL's implementation of the original Reeds-Shepp algorithm and our implementation of~\cite{Desaulniers1995}. Results obtained on an idle Linux OS using an Intel\textregistered~Core\texttrademark~i9--13980HX with -O3 optimization.}
	\begin{tabular}{lcccccc}
		\toprule
		\textbf{Method} & \textbf{Average Time per State }[\SI{}{\micro{}\second{}}] & \textbf{Time Ratio to OMPL} & \textbf{Maximum Path Length Error} & \textbf{Average Path Length Error} \\
		\midrule
		Proposed        & 0.0827                                                     & 15.12                       & \num{9.82e-15}                     & \num{3.28e-16}                     \\
		Desaulniers'    & 0.216                                                      & 5.79                        & \num{5.82e-15}                     & \num{2.77e-16}                     \\
		OMPL            & 1.25                                                       & ~1                          & ~-                                 & ~-                                 \\
		\bottomrule
	\end{tabular}
	{}\label{tab:benchmarks}
\end{table*}

\subsection{Validation}
{}\label{sec:validation}
We ran numerous experiments, each involving \num{1e6} randomly generated start and final configurations, as well as randomly chosen radii. For example, we sample $x_{0}$, $y_{0}$, $x_{f}$, and $y_{f}$ from uniform distributions spanning $-1000$ to $1000$. We sample $\theta_{0}$ and $\theta_{f}$ from a uniform distribution spanning $\rinterval{-\pi}{\pi}$. After computing the path with our proposed method and comparing it to OMPL's computed path, we forward simulate our path and ensure that the final configuration is reached with zero error. We also wrote a demo that generates and plots optimal paths for random start/final configurations --- bundled with our provided source code.

Furthermore, we sample \num{1e6} final configurations from uniform distributions spanning the ranges $x_{f} \in [-100, 100], y_{f} \in [-100, 100], \theta_{f} \in \rinterval{-\pi}{\pi}$ for a fixed start configuration $p_{0} = (0, 0, 0)$ and we obtain the path type $\# \in [1-20]$ for each final configuration. We then assign a unique color to each type, and we plot the result in 3D in order to visualize and inspect the partition in 3D. We illustrate the result in Fig.~\ref{fig:3D_cases}.

\begin{figure}[!ht]
	\centering
	\includegraphics[width=8.5cm,keepaspectratio]{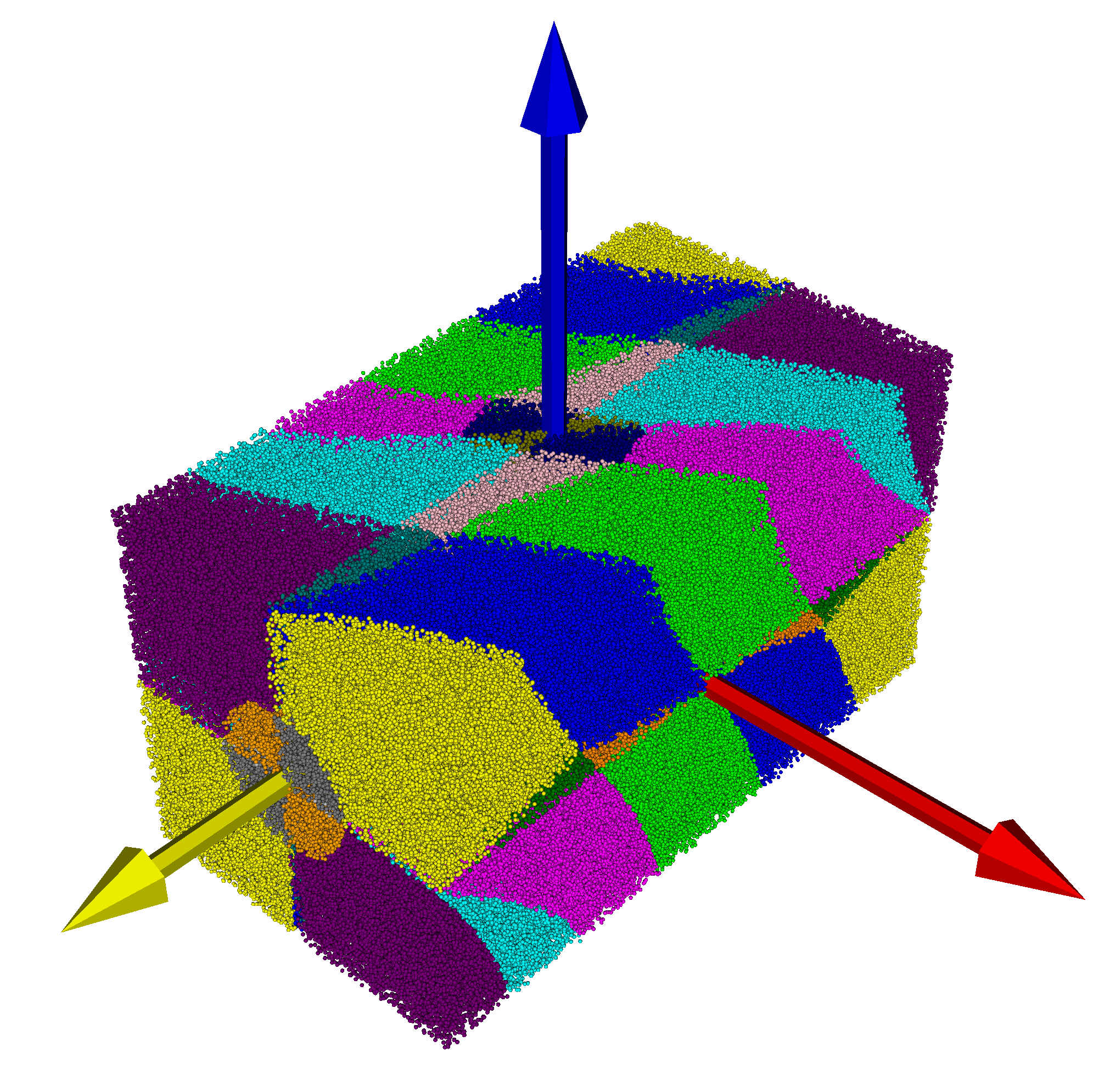}
	\caption{3D partition plot for \num{1e6} cases where $p_{0} = \left(0, 0, 0\right)$ and final configurations $p_{f}$ span $\left[\left[-100,100\right],\left[-100,100\right], \rinterval{-\pi}{\pi} \right]$. Radius $r = 20$. Each case has a unique color based on its optimal path type $\in \left[1-20\right]$. Red axis = $x$ axis, blue axis = $y$ axis, yellow axis = $\theta$ axis.}\label{fig:3D_cases}
\end{figure}

\begin{table*}[htbp]
	\centering
	\caption{Explicit $\Omega$ expressions obtained by solving equations~\eqref{eq:omega_11},\eqref{eq:omega_12}, and~\eqref{eq:omega_13}. Simpler forms arise by enforcing real-valued solutions or using equivalent geometry-based derivations.}\label{tab:omega_equations}
	\begin{tabular}{c}
		\toprule
		$\Omega_{1} = \Re{}\left[2 k \pi- i\ln\left(\frac{-re^{\frac{1}{2}i\theta_{0}} + i\sqrt{-ir\left(dx + idy\right) + e^{i\theta_{0}}\left(dx^{2}+dy^{2}\right)+ire^{2i\theta_{0}}\left(dx-idy\right)}}{ire^{\frac{3}{2}i\theta_{0}} + e^{\frac{1}{2}i\theta_{0}} \left(dx + i dy \right)}\right) \right]$ \\
		\midrule
		$\Omega_{2} = \Re{}\left[2 k \pi- i\ln\left(\frac{e^{\frac{1}{2}i\theta_{0}} \sqrt{ir\left(dx-dy\right) + e^{i\theta_{0}}\left(dx^{2}+dy^{2}\right)-ire^{2i\theta_{0}}\left(dx+dy\right)}-rie^{i\theta_{0}}}{r-e^{i\theta_{0}} \left(i dx + dy \right)}\right) \right]$                                 \\
		\midrule
		$a = -\frac{r}{4} \left(dx + idy \right)$                                                                                                                                                                                                                                                               \\
		$b = 2r^{2}-2re^{i\theta_{0}}\left(idx+dy\right)$                                                                                                                                                                                                                                                       \\
		$c = 2r^{2}e^{2i\theta_{0}} + 2re^{i\theta_{0}} \left(idx-dy\right)$                                                                                                                                                                                                                                    \\
		$d = -r\left(dx + idy\right) + 4ir^{2}e^{i\theta_{0}} + ie^{i\theta_{0}}\left(dx^{2}+dy^{2} \right) + re^{2i\theta_{0}}\left(dx-idy\right)$                                                                                                                                                             \\
		$e = ir^{2}e^{i\theta_{0}} + i \frac{1}{4}e^{i\theta_{0}}\left(dx^{2}+dy^{2} \right) + \frac{1}{4}re^{2i\theta_{0}}\left(dx-idy\right)$                                                                                                                                                                 \\
		$f = r^{2}e^{2i\theta_{0}} + re^{i\theta_{0}}\left(idx-dy\right)$                                                                                                                                                                                                                                       \\
		$\Omega_{3} = \Re{}\left[2 k \pi-i\ln\left(\frac{a + \frac{1}{4} \sqrt{4 b c + d^{2}} + e}{f} \right) \right]$                                                                                                                                                                                          \\
		\bottomrule
	\end{tabular}
\end{table*}

\subsection{Under-Specified Reeds-Shepp Solution}
{}\label{sec:underspecified_results}
Equations $\eqref{eq:omega_11}$, $\eqref{eq:omega_12}$, and $\eqref{eq:omega_13}$ were solved using symbolic solvers in order to obtain explicit expressions for $\Omega_{1}$, $\Omega_{2}$, and $\Omega_{3}$, respectively. We present the results in Table~\ref{tab:omega_equations}. In those expressions, $dx = x_{f}-x_{0}$ and $dy = y_{f}-y_{0}$ for the first quadrant, $Q_{1}$. For $Q_{2}$, $dx = x_{0}-x_{f}$. For $Q_{3}$, $dx = x_{0}-x_{f}$ and $dy = y_{0}-y_{f}$. Lastly, for quadrant $Q_{4}$, $dy = y_{0}-y_{f}$. Since the provided solutions for $\Omega$ consider by default $\theta_{0} = \frac{\pi}{2}$, one needs to perform operations such as $\Omega_{1} = \theta_{0}-\Omega_{1}$ for $R_{1}$ in $Q_{1}$. More implementation details are found in our provided open source repository.

We also visualize the solution of the under-specified Reeds-Shepp algorithm for a $1000\times{}1000$ grid map with the start configuration $p_{0} = (500,500,\frac{\pi}{2})$ and with a radius $r = 400$ in Fig.~\ref{fig:omega_distances}. We compute $\Omega$ for every $(x_{f},y_{f})$ in the grid and visualize the solution (in $\degree$) in the left side figure. We then compute the path lengths obtained for each final configuration $(x_{f}, y_{f}, \Omega)$ and visualize it in the right side figure.

An important note is that those computed values are rotationally invariant with respect to $p_{0}$, meaning the whole solution can be computed once and saved (memoization) for free online queries. Rotations can be accounted for by performing a simple transformation of the queried point of interest into the local frame of reference. For values between grid points, one can simply interpolate. On the other hand, the provided explicit formulas in Table~\ref{tab:omega_equations} can compute $\Omega$ for any point $\in \mathbb{R}^{2}$.

It is also noteworthy that the distances reported in Fig.~\ref{fig:omega_distances} were computed using our proposed algorithm. Computing distances for the under-specfied Reeds-Shepp problem serves as further validation of the correctness of our proposed accelerated Reeds-Shepp algorithms, since all of the computations in this scenario constitute edge cases, with at least one of the segments having a length of zero at the boundary of one or more path type subpartitions. We highlight the important fact that the maximum recorded length of all segments that are supposed to have a zero length is in the order of \num{1e-9}.

Lastly, and in order to validate that the computed path lengths are indeed the shortest, we compute for every $(x_{f}, y_{f})$ all the paths in the range $\theta_{f} \in \rinterval{-\pi}{\pi}$ at a discrete step of \SI{0.05}\degree{}, then compare the minimum distance with our computed one. The distances reported in Fig.~\ref{fig:omega_distances} are indeed the shorter ones.

\begin{figure}[!ht]
	\centering
	\begin{subfigure}{}
		\includegraphics[width=4.25cm]{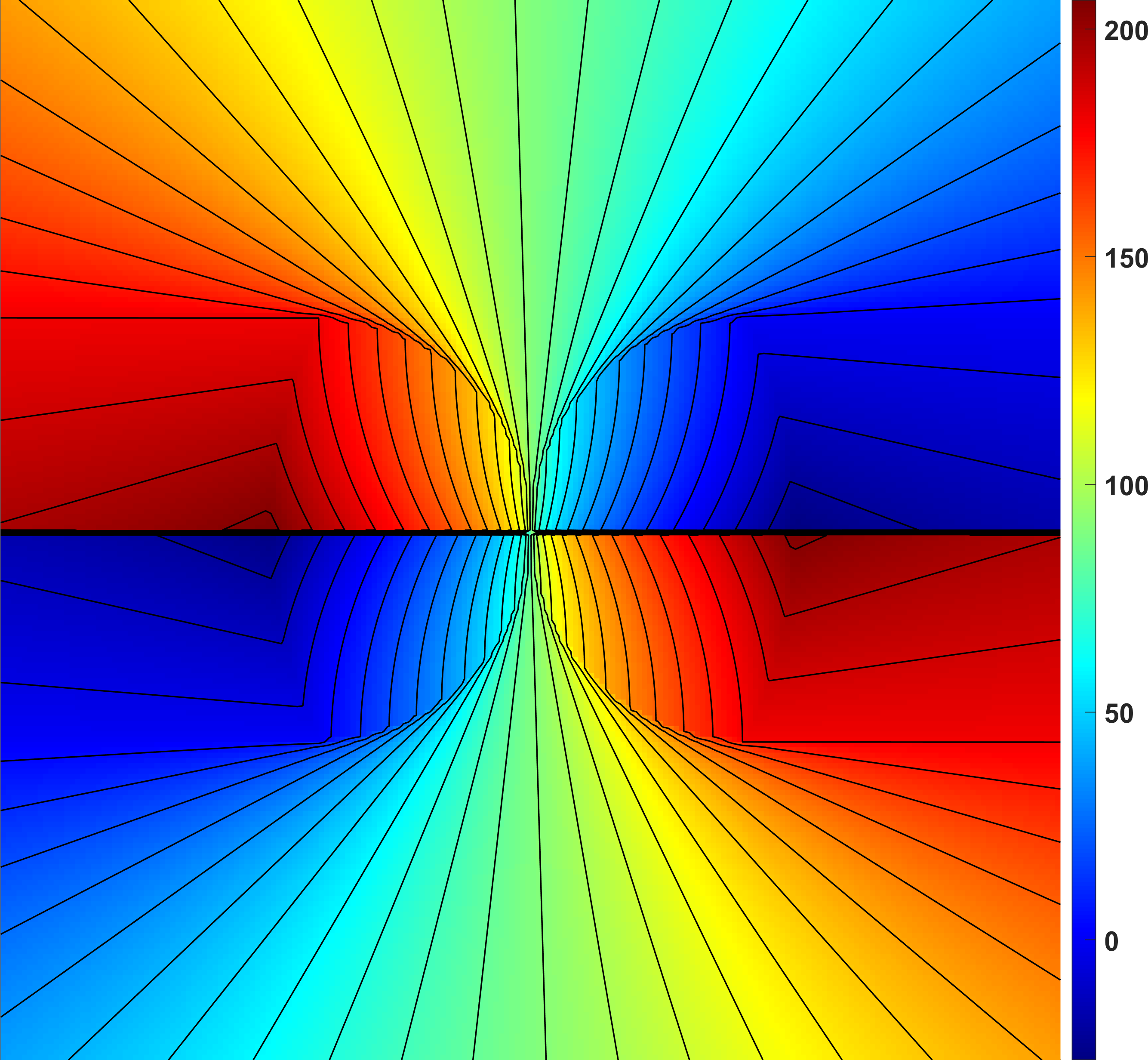}
	\end{subfigure}\hfil
	\begin{subfigure}{}
		\includegraphics[width=4.25cm]{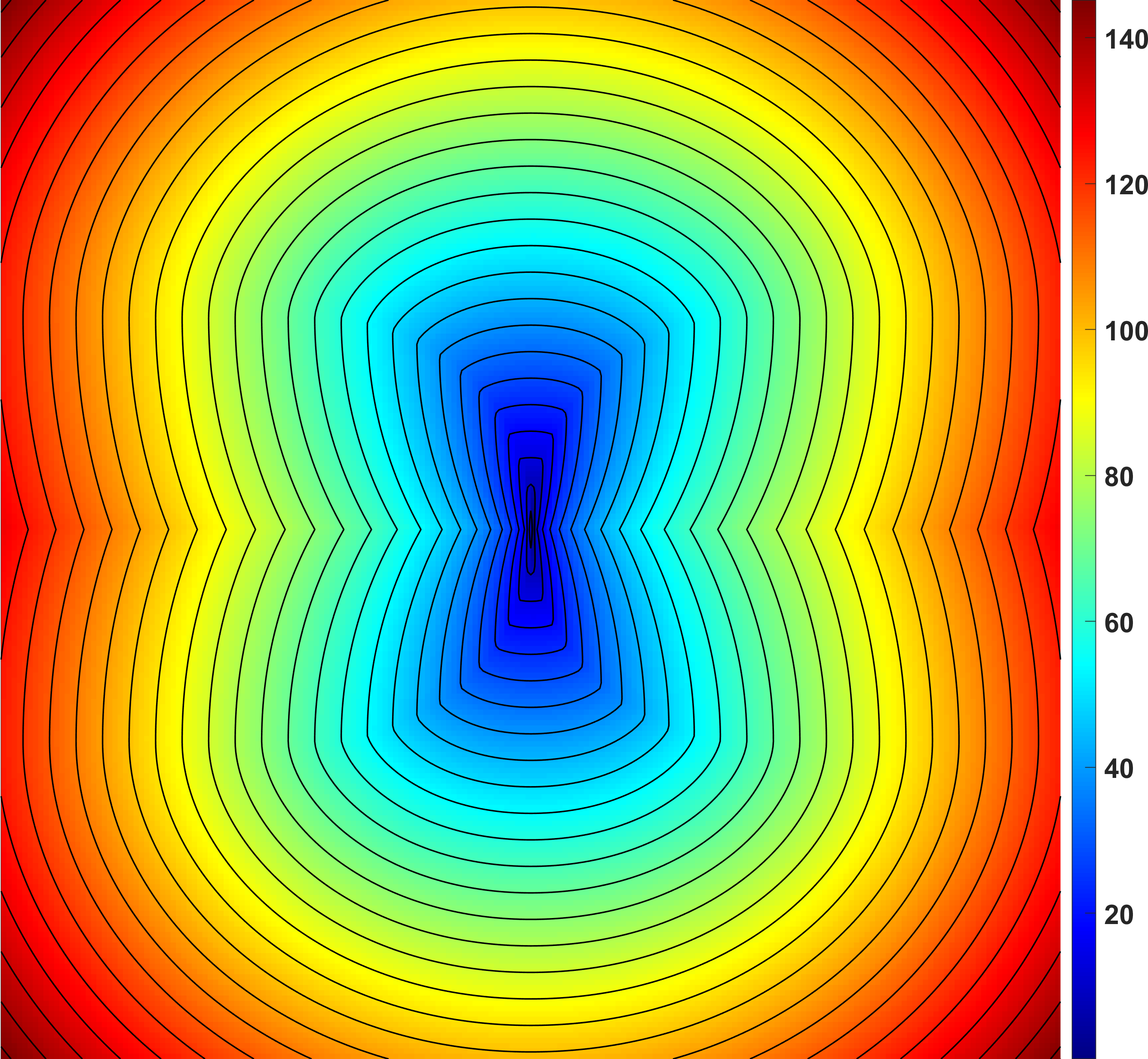}
	\end{subfigure}\hfil
	\caption{Solution of the under-specified Reeds-Shepp algorithm for a $1000\times{}1000$ grid map with the start configuration $p_{0} = (500,500,\frac{\pi}{2})$ and with a radius $r = 400$. We compute $\Omega$ for every $(x_{f},y_{f})$ in the grid and visualize the solution (in $\degree$) in the left side figure. We then compute the path lengths obtained for each final configuration $(x_{f}, y_{f}, \Omega)$ and visualize it in the right side figure. We also show the colorbars. Visualization scheme is interpolation.}\label{fig:omega_distances}
\end{figure}

\section{Conclusion \& Future Work}
In this work, we further simplify the problem of computing shortest-distance Reeds-Shepp paths between any start and final configuration by reducing the minimum sufficient set of path types to 20. We study each of these types individually and provide a new, intuitive, and efficient state-space partition based on geometrical reasoning.

Since fewer path types need to be considered and our partitioning allows us to find the shortest-distance solution by evaluating only one path type per query, we develop an algorithm that outperforms the state of the art by more than an order of magnitude. This improvement paves the way for reintroducing efficient recursive algorithms that account for obstacles and other constraints.

We conduct exhaustive experiments to validate the correctness and completeness of our proposed algorithm. Additionally, we implement a classical method in modern C++ as part of our benchmarking process.

Furthermore, we introduce the under-specified Reeds-Shepp problem and provide an explicit solution. This result enables the computation of a rotationally invariant and symmetric shortest-distance transform that can be precomputed and stored offline. All of our main contributions are documented through open-source code, with pseudocode provided in the appendix.

In future work, we aim to use the proposed accelerated and under-specified Reeds-Shepp algorithms as underlying solvers in a broader framework. This framework will focus on planning kinematically feasible paths composed of Reeds-Shepp primitives for kinematically constrained nonholonomic robots and autonomous ground vehicles operating in obstacle-dense environments.

\section*{Acknowledgments}
The authors would like to thank Professor Guy Desaulniers for the valuable input regarding the reproduction of the original work~\cite{Desaulniers1995}.

\appendices{}
\section{Algorithms}\label{app:algs}

\clearpage{}

\begin{algorithm}[H]
	\setcounter{algorithm}{0}
	\caption{Forward Projection to First Local Quadrant}
	{}\label{alg:ForwardProjection}
	\textbf{Inputs:}
	$\left(x_{0}, y_{0}\right) \gets{}$ start position \\
	\hspace*{11mm} $\left(x_{f}^{l}, y_{f}^{l}, \theta_{f}^{l}\right) \gets{}$ local final configuration \\
	\textbf{Output: } $p_{m}^{l} = (x_{m}^{l}, y_{m}^{l}, \theta_{m}^{l}) \gets{}$ local mirrored configuration
	\begin{algorithmic}[1]
		\Procedure{ForwardProjectToQ1}{Inputs}
		\State$dx = x_{f}^{l}-x_{0}$, $dy = y_{f}^{l}-y_{0}$
		\If{$dx > 0 \land{} dy > 0$} \Comment{Quadrant $Q_{1}$}
		\State$p_{m}^{l} \gets{} p_{f}^{l}$
		\ElsIf{$dx \leq{} 0 \land{} dy \geq{} 0$} \Comment{Quadrant $Q_{2}$}
		\State$x_{m}^{l} \gets{} 2 \cdot x_{0}-x_{f}^{l}$
		\State$\theta_{m}^{l} \gets{} 2 \pi-\theta_{f}^{l}$
		\ElsIf{$dx \leq{} 0 \land{} dy \leq{} 0 $} \Comment{Quadrant $Q_{3}$}
		\State$x_{m}^{l} \gets{} 2 \cdot x_{0}-x_{f}^{l}$
		\State$y_{m}^{l} \gets{} 2 \cdot y_{0}-y_{f}^{l}$
		\ElsIf{$dx > 0 \land{} dy < 0 $} \Comment{Quadrant $Q_{4}$}
		\State$y_{m}^{l} \gets{} 2 \cdot y_{0}-y_{f}^{l}$
		\State$\theta_{m}^{l} \gets{} 2 \pi-\theta_{f}^{l}$
		\EndIf{}
		\State\textbf{return} $p_{m}^{l}$
		\EndProcedure{}
	\end{algorithmic}
\end{algorithm}

\begin{algorithm}[H]
	\setcounter{algorithm}{2}
	\caption{Is In Set B}\label{alg:booleanset}
	\small{}
	\textbf{Inputs:}
	$p_{0}^{l} = (x_{0}^{l}, y_{0}^{l}, \theta_{0}^{l}) \gets{}$ local start configuration \\
	\hspace*{11mm} $p_{m}^{l} = (x_{m}^{l}, y_{m}^{l}, \theta_{m}^{l}) \gets{}$ final mirrored configuration \\
	\hspace*{11mm} $r \gets{}$ minimum turning radius \\
	\textbf{Output: } $\text{Boolean} \gets{}$ true if path is in set B, false otherwise
	\begin{algorithmic}[1]
		\Procedure{IsInSetB}{Inputs}
		\State$c_{0_{R_{x}}} \gets{} x_{0}^{l} + r\cos\left(\theta_{0}^{l}-\frac{\pi}{2}\right)$
		\State$c_{0_{R_{y}}} \gets{} y_{0}^{l} + r\sin\left(\theta_{0}^{l}-\frac{\pi}{2}\right)$
		\State$c_{0_{L_{x}}} \gets{} x_{0}^{l}-r\cos\left(\theta_{0}^{l}-\frac{\pi}{2}\right)$
		\State$c_{0_{L_{y}}} \gets{} y_{0}^{l}-r\sin\left(\theta_{0}^{l}-\frac{\pi}{2}\right)$
		\State$c_{m_{R_{x}}} \gets{} x_{m}^{l} + r\cos\left(\theta_{m}^{l}-\frac{\pi}{2}\right)$
		\State$c_{m_{R_{y}}} \gets{} y_{m}^{l} + r\sin\left(\theta_{m}^{l}-\frac{\pi}{2}\right)$
		\State$c_{m_{L_{x}}} \gets{} x_{m}^{l}-r\cos\left(\theta_{m}^{l}-\frac{\pi}{2}\right)$
		\State$c_{m_{L_{y}}} \gets{} y_{m}^{l}-r\sin\left(\theta_{m}^{l}-\frac{\pi}{2}\right)$
		\State$\mathcal{K} \gets{} 2r\sqrt{2}$
		\State~$LL \gets{} \twonorm{c_{0_{L}}-c_{m_{L}}}$, $LR \gets{}\twonorm{c_{0_{L}}-c_{m_{R}}}$
		\State~$RL \gets{} \twonorm{c_{0_{R}}-c_{m_{L}}}$, $RR \gets{}\twonorm{c_{0_{R}}-c_{m_{R}}}$
		\State~$\mathcal{P}_{1} \gets{} (RR \leq{} \mathcal{K}) \land{} (LL \leq{} \mathcal{K}) \land{} (LR \leq{} 2r)$
		\State~$\mathcal{P}_{2} \gets{} (RR \leq{} \mathcal{K}) \land{} (LL \leq{} \mathcal{K}) \land{} (RL \leq{} 2r)$
		\State~$\mathcal{P}_{3} \gets{} (LR \leq{} 2r) \land{} (LL \leq{} \mathcal{K}) \land{} (RL \leq{} 2r)$
		\State\textbf{return} $\left(\mathcal{P}_{1} \lor{} \mathcal{P}_{2} \lor{} \mathcal{P}_{3}\right)$
		\EndProcedure{}
	\end{algorithmic}
\end{algorithm}

\begin{algorithm}[H]
	\setcounter{algorithm}{3}
	\caption{Get Local Mirrored Optimal Path}
	{}\label{alg:main_partition}
	\textbf{Inputs:}
	$p_{0}^{l} \gets{}$ local start configuration \\
	\hspace*{11mm} $p_{f}^{l} \gets{}$ local mirrored configuration \\
	\textbf{Output: } $P_{m}^{l} \gets{}$ optimal local path
	\begin{algorithmic}[1]
		\Procedure{GetOptimalPath}{Inputs}
		\If{\Call{IsInSetB}{$p_{0}^{l}$, $p_{f}^{l}$}}
		\State~$P_{m}^{l} \gets{}$ \Call{SetBPartitions}{$p_{0}^{l}$, $p_{f}^{l}$}
		\Else{}
		\State~$P_{m}^{l} \gets{}$ \Call{SetAPartitions}{$p_{0}^{l}$, $p_{f}^{l}$}
		\EndIf{}
		\State\textbf{return} $P_{m}^{l}$
		\EndProcedure{}
	\end{algorithmic}
\end{algorithm}

\begin{algorithm}[H]
	\setcounter{algorithm}{1}
	\caption{Backward Projection to Original Quadrant}
	{}\label{alg:BackwardProjection}
	\textbf{Inputs:}
	$P_{m}^{L} \equiv{} \mathcal{T}_{m}^{1\times5}$, $\mathcal{D}_{m}^{1\times5}$, $\mathcal{L}_{m}^{1\times5}$, $Q$ \\
	\textbf{Output: } $P_{f}^{L} \equiv{} \mathcal{T}_{f}^{1\times5}$, $\mathcal{D}_{f}^{1\times5}$, $\mathcal{L}_{f}^{1\times5}$
	\begin{algorithmic}[1]
		\Procedure{BackwardProjectFromQ1}{Inputs}
		\If{$Q == 1$}
		\State$\mathcal{T}_{f}, \mathcal{D}_{f}, \mathcal{L}_{f} \gets{} \mathcal{T}_{m}, \mathcal{D}_{m}, \mathcal{L}_{m}$
		\State\textbf{return} $\mathcal{T}_{f}^{1\times5}$, $\mathcal{D}_{f}^{1\times5}$, $\mathcal{L}_{f}^{1\times5}$
		\EndIf{}
		\State$\mathcal{L}_{f} \gets{} \mathcal{L}_{m}$
		\For{$i = 1$ to $5$}
		\If{$\mathcal{T}_{m}[i]$ == `n'}
		\State\textbf{break}
		\EndIf{}
		\If{$Q == 2$}
		\State$\mathcal{D}_{f}[i]  \gets{} -\mathcal{D}_{m}[i]$
		\State$\mathcal{T}_{f}[i] \gets{}$ $\mathcal{T}_{m}[i]$
		\ElsIf{$Q == 3$}
		\State$\mathcal{D}_{f}[i]  \gets{} -\mathcal{D}_{m}[i]$
		\If{$\mathcal{T}_{m}[i]$ == `l'}
		\State$\mathcal{T}_{f}[i] \gets{}$ `r'
		\ElsIf{$\mathcal{T}_{m}[i]$ == `r'}
		\State$\mathcal{T}_{f}[i] \gets{}$ `l'
		\EndIf{}
		\ElsIf{$Q == 4$}
		\State$\mathcal{D}_{f}[i]  \gets{} \mathcal{D}_{m}[i]$
		\If{$\mathcal{T}_{m}[i]$ == `l'}
		\State$\mathcal{T}_{f}[i] \gets{}$ `r'
		\ElsIf{$\mathcal{T}_{m}[i]$ == `r'}
		\State$\mathcal{T}_{f}[i] \gets{}$ `l'
		\EndIf{}
		\EndIf{}
		\EndFor{}
		\EndProcedure{}
	\end{algorithmic}
\end{algorithm}

\clearpage{}

\begin{algorithm}[H]
	\setcounter{algorithm}{4}
	\caption{Set A Partitions}
	{}\label{alg:A_partitions}
	\textbf{Inputs:}
	$p_{0}^{l} \gets{}$ local start configuration \\
	\hspace*{11mm} $p_{f}^{l} \gets{}$ local mirrored configuration \\
	\textbf{Output: } $P_{m}^{l} \gets{}$ optimal local mirrored path \\
	\begin{algorithmic}[1]
		\Procedure{SetAParitions}{Inputs}
		\small{}
		\If{$\theta_{f} \geq{} 0$}
		\If{$\left(c_{f_{L_{y}}} \leq{} c_{0_{L_{y}}}\right) \land \left(c_{f_{R_{y}}} \leq c_{0_{L_{y}}}\right)$} \Comment{Lemma~\ref{lem:third}}
		\If{$\left(t_{2} \leq{} -2r\right) \lor{} \left(d_{1} \leq{} r\right)$} \Comment{Proposition~\ref{prop:p7_vs_p8}}
		\State~$P_{m}^{l} \gets{} P_{7}$
		\Else{}
		\State~$P_{m}^{l} \gets{} P_{8}$
		\EndIf{}
		\ElsIf{$\left(\theta_{f} < |\angle L_{f}L_{0}| \right)$} \Comment{Proposition~\ref{prop:p1}}
		\If{$\left(\theta_{f} > \angle L_{f}R_{0} \right)$} \Comment{Proposition~\ref{prop:p11}}
		\State~$P_{m}^{l} \gets{} P_{11}$
		\ElsIf{$\left(c_{f_{R_{x}}} \geq{} 2r + x_{0} \right) \lor{} \left(c_{f_{R_{y}}} \leq{} c_{0_{L_{y}}} \right)$}
		\State~$P_{m}^{l} \gets{} P_{1}$ \Comment{Proposition~\ref{prop:p10_vs_p1}}
		\ElsIf{$\left(|t_{2}| \leq{} 2r \right)$} \Comment{Proposition~\ref{prop:p9_vs_p10}}
		\State~$P_{m}^{l} \gets{} P_{9}$
		\Else{}
		\State~$P_{m}^{l} \gets{} P_{10}$
		\EndIf{}
		\Else{}
		\If{$\left(c_{f_{L_{x}}} < 0\right)$} \Comment{Proposition~\ref{prop:p11_vs_p2}}
		\State~$P_{m}^{l} \gets{} P_{11}$
		\ElsIf{$\left(\theta_{f} > \angle L_{f}L_{0} + \frac{\pi}{2} \right)$} \Comment{Proposition~\ref{prop:p3}}
		\State~$P_{m}^{l} \gets{} P_{3}$
		\Else{}
		\State~$P_{m}^{l} \gets{} P_{2}$
		\EndIf{}
		\EndIf{}
		\ElsIf{$\theta_{f} < 0$}
		\If{$\left(\theta_{f} < 2 \beta_{0} - \pi \right)$} \Comment{Proposition~\ref{prop:p5_vs_p6}}
		\If{$\left(\theta_{f} < \angle R_{0}L_{f} \right)$} \Comment{Proposition~\ref{prop:p6}}
		\State~$P_{m}^{l} \gets{} P_{6}$
		\ElsIf{$\left(|t_{2}| \leq{} 2r \right)$} \Comment{Proposition~\ref{prop:p12_vs_p5}}
		\State~$P_{m}^{l} \gets{} P_{12}$
		\Else{}
		\State~$P_{m}^{l} \gets{} P_{5}$
		\EndIf{}
		\Else{}
		\If{$\left(\theta_{f} \geq{} \angle R_{f}L_{0} \right) \lor{} t_{1} \leq{} -2r$} \Comment{Proposition~\ref{prop:p1_vs_p4}}
		\State~$P_{m}^{l} \gets{} P_{1}$
		\ElsIf{$\left(c_{f_{L_{x}}} \geq{} 2r \right)$} \Comment{Proposition~\ref{prop:p4_vs_p9}}
		\State~$P_{m}^{l} \gets{} P_{4}$
		\Else{}
		\State~$P_{m}^{l} \gets{} P_{9}$
		\EndIf{}
		\EndIf{}
		\EndIf{}
		\EndProcedure{}
	\end{algorithmic}
\end{algorithm}

\begin{algorithm}[H]
	\setcounter{algorithm}{5}
	\caption{Set B Partitions}
	{}\label{alg:B_partitions}
	\textbf{Inputs:}
	$p_{0}^{l} \gets{}$ local start configuration \\
	\hspace*{11mm} $p_{f}^{l} \gets{}$ local mirrored configuration \\
	\textbf{Output: } $P_{m}^{l} \gets{}$ optimal local mirrored path
	\begin{algorithmic}[1]
		\Procedure{SetBParitions}{Inputs}
		\small{}
		\If{$RL \geq \sqrt{20} r $} \Comment{Proposition~\ref{prop:p9_p12_setB}}
		\If{$\left(\theta_{f} > 2 \beta_{0} - \pi \right)$} \Comment{Proposition~\ref{prop:p9_vs_p12_setB}}
		\State~$P_{m}^{l} \gets{} P_{9}$
		\Else{}
		\State~$P_{m}^{l} \gets{} P_{12}$
		\EndIf{}
		\EndIf{}
		\If{$\theta_{f} \geq{} 0$}
		\If{$\theta_{f} < \frac{\pi}{2}$}
		\If{$\alpha \geq \beta$} \Comment{Proposition~\ref{prop:p13_vs_p14}}
		\If{$\left(T_{13} \leq{} V_{19} \right) \lor{} \left(T_{13} + U_{13} \leq{} 2 U_{19} \right)$}
		\State~$P_{m}^{l} \gets{} P_{13}$ \Comment{Proposition~\ref{prop:p13_vs_p19}}
		\Else{}
		\State~$P_{m}^{l} \gets{} P_{19}$
		\EndIf{}
		\ElsIf{$(RL \leq 2r) \lor (\beta_{3} \geq \gamma)$}
		\State~$P_{m}^{l} \gets{} P_{14}$ \Comment{Proposition~\ref{prop:p19_vs_p14}}
		\Else{}
		\State~$P_{m}^{l} \gets{} P_{19}$
		\EndIf{}
		\Else{}
		\If{$\left(LR \leq 2r\right) \land{} \left(RL \leq 2r\right)$}
		\State~$P_{m}^{l} \gets{} P_{15}$ \Comment{Proposition~\ref{prop:p15_vs_p14}}
		\Else{}
		\State~$P_{m}^{l} \gets{} P_{14}$
		\EndIf{}
		\EndIf{}
		\ElsIf{$\theta_{f} < 0$}
		\If{$\theta_{f} \geq 2 L_{f}R_{0} - \pi$} \Comment{Proposition~\ref{prop:p17_vs_p13}}
		\If{$\alpha > \beta_{1}$} \Comment{Proposition~\ref{prop:p16_vs_p13}}
		\If{$\left(V_{13} \leq T_{20} \right) \lor{} \left(T_{13} + U_{13} \leq{} \theta_{f} + 2 U_{20} \right)$}
		\State~$P_{m}^{l} \gets{} P_{13}$ \Comment{Proposition~\ref{prop:p13_vs_p20}}
		\Else{}
		\State~$P_{m}^{l} \gets{} P_{20}$
		\EndIf{}
		\Else{}
		\If{$\left(O > LL \right) \lor{} \left(O > RR \right)$}
		\State~$P_{m}^{l} \gets{} P_{20}$ \Comment{Proposition~\ref{prop:p20_vs_p16}}
		\Else{}
		\State~$P_{m}^{l} \gets{} P_{16}$
		\EndIf{}
		\EndIf{}
		\EndIf{}
		\If{$\alpha > \beta_{2}$} \Comment{Proposition~\ref{prop:p18_vs_p17}}
		\If{$\left(T_{17} \leq T_{20} \right) \lor{} \left(T_{17} + U_{17} \leq 2 U_{20} \right)$}
		\State~$P_{m}^{l} \gets{} P_{17}$ \Comment{Proposition~\ref{prop:p17_vs_p20}}
		\Else{}
		\State~$P_{m}^{l} \gets{} P_{20}$
		\EndIf{}
		\Else{}
		\If{$(O \leq RR) \lor{} (RL \leq 2r)$} \Comment{Proposition~\ref{prop:p18_vs_p20}}
		\State~$P_{m}^{l} \gets{} P_{18}$
		\Else{}
		\State~$P_{m}^{l} \gets{} P_{20}$
		\EndIf{}
		\EndIf{}
		\EndIf{}
		\EndProcedure{}
	\end{algorithmic}
\end{algorithm}

\clearpage{}

\bibliographystyle{bibtex/IEEEtran}
\bibliography{bibtex/IEEEabrv, bibtex/bibliography}

\vspace{11pt}

\vfill

\end{document}